%% file: main.tex
\documentclass[10pt]{article}

\PassOptionsToPackage{table,xcdraw}{xcolor}
\usepackage[preprint]{tmlr}

\usepackage[utf8]{inputenc}
\usepackage[T1]{fontenc}
\usepackage{microtype}
\usepackage{graphicx}
\usepackage{booktabs}
\usepackage{multirow}
\usepackage{array}
\usepackage{float}
\usepackage{placeins}

\usepackage{amsmath}
\usepackage{amssymb}
\usepackage{amsfonts}
\usepackage{mathtools}
\usepackage{amsthm}
\usepackage{bm}

\usepackage{comment}
\usepackage{enumitem}
\usepackage{xcolor}
\usepackage[normalem]{ulem} %
\definecolor{green}{RGB}{0,128,0}
\definecolor{red}{RGB}{255,0,0}

\usepackage{hyperref}
\usepackage{url}

\usepackage{algorithm,algorithmic}
\AtBeginEnvironment{algorithmic}{}

\newtheorem{assumption}{Assumption}
\newtheorem{theorem}{Theorem}
\newtheorem{lemma}{Lemma}

\def\month{MM}  %
\def\year{YYYY} %
\def\openreview{\url{https://openreview.net/forum?id=XXXX}} %

\title{GUARD: \underline{G}uided \underline{U}nlearning \underline{a}nd \underline{R}etention via \underline{D}ata Attribution for Large Language Models}

\author{%
\parbox{\textwidth}{\raggedright\textbf{\mbox{Peizhi Niu$^{*}$}, \mbox{Evelyn Ma$^{*}$}, \mbox{Boya Hou$^{*}$}, \mbox{Huiting Zhou$^{\dagger}$}, \mbox{Duo Zhou$^{\dagger}$}, \mbox{Huan Zhang}, \mbox{S. Rasoul Etesami}, \mbox{Olgica Milenkovic}}}\\
{\normalfont Department of Electrical \& Computer Engineering, University of Illinois Urbana-Champaign\\
\texttt{\{peizhin2, pingm, boyahou2, hzhou53, duozhou2, huanz, etesami1, milenkov\}@illinois.edu}}}

\begin{document}

\maketitle

{\renewcommand{\thefootnote}{\fnsymbol{footnote}}\footnotetext[1]{Equal contribution.}\footnotetext[2]{Equal contribution.}}

\begin{abstract}
  Unlearning large language models (LLMs) is an important problem due to copyright and privacy concerns. A key challenge in LLM unlearning is unintended forgetting because the removal of data inadvertently impairs the utility of the model. While prior work has primarily focused on architectural innovations, the influence of data-level factors on unlearning performance remains underexplored. To address this issue, we introduce GUARD, a new framework for \underline{G}uided \underline{U}nlearning \underline{A}nd \underline{R}etention via \underline{D}ata Attribution. GUARD uses a lightweight proxy data attribution metric tailored to LLM unlearning which efficiently quantifies the ``alignment'' between the Forget and Retain (or synthetic supplement) Sets. This metric leads to a new unlearning objective that assigns adaptive, nonuniform unlearning weights to samples which decay exponentially with respect to their proxy attribution scores, and thereby provably mitigates unintended retention loss. Extensive experiments on the TOFU benchmark with a $10\%$ forget set demonstrate that applying GUARD to Gradient Ascent (GA) improves the Retain Set ROUGE-L from $0.18$ to $22.40$ (a $124\times$ increase) and the Retain Probability from $0.06\%$ to $17.50\%$, reflecting substantial functional recovery from the GA collapse. Furthermore, GUARD combined with NPO (Negative Preference Optimization) outperforms Representation Misdirection for Unlearning (RMU) by offering $33.80\%$ and $53.10\%$ reductions in the Retain Set Sacrifice Rates computed from ROUGE-L and Probability ($\rho_r$ and $\rho_p$), and a $16.01\%$ improvement in knowledge retention on the MUSE NEWS Retain Set. These findings persist even when attribution scores are computed using synthetic or highly subsampled ($20\%$) Retain Sets, where NPO+GUARD still reduces $\rho_r$ and $\rho_p$ on the Retain Set by up to $43.89\%$ and $48.13\%$ over NPO. Similar improvements are observed for MUSE, with a tolerable absolute privacy leakage degradation of $13.7$ compared to the best-performing existing method in the privacy category.
\end{abstract}

\input{contents/1.intro}

\input{contents/3.notations}

\input{contents/4.methods}

\input{contents/5.1theorem_boya}

\input{contents/6.exp}

\newpage
\appendix

\input{contents/appendix}

\end{document}

%% file: contents/1.intro.tex
\section{Introduction}
\label{Intro}
Although LLMs offer state-of-the-art performance across a wide range of natural language processing tasks~\citep{ouyang2022training, wei2022finetuned, touvron2023llama}, their deployment has raised substantial privacy and security concerns. LLMs are typically trained on massive corpora of data that contain sensitive or proprietary information and this creates the risk of unintended data leakage and misuse~\citep{carlini2021memorization, yao2024survey}. Privacy challenges have prompted a growing interest in developing \emph{unlearning mechanisms} for the selective removal of specific information from LLMs without the use of retraining~\citep{yao2024large}. Despite significant advances in unlearning~\citep{liu2025rethinking,zhang2024negative, wang2024llm,jia2024wagle, wei2025llmsreallyforgetevaluating}, fine-tuning-based LLM unlearning methods still suffer from \textit{excessive forgetting} and \textit{retention degradation}~\citep{ha2025unlearning}, compromising the utility of the retained knowledge. Similarly, inference-time LLM unlearning models~\citep{liu2024large, ji2024reversing}, which avoid updates and are hence computationally attractive, do not guarantee forgetting as sensitive information remains embedded within the unaltered model parameters (Appendix~\ref{subsec:apdx-rw-unlearn}). For a more in-depth treatment of prior work on unlearning and attribution, we refer the reader to Appendix~\ref{sec:apdx-rw}. Current limitations of LLM unlearning methods consequently include one or more of the following categories: \textbf{a)} \emph{degraded retention}, owing to unlearning performance losses on the Retain Set~\citep{zhang2024negative};  \textbf{b)} \emph{privacy leakage}, caused by unacceptable partial removal of sensitive information~\citep{liu2024large};
\textbf{c)} \emph{computational inefficiency,} due to the frequent requirement for auxiliary LLM guidance/evaluation~\citep{ji2024reversing}; \textbf{d)} \emph{lack of theoretical performance guarantees.} 
To address the above problems, we propose a new unlearning method that ensures \textit{utility preservation, effective forgetting, computational efficiency, and comes with formal theoretical guarantees,} termed \textsc{GUARD} (\textbf{G}uided \textbf{U}nlearning \textbf{A}nd \textbf{R}etention via \textbf{D}ata Attribution). \textsc{GUARD} has the following important properties:
 
\textbf{a) Retention and b) Privacy Loss of GUARD.} Most existing unlearning methods use unprincipled allocation of model unlearning weights across forget samples despite the fact that %
neglecting to account for the impact of a sample can lead to utility loss. To counter this issue, GUARD \emph{dynamically adjusts unlearning weights} based on the influence of each forget sample on the performance over the retained data. More precisely, it allocates unlearning weights so that the data-wise loss weight decreases with the sample-wise contribution to retention, reducing the likelihood of unintentionally degrading retained knowledge. To this end it uses a \emph{specialized} gradient-based retention attribution score that quantifies the ``alignment'' of a forget sample with the retained knowledge, and is tailored to unlearning without requiring retraining or computation of the Hessian. To ensure effective forgetting, GUARD also includes a temperature-controlled reverse unification mechanism that regulates the variance of unlearning weights across samples. This enables the method to mitigate retention degradation by down-weighting high-retention-impact samples while maintaining forgetting.

\textbf{c) Computational Efficiency of GUARD.} Since GUARD has to perform computationally efficient data attribution on LLMs, one has to reconsider existing attribution methods. Conventional data attribution methods estimate the influence of individual training samples on overall model performance by comparing model outputs with and without the inclusion of each sample~\citep{koh2017understanding}. These approaches typically require model retraining or Hessian-based approximations, making them computationally prohibitive for LLMs. Moreover, evaluating the influence of samples based on global performance contradicts the unlearning objective which aims to \textit{decrease} the performance on the Forget Set and simultaneously \textit{preserve} it on the Retain Set (see Appendix~\ref{subsec:apdx-rw-da}). Our solution to this problem is a novel attribution score tailored to LLM unlearning. Inspired by gradient-based sample-wise representations~\citep{nguyen2025minibatchcoresetsmemoryefficientlanguage,wang2025rethinking}, we define the \emph{retention attribution score} of a forget sample as the inner product between its gradient and the gradient averaged over all retain samples. The score captures the alignment between the ``knowledge'' encoded by a forget sample and that of the retain distribution. It also avoids retraining or complicated Hessian estimation inherent to prior attribution techniques. Furthermore, through explicit separation of forget and retain knowledge before inter-sample influence computations, the score effectively captures the contribution of a sample to both utility and unlearning. 

\textbf{d) Theoretical Guarantees of GUARD.}
We also provide a rigorous theoretical analysis demonstrating that, compared to conventional unlearning baselines such as Gradient Ascent (GA), \textsc{GUARD} has the following advantages: (1) the influence of a forget sample on the Retain Set, as estimated by conventional Hessian-based attribution methods, equals a global rescaling of its retention attribution score plus a residual bounded by the spectral spread of the Hessian, and the rescaling is absorbed by the temperature; (2) it effectively reduces the loss on the Retain Set, leading to enhanced \textit{absolute retention}; (3) it maintains comparable degradation on the Forget Set, ensuring strong \textit{absolute forgetting}; and (4) it reduces the retention-performance tradeoff over baselines under an explicit sufficient condition, thereby improving relative retention efficiency.%

\textbf{Empirical Validation.} To test the performance of GUARD, we conduct extensive experiments on the TOFU~\citep{maini2024tofu} and MUSE~\citep{shi2024musemachineunlearningsixway} benchmarks, both specifically designed for LLM unlearning. \emph{All TOFU experiments strictly follow the official TOFU protocol} (AdamW, lr $10^{-5}$, weight decay $0.01$, $5$ unlearning epochs, effective batch size $32$, linear warm-up in epoch $1$), and we report numbers at the final ($5$-th) unlearning epoch, matching the convention adopted by the TOFU leaderboard. We cover four non-loss reweighting baselines (GA, GD, KM, PO) and their GUARD-augmented counterparts (GA+GUARD, GD+GUARD, KM+GUARD, PO+GUARD); four loss reweighting baselines (NPO, SimNPO, FLAT, SatImp) and NPO+GUARD (NPO was chosen for integration with GUARD due to its stability and performance); and RMU, a stable unlearning method frequently used in practice. NPO+GUARD outperforms RMU in retention for both datasets: It reduces the Retain Set Sacrifice Rates computed from ROUGE-L and Probability ($\rho_r$ and $\rho_p$) on TOFU by $33.80\%$ and $53.10\%$ over RMU when forgetting $10\%$ of the training data; on the MUSE NEWS Retain Set, it enhances knowledge retention by $16.01\%$ over RMU (Retain KnowMem $50.0$ vs.\ $43.1$). Furthermore, GUARD remains highly effective even with limited Retain Set access (or synthetic data replacements): When using only $20\%$ of the Retain Set, NPO+GUARD attains Retain Set $\rho_r = 17.78$ and $\rho_p = 5.69$, substantially better than the NPO baseline values of $31.69$ and $10.97$ and approaching the values $14.32$ and $2.87$ achieved by NPO+GUARD on the full Retain Set (see the summary in Table~\ref{tab:positioning}).

\begin{table}[t]
\centering
\caption{Positioning of \textsc{GUARD}. \textbf{Granularity} is the unit over which unlearning strength may vary (``---'' = uniform). \textbf{Retain set}: $\times$ = unused; \textit{reg.} = used only as an additive regularizer, which leaves the allocation across forget units unchanged; $\checkmark$ = determines the weights themselves. \textbf{Hessian-free} = no second-order estimate and no retraining, as originally defined; in Section~\ref{sec:exp} we compare against the \emph{first-order} approximations of RelatIF and RIF so the comparison isolates the score, not the curvature machinery. \textbf{Guar.} = formal bound on retain-side loss relative to a baseline. References: NPO~\citep{zhang2024negative}, SimNPO~\citep{fan2025simplicityprevailsrethinkingnegative}, FLAT~\citep{wang2024llm}, SatImp~\citep{yang2025exploringcriterialossreweighting}, WGA and TNPO~\citep{wang2025rethinking}, WAGLE~\citep{jia2024wagle}, influence functions~\citep{koh2017understanding}, TRAK~\citep{park2023trak}, EK-FAC~\citep{grosse2023studying}, RelatIF~\citep{barshan2020relatif}, and RIF~\citep{rubinstein2025rescaled}.}
\label{tab:positioning}
\vspace{0.05in}
\small
\begin{tabular}{@{}>{\raggedright\arraybackslash}p{3.05cm}>{\raggedright\arraybackslash}p{4.35cm}>{\raggedright\arraybackslash}p{1.95cm}ccc@{}}
\toprule
\textbf{Method} & \textbf{Signal setting unlearning strength} & \textbf{Granularity} & \textbf{Retain} & \textbf{Hess.-} & \textbf{Guar-} \\
 & & & \textbf{set} & \textbf{free} & \textbf{antee} \\
\midrule
GA & uniform & --- & $\times$ & $\checkmark$ & $\times$ \\
PO, GD, KM, RMU & uniform (plus a retain regularizer) & --- & \textit{reg.} & $\checkmark$ & $\times$ \\
NPO, SimNPO & forget likelihood (ratio / reference-free) & per-sample & \textit{reg.} & $\checkmark$ & $\times$ \\
FLAT & $f$-divergence, template vs.\ forget response & per-branch & $\times$ & $\checkmark$ & $\times$ \\
SatImp & token saturation / importance & per-token & \textit{reg.} & $\checkmark$ & $\times$ \\
WGA, TNPO & forget token likelihood (G-effect motivated) & per-token & $\times$ & $\checkmark$ & $\times$ \\
WAGLE & per-weight product of forget and retain gradients & per-weight & $\checkmark$ & $\checkmark$\textsuperscript{$\dagger$} & $\times$ \\
\midrule
Influence fn., TRAK, EK-FAC & global loss influence weighted by the inverse Hessian & per-sample & $\times$ & $\times$ & $\times$ \\
RelatIF, RIF & influence normalized by the sample's self-influence, or rescaled by its leverage score & per-sample & $\times$ & $\times$ & $\times$ \\
\midrule
\textsc{GUARD} (ours) & alignment of each forget gradient with the average retain gradient & per-forget-sample & $\checkmark$ & $\checkmark$ & $\checkmark$ \\
\bottomrule
\multicolumn{6}{@{}p{15.2cm}@{}}{\footnotesize \textsuperscript{$\dagger$} WAGLE scores are derived under a \emph{diagonal} Hessian assumption, so only first-order gradients are computed.}
\end{tabular}
\end{table}

%% file: contents/3.notations.tex
\section{LLM Unlearning Preliminaries}

We start with a review of standard setups for unlearning of LLMs~\citep{feng2024fine, maini2024tofu, zhang2024negative, wang2024llm, jia2024wagle, qiu2025survey, liu2024large, ji2024reversing, yao2024large, zhang2024catastrophic}. 

\textbf{Data - The Forget and Retain Sets.} We define a data point as $(x,y)$, where $x \in \mathbb{R}^p$ represents the input of the LLM, while $y \in \mathbb{R}^q$ denotes the corresponding ground-truth label or response. The whole data set used for fine-tuning an LLM is denoted by $D_0 = \{(x_i,y_i)\}_{i \in [n_0]}$, where $n_0$ is the total number of data points. The Forget Set, denoted by $D_f$, consists of data points requested for removal, while the Retain Set, denoted by $D_r$, contains data points to be kept. The index sets of these data points are defined as $I_f := \{i \mid (x_i, y_i) \in D_f\}$ and $I_r := \{j \mid (x_j, y_j) \in D_r\}$, respectively. We also let $n_f:= |I_f|$ and $n_r := |I_r|$ represent the respective cardinalities of the sets. Clearly, $n_f + n_r = n_0$ and $D_0 = D_f \cup D_r$.

\textbf{The Empirical Loss.} Let $\Theta$ be the parameter space, and let the LLM be parameterized by $\theta \in \Theta$. Given a loss function $l: \mathbb{R}^{p} \times \mathbb{R}^{q} \times \Theta \to \mathbb{R}$, the loss incurred on a data point $(x,y)$ equals $l(x,y;\theta)$. For a data set $D$, the empirical loss is given by
    ${L}_{D}(\theta) = \frac{ 1}{|D|} \sum_{(x_i, y_i) \in D}l(x_i,y_i;\theta)$.

\textbf{Standard LLM Unlearning.} Henceforth, \emph{standard LLM unlearning} refers to the process of adapting an LLM initially trained on $D_0$ in order to remove the influence of a designated Forget Set $D_f$. One typical approach is Gradient Ascent (GA), which updates the model parameters $\theta_0$ as follows:
    \begin{align}
        \theta_{GA} = \theta_0 + \eta \cdot \nabla_{\theta} {L}_{D_f}(\theta_0)  
        = \theta_0 + \eta \cdot \sum_{i \in I_f} \frac{1}{n_f}\cdot \nabla_{\theta} l(x_i,y_i;\theta_0),
        \label{eq:prelim-GA}
    \end{align}
where $\eta$ denotes the unlearning rate. Alternative unlearning methods are discussed in Appendix~\ref{subsec:apdx-unlearn-update}.

%% file: contents/4.methods.tex
\section{{GUARD}}
\label{sec:GUARD}

\begin{figure}[t]
    \centering
    \includegraphics[width=0.74\linewidth]{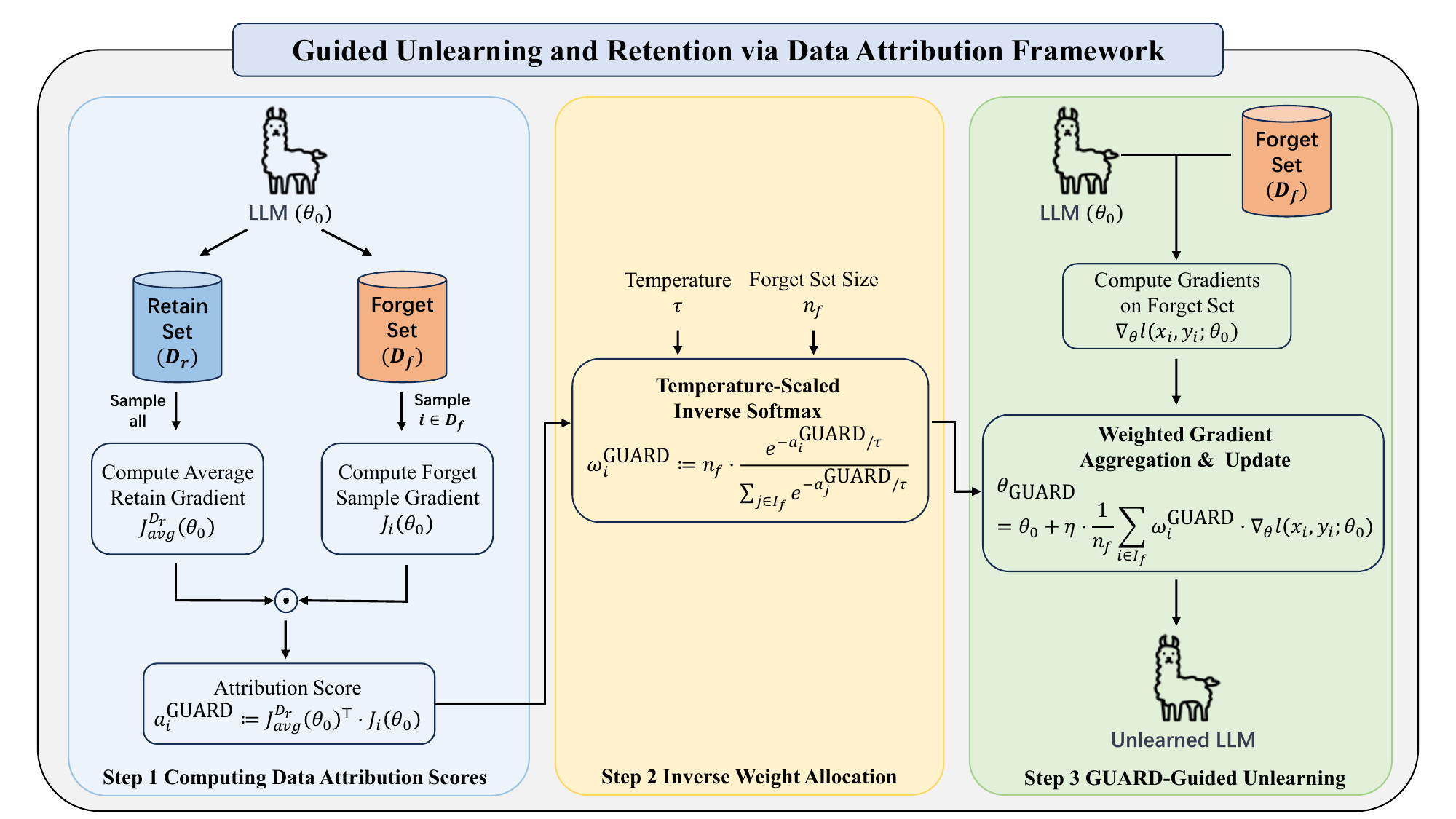}
    \caption{The pipeline. \textbf{Step 1:} Compute the proxy data attribution score $a^{\text{GUARD}}_i$ for each forget sample as the inner product of its gradient and the average gradient of the Retain Set. \textbf{Step 2:} Convert the scores into per-sample unlearning weights $\omega^{\text{GUARD}}_i$ via temperature-scaled reverse Softmax unification. \textbf{Step 3:} Apply the retention-aware weights to the unlearning objective.}
    \label{fig:GUARD-pipeline}
    \vspace{-1em}
\end{figure}

\textbf{The Pipeline.} Unlike traditional methods that assign retention-unaware unlearning weights to all data points (e.g., uniform weights in GA, equal to $\frac{1}{n_f}$, see Eq.~\eqref{eq:prelim-GA}), \textsc{GUARD} introduces a new objective that controllably leverages sample loss weights that are inversely correlated to  data attribution scores. Additionally, \textsc{GUARD} uses a new data attribution scheme specifically tailored to the LLM unlearning setting. The \textsc{GUARD} pipeline %
comprises the steps shown in Fig.~\ref{fig:GUARD-pipeline}:\\
\textbf{Step 1.} \textit{Computing the data attribution scores for retention:} Estimating proxy data attribution scores for all data points in the Forget Set.\\
\textbf{Step 2.} \textit{Deriving retention-aware unlearning weights:} Calculating unlearning weights that are inversely related to the computed attribution scores, with the degree of adjustment modulated by a temperature parameter~$\tau$.\\
\textbf{Step 3.} \textit{Performing GUARD unlearning}: Optimizing the GUARD objective using the retention-aware unlearning weights.

\textbf{Data Attribution for LLM Unlearning.} Conventional data attribution methods estimate the influence of individual training samples by measuring the change in the model outputs when a sample is removed from the training set. These approaches typically rely on retraining or Hessian-based computations (see Appendix~\ref{subsec:apdx-rw-da}). However, such methods are impractical in the context of LLM unlearning for two key reasons: (1) their computational cost is prohibitive at the operational scale of LLMs; and (2) the objective of unlearning differs fundamentally from standard attribution, i.e., effective unlearning necessitates different performance ``shifts'', namely, degraded performance on the Forget Set $D_f$ and preserved performance on the Retain Set $D_r$ (rather than a uniform unidirectional performance change over the entire dataset). To address these challenges, we propose a proxy attribution score $a_i^{\text{GUARD}}$ for $i \in I_f$ as 
\begin{align}
\label{eq:def-J-prod}
a_i^{\text{GUARD}} := J^{D_r}_{\text{avg}}(\theta_0)^\top J_i(\theta_0),\quad \forall i\in I_f.
\end{align}
Here, $J^{D_r}_{avg}(\theta_0)$ denotes the average gradient over $D_r$ and $J_i(\theta_0)$ is the gradient of sample $i$, both evaluated for the initial model parameters $\theta_0$. 
This formulation captures the alignment between the information content of the Forget and the Retain Set. 
Representing sample-level influence via gradients is a well-established practice, particularly for LLMs~\citep{nguyen2025minibatchcoresetsmemoryefficientlanguage,wang2025rethinking}. Note that compared to other first-order influence measures not specifically tailored to unlearning such as RelatIF, the attribution score $a_i^{\text{GUARD}}$ defined in~\eqref{eq:def-J-prod} preserves gradient magnitude. For unlearning, \emph{this magnitude is informative,} as it typically reflects the ``difficulty'' of a sample or the strength of its encoded knowledge (experimental comparison with alternative influence functions can be found in Appendix~\ref{app:alternative_attribution}). Beyond its intuitive interpretation and computational efficiency, our proxy score also admits theoretical analysis (Appendix~\ref{sec:apdx-proxy-da}) and exhibits robustness to the choice of the Retain Set (Section~\ref{sec:exp} and Appendix~\ref{subsec:Robustness}).

\textbf{Retention-Aware Unlearning Weights.} A large retention attribution score for a forget sample \(i\) indicates that the sample \((x_i, y_i)\) has a strong positive influence on the performance of the model over the Retain Set \(D_r\) (e.g., its removal leads to a large performance drop on \(D_r\)). We assign unlearning weights to data samples that decrease exponentially in their attribution scores; at the same time, we modulate the variance of the weight distribution by a temperature-controlled softmax.
We define the unlearning weight for $i \in I_f$ as
\begin{align}
    \omega^{\text{GUARD}}_i := n_f \cdot \frac{e^{-{a}^{\text{GUARD}}_i/\tau}}{\sum_{j\in I_f} e^{-{a}^{\text{GUARD}}_j/\tau}},
    \label{eq:guard-multiplier}
\end{align}  
where as before, \({a}^{\text{GUARD}}_i\) is the attribution score defined in Eq.~\eqref{eq:def-J-prod}, \(\tau\) is the  temperature hyperparameter, and \(n_f = |I_f|\) denotes the size of the Forget Set. We refer to the mapping in Eq.~\eqref{eq:guard-multiplier} from scores to weights as \emph{reverse Softmax unification}. The unlearning weights \(\omega^{\text{GUARD}}_i\) satisfy the:

\textit{1. Normalization Property,} which asserts that 
$\frac{1}{n_f} \sum_{i\in I_f} \omega^{\text{GUARD}}_i = 1$.

\textit{2. Inverse Attribution Score Property,}  
    which asserts that unlearning weights should be decreasing in the attribution scores of the samples; this mitigates unintended loss of knowledge in \(D_r\) during unlearning of \(D_f\).

\textit{3. Controlled Variance Property,}  
    which asserts that the degree of weight differentiation is governed by a temperature parameter \(\tau\) so as to enable effective forgetting on \(D_f\). Comparable forgetting to baselines is supported theoretically (Section~\ref{sec:theorem}) and validated empirically (Section~\ref{sec:exp}).

Note that the inner product \({a}^{\text{GUARD}}_i\) in Eq.~\eqref{eq:def-J-prod} quantifies how much a forget sample contributes to updates in the direction of the retain gradient. 
Samples whose gradients are strongly aligned with the average Retain Set gradient $J^{D_r}_{\text{avg}}(\theta_0)$ cause updates that degrade retention and should therefore be down-weighted during unlearning. The softmax weighting in Eq.~\eqref{eq:guard-multiplier} implements this idea by assigning smaller weights to high-alignment scores, while preserving normalization across the Forget Set. The temperature $\tau$ controls the strength of suppression. \emph{This perspective naturally connects to a constrained unlearning formulation that maximizes the loss on the Forget Set while limiting degradation on the Retain Set.} A formal connection to a constrained optimization view is provided in Section~\ref{sec:theorem}.

\textbf{Unlearning with GUARD.} For a given Forget Set $D_f$ and a model parameterized by $\theta$, our new GUARD unlearning loss objective, $L^{\text{GUARD}}: \Theta \to \mathbb{R}$, is defined as
\begin{align*}
    L^{\text{GUARD}}_{D_f}(\theta) := \frac{1}{n_f}\sum_{i\in I_f} \omega_i^{\text{GUARD}} \cdot l(x_i, y_i; \theta),
\end{align*}
where $\omega_i$ is the unlearning weight defined in \eqref{eq:guard-multiplier}. Using the GA framework from Eq.~\eqref{eq:prelim-GA} as an example, the model update rule under \textsc{GUARD} becomes: 
\begin{align}
\label{eq:GA-guard-update}
    \theta^{\text{GUARD}}_{\text{GA}} 
    &= \theta_{0} +  \frac{\eta}{n_f} \sum_{i \in I_f}  \omega_i^{\text{GUARD}} \cdot\nabla_{\theta}\left[  l(x_i, y_i; \theta_0)\right]. 
\end{align}
As previously discussed,
such an unlearning weight allocation (Eq. \eqref{eq:GA-guard-update}) mitigates unintended forgetting of retained knowledge while it also ensures effective forgetting of undesired information (\textsc{GUARD} objectives for methods other than GA are described in Appendix~\ref{subsec:apdx-guard-obj}).

\textbf{The \textsc{GUARD} Algorithm.} The above approaches   
 are summarized in Algorithm~\ref{alg:GUARD}, which outlines the implementation of \textsc{GUARD} within the GA unlearning framework (Eq. \eqref{eq:prelim-GA}). The algorithm is presented as a single update step for clarity. In our experiments, we follow the official TOFU multi-epoch protocol ($5$ unlearning epochs)~\citep{maini2024tofu}: the per-sample attribution-driven weights $\omega_i^{\text{GUARD}}$ are computed \emph{once} on the fine-tuned model checkpoint and reused across all unlearning epochs, so the additional cost over the underlying baseline is a one-time preprocessing step.

\begin{algorithm}[H]
    \caption{GUARD: Guided Unlearning and Retention via Data Attribution}
    \label{alg:GUARD}
    \begin{algorithmic}[1]
        \STATE \textbf{Input}: Initial model weights $\theta_0$; unlearning rate $\eta$; temperature $\tau$.
        \STATE Compute retention attribution ${a}^{\text{GUARD}}_i$ for all $i \in I_f$ using Eq.~\eqref{eq:def-J-prod} 
        \label{line:IF}
        \COMMENT{Prepare Attribution}
        \STATE Compute unlearning weights $\omega_i$ for all $i \in I_f$ using Eq.~\eqref{eq:guard-multiplier} \COMMENT{Prepare Sample-wise Weights}
        \label{line:multiplier}
        \STATE Update model $\theta_0$ to $\theta_{\text{GA}}^{\text{GUARD}}$ by Eq.~\eqref{eq:GA-guard-update}
        \label{line:update}
        \COMMENT{Update Model}
        \STATE \textbf{Output}: Unlearned model weights $\theta_{\text{GA}}^{\text{GUARD}}$.
    \end{algorithmic}
\end{algorithm}

%% file: contents/5.1theorem_boya.tex
\section{Theoretical Analysis}
\label{sec:theorem}

Let $\psi \in \{0, \text{GA}, \text{GUARD}\}$ denote the superscript corresponding to a model trained without unlearning, trained with baseline unlearning (GA), and GUARD, respectively. For a model parameterized by $\theta^\psi$, we define the empirical loss over the Forget Set $D_f$ and the Retain Set $D_r$ as:
\begin{align}
    L^\psi_{D_f} = \frac{1}{n_f} \sum_{i \in I_f} l(x_i, y_i; \theta^\psi) , \quad
    L^\psi_{D_r} = \frac{1}{n_r} \sum_{i \in I_r} l(x_i, y_i; \theta^\psi),
\end{align}
where $n_f = |I_f|, \, n_r = |I_r|$. Let $g_i := \nabla_\theta l(x_i,y_i;\theta^0)$ denote the gradient of sample $i$ at the model before unlearning, $\theta^0 = \theta_0$. Let $\overline{g}_f := \frac{1}{n_f} \sum_{i \in I_f} g_i$ and $\overline{g}_r := \frac{1}{n_r} \sum_{i \in I_r} g_i$ denote the average gradients over the sets $D_f$ and $D_r$, respectively. We denote the global alignment (score) between the two sets as $\kappa := \langle \overline{g}_f, \overline{g}_r \rangle$ and the normalized alignment $\delta_{\kappa} := \kappa / \tau$, where $\tau$ denotes the GUARD temperature. For each forget index $j \in I_f$, the per-sample alignment is defined as $a_j := \langle \overline{g}_r, g_j \rangle$. We further let $\delta_a := \max_{j\in I_f}|a_j|/\tau$ denote the largest normalized per-sample alignment and $G := \max_{i \in I_f \cup I_r} \|g_i\|_2$ the largest gradient norm. We compare a single unlearning step of GA and GUARD from $\theta^0$, i.e., $\theta^{\text{GA}} := \theta^0 + \eta\,\overline{g}_f$ and $\theta^{\text{GUARD}} := \theta^0 + \eta\,\overline{g}'_f$, as in Eq.~\eqref{eq:prelim-GA} and Eq.~\eqref{eq:GA-guard-update}, where $\overline{g}'_f := \frac{1}{n_f}\sum_{i\in I_f}\omega_i^{\text{GUARD}} g_i$ is the reweighted average forget gradient. We also define the alignment variance
\begin{align}
    \sigma^2_\kappa := \frac{1}{n_f} \sum_{j \in I_f} a_j^2 - \Big( \frac{1}{n_f} \sum_{j \in I_f} a_j \Big)^2 \geq 0. 
    \label{eq.variance}
\end{align}

\begin{assumption}
\label{asum:r-f-entangle}
The global alignment $\kappa$ and the per-sample alignments $a_j,$ $j \in I_f,$ satisfy:
(1) $\kappa > 0$;  (2) $|a_j| / \tau \ll 1,$  $\forall \, j \in I_f$.
\end{assumption}
Here, (1) reflects feature-sharing of LLMs: Without positive alignment, a GA step on $D_f$ would not increase the loss on $D_r$ to first order; and (2) reflects weak per-sample alignment, enabling selective unlearning. Assumption~\ref{asum:r-f-entangle} implies $0 < \delta_{\kappa} \ll 1$.

\begin{assumption}
\label{asum:small-update}
    The Hessians of $L_{D_f}$ and $L_{D_r}$ are bounded near $\theta^0$, and $\eta$ is small enough that the second-order Taylor remainders, which are of order $\eta^2 G^2$, are negligible.
\end{assumption}
This is a standard influence assumption~\citep{koh2017understanding}, also valid for small unlearning perturbations.

Let the centered empirical forget gradient covariance be $\mathrm{Cov}_f := \frac{1}{n_f} \sum_{j \in I_f} (g_j - \overline{g}_f)(g_j - \overline{g}_f)^\top$ and the uncentered empirical forget gradient one $\Sigma_f := \frac{1}{n_f} \sum_{j \in I_f} g_j g_j^\top = \bar{g}_f \bar{g}_f^\top + \mathrm{Cov}_f$.
We next study the performance of GUARD on the Retain Set (Theorem~\ref{lemma:r-loss-GUARD-GA}), the Forget Set (Theorem~\ref{lemma:f-loss-GUARD-GA}), and ultimately, on the newly defined Sacrifice Rate (Theorem~\ref{thrm:SR-GUARD-GA}). The proofs are deferred to Appendix~\ref{sec:apdx-proof-sr}. 

\begin{theorem}
\label{lemma:r-loss-GUARD-GA}
Under Assumptions~\ref{asum:r-f-entangle} and~\ref{asum:small-update}, the losses on the Retain Set of \textsc{GUARD} and GA satisfy
\[
    L_{D_r}^{\text{GA}} - L_{D_r}^{\text{GUARD}} = \frac{\eta}{(1 - \delta_{\kappa})\tau} \cdot \sigma^2_\kappa \;+\; \eta \cdot O(\tau\,\delta_a^3) \;+\; O(\eta^2 G^2).
\]
\end{theorem}
Under~ Assumption~\ref{asum:r-f-entangle} we have $\delta_\kappa \ll 1$. Hence, to first order, \textsc{GUARD} performs better on the Retain Set than GA. 
The improvement scales linearly with the unlearning rate $\eta$. However, excessively large $\eta$'s may violate the small-update assumption and degrade the accuracy of the approximation. Furthermore, the retention gain is proportional to the alignment variance $\sigma^2_\kappa$: A large variance indicates that forget samples differ strongly in how much they align with the Retain Set, which is the heterogeneity GUARD exploits. Note that when $a_j =\kappa$, $\forall j \in I_f$, that is, every forget gradient has the same alignment with the retain gradient, $\sigma_\kappa = 0$. Hence, $\omega_j^{\mathrm{GUARD}}=1$, and the weights become uniform. We now characterize the \textsc{GUARD} Forget Loss.%
\begin{theorem}
\label{lemma:f-loss-GUARD-GA}
Under Assumptions~\ref{asum:r-f-entangle} and~\ref{asum:small-update}, we have
\begin{align*}
\left|L^{\text{GA}}_{D_f} - L^{\text{GUARD}}_{D_f} \right|
\leq \frac{\eta \sigma_\kappa}{\tau-\kappa}  \sqrt{\lambda_{\max}(\mathrm{Cov}_f)} \|\bar{g}_f\|_2
\;+\; \eta \cdot O(\delta_a^2 G^2) \;+\; O(\eta^2 G^2),
\end{align*}
where all the parameters are as previously defined.
\end{theorem}

Next, we define the \textit{Sacrifice Rate} $\rho_{\epsilon}(D)$ which measures the relative performance degradation on a data set $D$ other than the Forget Set per unit of performance decrease on the Forget Set: The lower the rate, the better the relative knowledge retention (see also Section~\ref{sec:exp} for additional absolute metrics): 
    \begin{equation} \label{eq:sr}
    \rho_{\epsilon}(D) = \frac{\epsilon(D;\theta^0)-\epsilon(D;\theta)}{\epsilon(D_f;\theta^0)-\epsilon(D_f;\theta)}.
    \end{equation}
A negative Sacrifice Rate indicates that the performance on $D$ improved during unlearning. In our theoretical analysis, we take $\epsilon$ to be the loss and $D = D_r$, i.e., $\rho^{\psi} := (L^\psi_{D_r} - L^0_{D_r})/(L^\psi_{D_f} - L^0_{D_f})$ for $\psi \in \{\text{GA}, \text{GUARD}\}$.

For the Sacrifice Rate metric, we have the following result for $\sigma_\kappa>0$.
\begin{theorem}
\label{thrm:SR-GUARD-GA}
Let Assumptions~\ref{asum:r-f-entangle} and~\ref{asum:small-update} hold true. If, in addition, 
\begin{align}
\|\bar g_f\|_2 > \max \left\{ \frac{\sigma_\kappa}{\tau-\kappa}, \frac{\kappa }{\sigma_\kappa}\right\} \sqrt{\lambda_{\max}(\mathrm{Cov}_f)},
\label{eq:suff-cond}
\end{align}%
then, to leading order (i.e., ignoring the remainders in Theorems~\ref{lemma:r-loss-GUARD-GA} and~\ref{lemma:f-loss-GUARD-GA}), GUARD achieves a Sacrifice Rate no worse than that of GA, i.e., $\rho^{\text{GA}} - \rho^{\text{GUARD}} \geq 0$.  
\end{theorem}

\textsc{GUARD} performs better than GA when the norm of the mean forget gradient $\|\bar g_f\|_2$ is large relative to the square root of the largest eigenvalue of the forget gradient covariance matrix, $\sqrt{\lambda_{\max}(\mathrm{Cov}_f)}$.
We further observe that the alignment variance $\sigma_\kappa^2$, which captures the variability of the per-sample alignment $a_i$ for $i \in I_f$, plays a critical role: A larger $\sigma_\kappa^2$ reflects greater heterogeneity across samples. In this case, the weighting mechanism in \textsc{GUARD} effectively separates high- and low-alignment samples by down-weighting samples that most increase the Retain Loss and up-weighting less harmful ones.

Consider a spiked covariance model with $\mathrm{Cov}_f=\nu^2 I+U\Lambda U^\top,$ with $U^\top U=I_k$, $\Lambda\succeq 0$, and let $\cos^2\phi:=\|U^\top\bar g_f\|_2^2/\|\bar g_f\|_2^2$ be the squared cosine of the angle between $\bar g_f$ and the spike subspace $\mathrm{span}(U)$.
Under this model, the second condition in \eqref{eq:suff-cond}, namely, $\|\bar g_f\|_2\ge(\kappa/\sigma_\kappa)\sqrt{\lambda_{\max}(\mathrm{Cov}_f)}$, can be relaxed to
\begin{align}
\|\bar g_f\|_2 \geq \frac{\kappa}{\sigma_\kappa} \sqrt{\nu^2+\lambda_{\max}(\Lambda)\cos^2\phi}.
\label{eq:spiked}
\end{align}
Here, $\kappa/\sigma_\kappa$ is the mean of the per-sample alignments $a_j$ divided by the standard deviation $\sigma_\kappa$, and it is small when forget samples differ strongly in how much they align with the Retain Set, which is the exact type of heterogeneity GUARD exploits. Also, $\sqrt{\nu^2+\lambda_{\max}(\Lambda)\cos^2\phi}$ upper-bounds the standard deviation of the per-sample forget gradients projected onto the direction of $\bar g_f$. Condition~\eqref{eq:spiked} therefore holds when the mean forget gradient is large compared with how much individual forget gradients vary along it and when the alignments are heterogeneous. If $\bar g_f\perp\mathrm{span}(U)$, i.e., $\cos\phi=0$, the spike drops out entirely, and $\sqrt{\nu^2+\lambda_{\max}(\Lambda)\cos^2\phi}=\nu$. If $\bar g_f\in\mathrm{span}(U)$, i.e., $\cos\phi=1$, $\sqrt{\nu^2+\lambda_{\max}(\Lambda)\cos^2\phi}=\sqrt{\lambda_{\max}(\mathrm{Cov}_f)},$ and $\|\bar g_f\|_2\ge(\kappa/\sigma_\kappa)\sqrt{\lambda_{\max}(\mathrm{Cov}_f)}$ in~\eqref{eq:suff-cond} follows. 
Thus, $\sqrt{\nu^2+\lambda_{\max}(\Lambda)\cos^2\phi}$ increases with $\cos\phi$, from $\nu$ when $\bar g_f\perp\mathrm{span}(U)$ to $\sqrt{\lambda_{\max}(\mathrm{Cov}_f)}$ when $\bar g_f\in\mathrm{span}(U),$ where~\eqref{eq:spiked} coincides with~\eqref{eq:suff-cond}. Hence, the condition is easiest to satisfy when the dominant fluctuations of the per-sample forget gradients are not along their mean, and it becomes the original condition ~\eqref{eq:suff-cond} when the two directions align.

\textbf{The Special Case of Isotropic Gradients.} Setting $\Lambda = 0$ in the spiked model gives isotropic gradients, i.e., $\mathrm{Cov}_f = \nu^2 I$ for some $\nu > 0$. Gradients of LLMs are not isotropic, but this special case makes the results above even easier to interpret. Since $\sigma_\kappa^2 = \overline{g}_r^\top \mathrm{Cov}_f\, \overline{g}_r = \nu^2 \|\overline{g}_r\|_2^2$, Theorem~\ref{lemma:r-loss-GUARD-GA} shows that, to first order, \textsc{GUARD} reduces the loss on the Retain Set by $\eta\, \nu^2 \|\overline{g}_r\|_2^2/(\tau-\kappa)$ compared to GA. Moreover, the second condition in~\eqref{eq:suff-cond} is no longer needed, and the leading-order term of $\rho^{\text{GA}} - \rho^{\text{GUARD}}$ is nonnegative whenever $\|\overline{g}_f\|_2^2 > \nu^2 \kappa/(\tau - \kappa)$ (Appendix~\ref{sec:apdx-proof-sr}). Note that Theorem~\ref{thrm:SR-GUARD-GA} provides a \emph{conservative sufficient condition}; for the sake of completeness, we also perform empirical evaluations of the condition as outlined below.

\begin{figure}[b]
    \centering
    \includegraphics[width=0.73\linewidth]{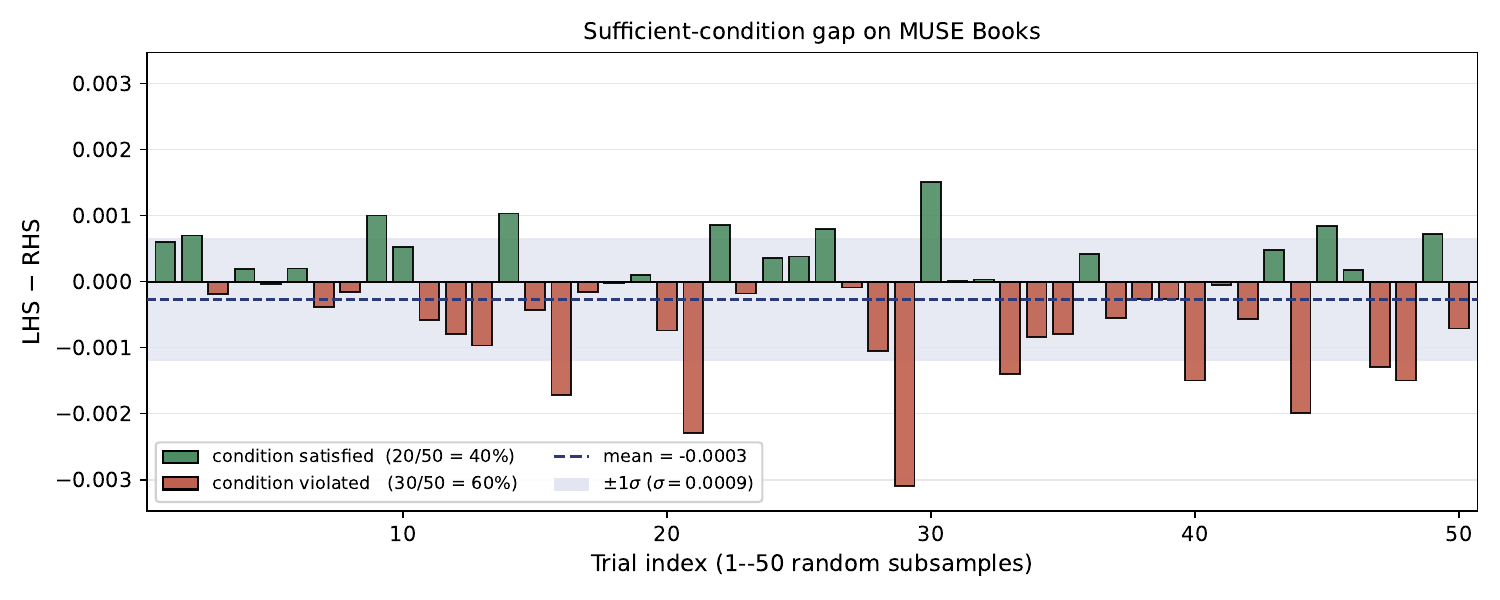}
    \caption{Empirical distribution of $\mathrm{LHS} - \mathrm{RHS}$ for the sufficient condition~\eqref{eq:suff-cond} for $50$ trials, each on a random subsample of MUSE BOOKS ($\tau = 0.03$). Green/red bars indicate trials where the sufficient condition is satisfied/violated. The condition holds for $20$ of the $50$ trials, and the gaps are small ($|\mathrm{LHS}-\mathrm{RHS}| < 0.004$).}
    \label{fig:suff-cond-muse-books}
\end{figure}

\textbf{Practical Validity of the Sufficient Condition.}  All quantities in condition~\eqref{eq:suff-cond} are empirical gradient statistics, so the condition can be checked directly. On the BOOKS corpus of the MUSE benchmark, we check it at the released target model with the temperature $\tau = 0.03$ used throughout the paper, in $50$ trials, each on a random subsample, following the estimation protocol in Appendix~\ref{subsec:apdx-thrm-verify}. As shown in Figure~\ref{fig:suff-cond-muse-books}, the condition is satisfied in $40\%$ of the trials, and the violations are small. On TOFU with a $10\%$ forget split, the same check succeeds in $28\%$ of the trials (Appendix~\ref{subsec:apdx-thrm-verify}), since the gradient statistics in~\eqref{eq:suff-cond} depend on the data set and the model. Even when the condition is violated, \textsc{GUARD} still improves retention over GA on both benchmarks (Section~\ref{sec:exp} and Table~\ref{tab:unlearning_muse_book}). These experiments use multiple AdamW steps rather than the single gradient step analyzed above, so they illustrate rather than test the theorem.

\textbf{Lagrangian Multiplier Interpretation of GUARD.}\label{app:lagrangian_analysis} 
To formally ground the rationale of the weighting mechanism of GUARD, we analyze a linearized unlearning objective subject to a hard retention constraint. Let $\theta \in \mathbb{R}^d$ denote the model parameters. Standard unlearning (e.g., GA) aims to find a parameter update $\Delta$ that maximizes the loss increase on the Forget Set. However, an ideal unlearning algorithm should do so while ensuring the loss increase on the Retain Set remains bounded by some $\varepsilon$.
Using a first-order Taylor expansion, the above problem can be formulated as a linear objective with a quadratic regularizer:
\begin{equation}
\label{eq:primal_lagrangian}
\max_{\Delta} \quad \overline{g}_f^\top \Delta - \frac{\zeta}{2}\|\Delta\|^2 \qquad \text{s.t.} \quad \overline{g}_r^\top \Delta \le \varepsilon,
\end{equation}
where $\overline{g}_f$ and $\overline{g}_r$ are the average gradients over the Forget and Retain Sets as defined above, and $\zeta > 0$ is a regularization parameter controlling the step size (trust region). %
One can form the Lagrangian for Eq.~\eqref{eq:primal_lagrangian},
\[
\mathcal{L}(\Delta, \lambda) = -\overline{g}_f^\top \Delta + \frac{\zeta}{2}\|\Delta\|^2 + \lambda(\overline{g}_r^\top \Delta - \varepsilon),
\]
where $\lambda \ge 0$ denotes the Lagrange multiplier. The Karush-Kuhn-Tucker (KKT) optimality conditions with respect to $\Delta$ yield the stationarity equation:
\[
\nabla_{\Delta}\mathcal{L} = -\overline{g}_f + \zeta \Delta^\star + \lambda \overline{g}_r = 0 \quad \implies \quad \Delta^\star = \frac{1}{\zeta}(\overline{g}_f - \lambda \overline{g}_r).
\]
If the constraint is active (i.e., the unlearning update threatens retention utility), substituting $\Delta^\star$ into the binding constraint $\overline{g}_r^\top \Delta^\star = \varepsilon$ allows us to solve for the optimal multiplier $\lambda^\star$:
\begin{equation}
\overline{g}_r^\top \left[ \frac{1}{\zeta}(\overline{g}_f - \lambda^\star \overline{g}_r) \right] = \varepsilon \quad \implies \quad \lambda^\star = \frac{\overline{g}_r^\top \overline{g}_f - \zeta\varepsilon}{\|\overline{g}_r\|^2}.
\end{equation}
Crucially, the optimal multiplier $\lambda^\star$ grows with the inner product $\langle \overline{g}_r, \overline{g}_f \rangle$, which represents the global alignment (or conflict) between the forget and retain gradients. Substituting $\lambda^\star$ back into the expression for $\Delta^\star$, the optimal update direction becomes:
\begin{equation}
\label{eq:optimal_update_lagrangian}
\Delta^\star \propto \underbrace{\overline{g}_f}_{\text{Standard Ascent}} - \underbrace{\frac{\langle \overline{g}_r, \overline{g}_f \rangle - C}{\|\overline{g}_r\|^2} \overline{g}_r}_{\text{Retention Correction}},
\end{equation}
where $C = \zeta\varepsilon$ is a constant. Eq.~\eqref{eq:optimal_update_lagrangian} reveals that theoretically optimal unlearning requires subtracting the component of the forget gradient that projects onto the direction of $\overline{g}_r$.

\paragraph{How Lagrangians Relate to GUARD.}
GUARD translates this idea of global subtraction into a sample-wise reweighting mechanism. Instead of computing a single global correction coefficient based on $\langle \overline{g}_r, \overline{g}_f \rangle$, GUARD assesses the conflict for each individual forget sample $i$, i.e.
\[
a_i^{\text{GUARD}} = \langle \overline{g}_r, g_i \rangle,
\]
where $g_i := J_i(\theta_0)$ is the gradient of forget sample $i$. By assigning lower weights $\omega^{\text{GUARD}}_i$ to samples with high scores $a_i^{\text{GUARD}}$, GUARD effectively suppresses the contribution of high-conflict samples to the aggregated update $\overline{g}'_f$. This acts as a \textit{soft realization} of the projection in Eq.~\eqref{eq:optimal_update_lagrangian}: Rather than explicitly subtracting the retain gradient vector, we dampen the source of the conflict itself (the specific samples aligned with $\overline{g}_r$).

%% file: contents/6.exp.tex
\section{Experiments}

\label{sec:exp}

\input{Tables/llama_10}

\textbf{The Datasets.} We evaluate GUARD on \textbf{TOFU}~\citep{maini2024tofu} and \textbf{MUSE}~\citep{shi2024musemachineunlearningsixway} benchmarks for LLM unlearning~\citep{jia2024wagle, ji2024reversing}. TOFU comprises $4{,}000$ Q-A pairs from $200$ synthetic author profiles. The data is split into: 
a) \textit{Forget/Retain Sets}, sampled to unlearn $1\%$, $5\%$, or $10\%$ of the data (the forget01, forget05, and forget10 splits); 
b) \textit{Generalization Sets} (the \textit{Real Author Set} and the \textit{Real World Set}) to evaluate downstream retention and generalization. MUSE provides two real-world corpora: 
a) \textit{NEWS (BBC articles)} with $0.8$M/$1.6$M token Forget/Retain Sets;
b) \textit{BOOKS} with Harry Potter books as a $1.1$M token Forget Set, and FanWiki as $0.5$M token Retain Set. Each corpus includes text and knowledge sets with Q-A pairs derived from the original content.

\textbf{The Backbone LLM.} Following the original setups of the benchmarks, we evaluate unlearning performance using \textit{Phi-1.5B}~\citep{li2023textbooks} and \textit{LLaMA-2-7B}~\citep{touvron2023llama2} on TOFU, \textit{LLaMA-2-7B} on MUSE NEWS corpus, and \textit{ICLM-7B}~\citep{shi2024incontextpretraininglanguagemodeling} on MUSE BOOKS corpus.

\textbf{Baselines.} We compare against four non-loss-reweighting baselines (\textit{GA}~\citep{thudi2022unrolling}, Gradient Difference (\textit{GD})~\citep{liu2022continual}, KL Minimization (\textit{KM}) and Preference Optimization (\textit{PO})~\citep{maini2024tofu}; PO is omitted for MUSE which is text-only), four loss-reweighting methods (\textit{NPO}~\citep{zhang2024negative}, \textit{SimNPO}~\citep{fan2025simplicityprevailsrethinkingnegative}, \textit{FLAT}~\citep{wang2024llm} with the best-performing TV variant, \textit{SatImp}~\citep{yang2025exploringcriterialossreweighting}), and \textit{RMU}~\citep{li2024wmdp}. We apply \textsc{GUARD} on top of each non-loss-reweighting baseline and on top of NPO (chosen for stability); RMU is excluded from integration, since its forget-side term matches the activations of a single layer to a fixed random target, a representation-level objective that lies outside the scope of analysis of Section~\ref{sec:theorem}. Appendices~\ref{sec:apdx-rw},~\ref{subsec:apdx-unlearn-update},~\ref{subsec:config} describe the baselines and hyperparameters.

\textbf{Evaluation Metrics.} We focus on metrics designed for the TOFU and MUSE benchmarks.

\emph{Evaluations of the TOFU dataset:} To measure forgetting, retention, and the trade-off between the two, we evaluate the Absolute Performance and Sacrifice Rate.\\
(a) \textit{Absolute Performance}. Given a dataset $D$ and model $\theta$, we use the Absolute Performance assessment $\epsilon(D; \theta)$ from the TOFU benchmark (ROUGE-L, Probability, Truth Ratio, all reported in percent). For ROUGE-L and Probability, a higher value on the Retain Set/lower value on the Forget Set indicate better retention/forgetting. The Truth Ratio (TR) measures the model's relative probability mass on perturbed answers vs.\ the correct ones; following the standard TOFU convention~\citep{maini2024tofu}, it is reported as $\max(0, 1-\text{TR})$ on the Retain Set and as $\min(\text{TR}, 1/\text{TR})$ on the Forget Set, so that a \emph{higher} reported Retain TR and a \emph{higher} reported Forget TR are preferred.\\
(b) \textit{Sacrifice Rate.} This metric, introduced in Eq.~\eqref{eq:sr}, is considered jointly with the Absolute Performance metric in order to capture both absolute and relative improvements. We denote by $\rho_r$, $\rho_p$, and $\rho_t$ the Sacrifice Rates computed from ROUGE-L, Probability, and Truth Ratio, respectively. Sacrifice Rates are also reported in percent. For the Truth Ratio, a higher value on the Forget Set indicates forgetting, so the sign of the denominator in Eq.~\eqref{eq:sr} is reversed for $\rho_t$.

\input{Tables/abso_llama_05}

\emph{Evaluations of the MUSE dataset:} Following the original MUSE framework, we evaluate methods according to four key criteria: (a) \textit{Forget Set VerbMem.} This criterion measures exact text reproduction from the Forget Set. Lower values indicate better unlearning. (b) \textit{Forget Set KnowMem.} This metric assesses the ability of the model to answer questions about the Forget Set. Lower values indicate better unlearning. (c) \textit{PrivLeak.} This metric quantifies membership inference vulnerability. Values within $[-5\%, 5\%]$ indicate successful unlearning, while values outside this range suggest over/under-unlearning. (d) \textit{Retain Set KnowMem.} This criterion is used to evaluate knowledge retention on the Retain Set. Higher scores indicate better utility preservation.

\textbf{Configuration and Hyperparameters.}  We follow the protocol outlined by \citet{maini2024tofu} for all TOFU experiments: An AdamW optimizer with learning rate $10^{-5}$, weight decay $0.01$, $5$ unlearning epochs, an effective batch size of $32$, and a linear warm-up during the first epoch. All reported TOFU numbers correspond to the checkpoint after the final ($5$-th) unlearning epoch, matching the convention used by the official TOFU leaderboard. For MUSE, we adopt the standard MUSE configuration~\citep{shi2024musemachineunlearningsixway} (AdamW, lr $=10^{-5}$, per-device batch size $32$, up to $10$ unlearning epochs with early stopping when utility drops below the released MUSE \texttt{Retrain} reference, i.e., the model retrained without the Forget Set). When applying GUARD on top of a baseline, we keep the baseline-specific hyperparameters at their published values and introduce only the GUARD temperature $\tau$, isolating the effect of attribution-driven reweighting. We make two explicit changes: for NPO on MUSE, to boost privacy, and for RMU, for which we use the configuration of the \texttt{open-unlearning} codebase\footnote{\url{https://github.com/locuslab/open-unlearning}} (Appendix~\ref{subsec:config}). 

\textbf{Results.} We evaluate the effectiveness of \textsc{GUARD} by focusing on the LLaMA-2-7B backbone, and defer the results for Phi-1.5B to Appendix~\ref{subsec:apdx-exp-phi}, as both models exhibit similar behaviors. We report on experimental results for both TOFU (Sacrifice Rates in Table~\ref{tab:llama-forget10} and Absolute Performance metrics in Table~\ref{tab:abso-llama-05}) and MUSE (four criteria in Table~\ref{tab:unlearning_muse_new}). Due to space limitations, in the main text we focus on results for $10\%$ unlearning on TOFU as this is the most challenging scenario, while results for $1\%$ and $5\%$ data unlearning are delegated to Appendix~\ref{subsec:apdx-exp-llama}. For the same reason, we only present results for MUSE NEWS and delegate the MUSE BOOKS dataset analyses to Appendix~\ref{subsec:apdx-exp-muse-books}.

\textbf{Performance of \textsc{GUARD} on TOFU.} As shown in Table~\ref{tab:llama-forget10}, \textsc{GUARD} consistently lowers the Sacrifice Rate across all baselines and evaluation sets wherever a baseline degrades performance (positive values). Where a baseline improves an evaluation set (negative values), \textsc{GUARD} keeps the change closer to zero. GA+\textsc{GUARD} reduces $\rho_t$ on the Retain Set by up to $30.39\%$ over GA. Table~\ref{tab:abso-llama-05} further shows that on the Retain Set GA+\textsc{GUARD} improves ROUGE-L from $0.18$ to $22.40$ (a $124\times$ increase) and Probability from $0.06\%$ to $17.50\%$, demonstrating substantial recovery from GA collapse while maintaining comparable Forget Set degradation. All baselines are run under the same TOFU protocol with the hyperparameters of their official releases (for RMU, the \texttt{open-unlearning} configuration), so Tables~\ref{tab:llama-forget10}--\ref{tab:abso-llama-05} are direct comparisons. Among all loss-reweighting methods, NPO+\textsc{GUARD} achieves the lowest Retain $\rho_r$ and $\rho_p$ ($14.32$ and $2.87$, reductions of $33.80\%$ and $53.10\%$ over the strongest baseline RMU at $21.63$ and $6.12$).

\textbf{Performance of \textsc{GUARD} on MUSE.} Table~\ref{tab:unlearning_muse_new} shows that GUARD consistently improves retention performance across all baselines while maintaining effective forgetting, with NPO + \textsc{GUARD} achieving the best overall performance on Retain Set KnowMem ($50.0$) compared to all competing loss reweighting methods. We note that \textsc{GUARD} exhibits somewhat higher PrivLeak scores, which reflects the fact that retention of relevant knowledge naturally makes the model more vulnerable to membership inference attacks. Moreover, this moderate privacy degradation (a $5.59\%$ leakage increase from NPO at $93.0$ to NPO + \textsc{GUARD} at $98.2$) is substantially outweighed by the significant gain in knowledge retention (a $34.41\%$ improvement in Retain Set KnowMem, from $37.2$ to $50.0$), with comparable Forget Set KnowMem ($45.2$ vs.\ $44.8$ for NPO).

\input{Tables/muse}

\textbf{Subsampled and Synthetic Retain Sets.} A practical concern is whether \textsc{GUARD} relies on full access to the original Retain Set in order to compute its attribution score $a_i^{\text{GUARD}} = [J_{\text{avg}}^{D_r}]^\top J_i(\theta_0)$. We performed two ablations on TOFU forget10 (LLaMA-2-7B): (i) \emph{Limited Retain}: The gradient $J_{\text{avg}}^{D_r}$ is computed on a $20\%$ randomly sampled subset of $D_r$; and (ii) \emph{Synthetic Retain}: $D_r$ is replaced \emph{entirely} by a GPT-4 generated proxy $\tilde{D}_r$ that only shares the general format and topic area of $D_r$ (for TOFU, the format is QA pairs and the topic is fictional authors), but is nearly content-disjoint from $D_f$. The synthesis prompt only uses this information which is  typically maintained even after compliance-driven data deletion (see Appendix~\ref{subsec:synthetic_retain}). In both cases, the unlearning step still runs on the original full data and only the attribution score is computed on the surrogate. As shown in Table~\ref{tab:robustness_tofu}, in both settings we preserve most of the full-retain improvements of \textsc{GUARD} over NPO: With a $20\%$ retain sampling, NPO+\textsc{GUARD} reduces $\rho_r$ and $\rho_p$ from $31.69$ and $10.97$ (NPO) to $17.78$ and $5.69$ (vs.\ $14.32$ and $2.87$ at full-retain); with the synthetic proxy and no original retain data used in the attribution, $\rho_r$ and $\rho_p$ drop to $19.60$ and $6.85$, close to the result for $20\%$-subsampling and well below NPO. Retain ROUGE-L and Probability follow the same trend, while Forget Set metrics remain comparable across the four settings, confirming that surrogate-based attribution does not compromise forgetting. Experiments on MUSE NEWS and MUSE BOOKS are deferred to Appendix~\ref{subsec:Robustness} (Table~\ref{tab:robustness_muse}).

\input{Tables/robustness_tofu}

\textbf{Alternative Attribution Methods.} We also compare the gradient-product attribution method of \textsc{GUARD} against three alternatives, namely feature-space Embedding Similarity (ES) and two first-order influence-function approximations, RelatIF~\citep{barshan2020relatif} and RIF~\citep{rubinstein2025rescaled}, all integrated into NPO under the same protocol. \textsc{GUARD} consistently achieves the lowest Retain $\rho_r$ ($14.32$ vs.\ RIF's $19.08$ and ES's $26.53$). This result is intuitively justified: The unnormalized gradient dot product preserves both \emph{directional alignment} and \emph{magnitude information,} while, by Lemma~\ref{lemma:retention-if} (Appendix~\ref{subsec:method-if}), dropping the inverse Hessian only rescales the score globally, which is absorbed by the temperature, up to a bounded residual. Experimental per-metric comparisons are reported in  Appendix~\ref{app:alternative_attribution}. For a runtime analysis that shows that the overhead of \textsc{GUARD} is small and incurred only once see Appendix~\ref{subsec:time}. %

%% file: Tables/llama_10.tex
\begin{table*}[ht]
\centering
\caption{Unlearning experiments for LLaMA-2-7B on $10\%$ of TOFU data points. Arrows ($\uparrow$/$\downarrow$) indicate whether higher/lower values are preferred; bold numbers indicate the lower value within each (baseline, baseline+\textsc{GUARD}) pair for GA, KM, PO, and GD, and the lowest value among the loss-reweighting methods, RMU, and NPO+\textsc{GUARD}. \textsc{GUARD} reduces retention degradation compared to baselines.
}
\label{tab:llama-forget10}
\vspace{0.1in}
\fontsize{7pt}{7.5pt}\selectfont
\resizebox{0.75\textwidth}{!}{
\begin{tabular}{lccccccccc}
\toprule
\multirow{2}{*}{\textbf{Methods}}
& \multicolumn{3}{c}{\textbf{Retain Set($\downarrow$)}}
& \multicolumn{3}{c}{\textbf{Real Author Set($\downarrow$)}}
& \multicolumn{3}{c}{\textbf{Real World Set($\downarrow$)}} \\
\cmidrule(lr){2-4} \cmidrule(lr){5-7} \cmidrule(lr){8-10}
& $\rho_{r}$ & $\rho_{p}$ & $\rho_{t}$
& $\rho_{r}$ & $\rho_{p}$ & $\rho_{t}$
& $\rho_{r}$ & $\rho_{p}$ & $\rho_{t}$ \\
\midrule
GA
& 100.17 & 99.87 & 121.09
& 94.80 & 21.71 & 69.67
& 89.72 & 17.63 & 46.36 \\
GA + GUARD
&  \textbf{78.98} &  \textbf{83.46} &  \textbf{84.29}
&  \textbf{79.63} &  \textbf{16.45} &  \textbf{49.90}
&  \textbf{68.39} &  \textbf{12.05} &  \textbf{26.64} \\
\midrule
KM
& 100.05 & 99.86 & 118.67
& 94.71 & 17.92 & 60.88
& 89.64 & 16.72 & 45.51 \\
KM + GUARD
&  \textbf{82.13} &  \textbf{86.05} &  \textbf{89.15}
&  \textbf{81.74} &  \textbf{13.75} &  \textbf{42.74}
&  \textbf{70.21} &  \textbf{12.84} &  \textbf{29.48} \\
\midrule
PO
& 21.84 & 36.72 & 60.00
& 45.02 & 34.20 & 317.41
& 4.23 & 17.35 & 198.15 \\
PO + GUARD
&  \textbf{14.91} &  \textbf{9.95} &  \textbf{7.53}
&  \textbf{42.19} &  \textbf{26.22} &  \textbf{102.47}
&  \textbf{1.01} &  \textbf{2.17} &  \textbf{55.06} \\
\midrule
GD
& 52.41 & 44.10 &  \textbf{-4.94}
& 12.33 & \textbf{-10.41} & \textbf{-43.01}
& \textbf{-0.20} & \textbf{-5.04} & \textbf{-21.37} \\
GD + GUARD
&  \textbf{39.21} &  \textbf{31.87} & -1.15
&  \textbf{9.02} &   -4.70 &   -19.72
&   -0.05 &   -1.66 &   -4.51 \\
\midrule
NPO
& 31.69 & 10.97 & \textbf{-14.84}
& 23.36 & 11.64 & -21.15
& 5.92 & 5.05 & -16.56 \\
SimNPO
& 28.82 & 9.77 & -12.18
& 22.12 & 10.60 & -22.68
& 5.41 & 4.61 & -16.90 \\
FLAT (TV)
& 36.93 & 13.12 & -12.80
& 26.32 & 13.12 & -18.39
& 6.75 & 6.09 & -14.52 \\
SatImp
& 32.86 & 11.72 & -13.56
& 23.32 & 11.89 & \textbf{-23.11}
& 5.95 & 4.69 & \textbf{-18.52} \\
RMU
& 21.63 & 6.12 & -12.64
& 16.60 & 9.87 & -17.98
& 4.69 & 4.18 & -13.95 \\

NPO + GUARD
& \textbf{14.32} & \textbf{2.87} & -6.10
& \textbf{13.59} & \textbf{7.14} & -10.44
& \textbf{2.70} & \textbf{2.23} & -10.29 \\
\bottomrule
\end{tabular}}

\end{table*}

%% file: Tables/abso_llama_05.tex
\begin{table}[t]
\centering
\caption{Impact of \textsc{GUARD} on Absolute Performance for the Retain and Forget Sets for LLaMA-2-7B on $10\%$ of TOFU data points. Arrows indicate the preferred direction following TOFU~ \citep{maini2024tofu}; bold letters indicate the top performers within (baseline, baseline+\textsc{GUARD}) pairs for GA, KM, PO, and GD, and loss-reweighting methods, RMU, and NPO+\textsc{GUARD}. \textsc{GUARD} consistently improves ROUGE-L and Probability on the Retain Set, and also improves the Retain Truth Ratio for GA, KM, and PO.}

\label{tab:abso-llama-05}
\vspace{0.1in}
\fontsize{7pt}{7.5pt}\selectfont
\resizebox{0.75\textwidth}{!}{
\begin{tabular}{lcccccc}
\toprule
\multirow{2}{*}{\textbf{Methods}}
& \multicolumn{3}{c}{\textbf{Retain Set}}
& \multicolumn{3}{c}{\textbf{Forget Set}} \\
\cmidrule(lr){2-4} \cmidrule(lr){5-7}
& ROUGE-L ($\uparrow$) & Prob. (\%) ($\uparrow$) & T.Ratio ($\uparrow$)
& ROUGE-L ($\downarrow$) & Prob. (\%) ($\downarrow$) & T.Ratio ($\uparrow$) \\
\midrule
w/o Unlearn
& 98.57 & 98.95 & 47.47
& 98.54 & 99.09 & 51.60 \\
\midrule
GA
& 0.18 & 0.06 & 8.42
& \textbf{0.32} & \textbf{0.07} & \textbf{83.85} \\
GA + GUARD
& \textbf{22.40} & \textbf{17.50} & \textbf{23.11}
& 2.10 & 1.50 & 80.50 \\
\midrule
KM
& 0.18 & 0.05 & 7.18
& \textbf{0.20} & \textbf{0.05} & \textbf{85.55} \\
KM + GUARD
& \textbf{19.20} & \textbf{14.80} & \textbf{20.10}
& 1.90 & 1.30 & 82.30 \\
\midrule
PO
& 78.20 & 93.85 & 45.85
& 5.25 & 85.20 & 54.30 \\
PO + GUARD
& \textbf{84.50} & \textbf{97.30} & \textbf{46.80}
& \textbf{4.20} & \textbf{82.50} & \textbf{60.50} \\
\midrule
GD
& 47.15 & 55.30 & \textbf{48.95}
& \textbf{0.42} & 0.12 & \textbf{81.55} \\
GD + GUARD
& \textbf{60.30} & \textbf{67.40} & 47.80
& 0.95 & \textbf{0.10} & 80.40 \\
\midrule
NPO
& 78.50 & 92.30 & 49.80
& 35.20 & 38.50 & 67.30 \\
SimNPO
& 80.20 & 93.10 & 49.20
& 34.80 & 39.20 & 65.80 \\
FLAT (TV)
& 76.40 & 91.50 & \textbf{49.85}
& 38.50 & 42.30 & \textbf{70.20} \\
SatImp
& 77.30 & 91.80 & 49.30
& 33.80 & 38.10 & 65.10 \\
RMU
& 84.50 & 95.20 & 49.10
& \textbf{33.50} & \textbf{37.80} & 64.50 \\

NPO + GUARD
& \textbf{89.30} & \textbf{97.20} & 48.30
& 33.80 & 38.20 & 65.20 \\
\bottomrule
\end{tabular}}
\vspace{-1em}
\end{table}

%% file: Tables/muse.tex
\begin{table*}[t]
\centering
\caption{Unlearning experiments with LLaMA-2-7B on the MUSE NEWS dataset, ``w/o Unlearn'' is the released MUSE NEWS target model.}
\label{tab:unlearning_muse_new}
\vspace{0.1in}
\fontsize{7pt}{7.5pt}\selectfont
\resizebox{0.85\textwidth}{!}{
\begin{tabular}{lcccc}
\toprule

\multirow{3}{*}{\textbf{Methods}} & \textbf{Forget Set VerbMem} & \textbf{Forget Set KnowMem} & \textbf{PrivLeak} & \textbf{Retain Set KnowMem} \\
& \textbf{($\downarrow$)} & \textbf{($\downarrow$)} & \textbf{($\in [-5\%, 5\%]$)} & \textbf{($\uparrow$)} \\
\cmidrule(lr){2-5}
\multicolumn{5}{c}{\textbf{NEWS}} \\
\midrule
w/o Unlearn & 58.4 & 63.9 & -99.8 & 55.2 \\
\midrule
GA & \textbf{0.0} & \textbf{0.0} & \textbf{5.6} & 0.0 \\
GA + GUARD & \textbf{0.0} & \textbf{0.0} & 6.4 & \textbf{8.7} \\
\midrule
KM & \textbf{27.8} & 50.1 & \textbf{-95.8} & 43.9 \\
KM + GUARD & 28.3 & \textbf{49.5} & -97.2 & \textbf{50.4} \\
\midrule
GD & 4.6 & 31.8 & \textbf{102.4} & 27.8 \\
GD + GUARD & \textbf{2.7} & \textbf{26.1} & 106.5 & \textbf{44.9} \\
\midrule
NPO & 1.5 & 44.8 & 93.0 & 37.2 \\
SimNPO & \textbf{1.4} & 44.5 & 91.0 & 41.5 \\
FLAT (TV) & 1.7 & \textbf{43.8} & \textbf{84.5} & 33.0 \\
SatImp & 1.6 & 45.0 & 95.5 & 43.9 \\
RMU & 1.5 & 44.0 & 95.0 & 43.1\\
NPO + GUARD & 1.6 & 45.2 & 98.2 & \textbf{50.0} \\
\bottomrule
\end{tabular}}
\end{table*}

%% file: Tables/robustness_tofu.tex
\begin{table}[t]
\centering
\caption{\textsc{GUARD} used with subsampled (random, $20\%$) and synthetic (GPT-4 generated, \emph{no} original retain data) Retain Sets on TOFU forget10 (LLaMA-2-7B). In both cases, \textsc{GUARD} retains its gains over NPO %
and, in Retain ROUGE-L and $\rho_r$, over RMU (which inherently requires knowledge of the Retain Set): The synthetic-retain \textsc{GUARD}, using \emph{no} original retain data for attribution, outperforms RMU at full retain access wrt Retain ROUGE-L and $\rho_r$. The synthetic proxy is competitive with $20\%$ random subsampling.} %
\label{tab:robustness_tofu}
\vspace{0.1in}
\fontsize{7pt}{8pt}\selectfont
\resizebox{\textwidth}{!}{
\begin{tabular}{lcccccccc}
\toprule
\multirow{2}{*}{\textbf{Methods}}
& \multicolumn{5}{c}{\textbf{Retain Set}}
& \multicolumn{3}{c}{\textbf{Forget Set}} \\
\cmidrule(lr){2-6} \cmidrule(lr){7-9}
& ROUGE-L ($\uparrow$) & Prob. (\%) ($\uparrow$) & T.\ Ratio ($\uparrow$) & $\rho_{r}$ ($\downarrow$) & $\rho_{p}$ ($\downarrow$)
& ROUGE-L ($\downarrow$) & Prob. (\%) ($\downarrow$) & T.\ Ratio ($\uparrow$) \\
\midrule
NPO
& 78.50 & 92.30 & \textbf{49.80} & 31.69 & 10.97
& 35.20 & 38.50 & \textbf{67.30} \\
{RMU (full retain)}
& 84.50 & 95.20 & 49.10 & 21.63 & 6.12
& \textbf{33.50} & \textbf{37.80} & 64.50 \\
\midrule
NPO + \textsc{GUARD} (Full Retain)
& \textbf{89.30} & \textbf{97.20} & 48.30 & \textbf{14.32} & \textbf{2.87}
& 33.80 & 38.20 & 65.20 \\
NPO + \textsc{GUARD} ($20\%$ Retain)
& 87.20 & 95.50 & 48.65 & 17.78 & 5.69
& 34.59 & 38.47 & 64.81 \\
NPO + \textsc{GUARD} (Synthetic Retain)
& 86.10 & 94.80 & 48.85 & 19.60 & 6.85
& 34.92 & 38.55 & 64.55 \\
\bottomrule
\end{tabular}}
\end{table}

%% file: contents/appendix.tex
\textbf{\Large Appendix}

\input{contents/a.0.1.theory_da}

\input{contents/a.1.theory_boya}

\input{contents/a.2.algorithm}

\input{contents/a.5.exp}

\input{contents/a.7.limitation}

%% file: contents/a.0.1.theory_da.tex
\section{Related Works}
\label{sec:apdx-rw}

\subsection{Unlearning LLMs}
\label{subsec:apdx-rw-unlearn}
Existing LLM Unlearning approaches broadly fall into two categories: Fine-tuning-based and inference-time-based methods. Fine-tuning approaches such as Gradient Ascent (GA)~\citep{thudi2022unrolling, feng2024fine}, Gradient Difference (GD)~\citep{liu2022continual}, KL Minimization (KM) and Preference Optimization~\citep{maini2024tofu}, and variants thereof, modify model weights by adjusting training objectives, typically through loss maximization on the Forget Set~\citep{zhang2024negative} and incorporation of regularization terms based on KL divergence, retained data consistency, or prompt-guided responses~\citep{xu2025forgetting, bhaila2025soft}. 
A complementary line of work explores inference-time unlearning, which avoids model updates by modifying prompts or logits at inference to suppress specific outputs~\citep{liu2024large, ji2024reversing}. While computationally attractive, these methods do not guarantee actual forgetting, as sensitive information remains embedded in the unchanged model parameters. Emerging directions~\citep{yao2024large, zhang2024catastrophic} also investigate properties beyond forgetting and retention, such as robustness of unlearning mechanisms and recovery from catastrophic forgetting. 

\paragraph{Fine-Tuning-Based Methods.}
Many unlearning techniques modify model weights through loss-based fine-tuning strategies, such as GA~\citep{feng2024fine}, preference optimization (PO)~\citep{maini2024tofu}, negative preference optimization (NPO)~\citep{zhang2024negative}, and representation misdirection~\citep{li2024wmdp}. RMU was introduced together with the WMDP benchmark for hazardous-knowledge removal, and it performs unlearning via representation-level (activation) misdirection while aiming to preserve benign capabilities~\citep{li2024wmdp}. Methods of this type  typically maximize the loss on the Forget Set while incorporating regularization objectives such as KL divergence minimization, joint training with retain data~\citep{xu2025forgetting}, or prompt-based desirable loss surrogates~\citep{bhaila2025soft}. Despite their effectiveness with respect to forgetting, these methods often significantly degrade performance on retained knowledge~\citep{qiu2025survey}. WAGLE~\citep{jia2024wagle} focuses on improving computation efficiency by isolating critical parameters for forgetting via weight attribution, whereas FLAT~\citep{wang2024llm} adjusts the loss using only forget data, maximizing an $f$-divergence between template and forget responses, yet both empirically fail to preserve retention.

\paragraph{Loss-Reweighting Methods.}
\citet{yang2025exploringcriterialossreweighting} analyze loss reweighting mechanisms through saturation and importance-based categorizations. NPO~\citep{zhang2024negative} treats Forget Data as negative responses in preference optimization, providing bounded objectives with adaptive gradient smoothing to prevent catastrophic collapse during unlearning. SimNPO~\citep{fan2025simplicityprevailsrethinkingnegative} further addresses reference model bias in NPO by adopting a reference-free, length-normalized approach that more effectively allocates optimization power across forget samples with varying difficulty levels. Utility-guided methods, such as gradient-based influence modeling~\citep{wang2025rethinking}, propose retention-aware metrics and derive from them the weighted GA (WGA) and token-wise NPO (TNPO) unlearning objectives.

\paragraph{Input-Based and In-Context Methods.}
Input-level methods operate at inference time by manipulating inputs to suppress undesired outputs without altering model weights. These include prompt filtering~\citep{liu2024large} and logit-based interventions~\citep{ji2024reversing}. ECO Prompts~\citep{liu2024large} use prompt classifiers and zeroth-order optimization to induce forgetting, while Unlearning from Logit Difference (ULD)~\citep{ji2024reversing} employs an assistant LLM to reduce side effects. Although computationally efficient, these methods rely on access to sensitive data at inference time, which limits their privacy guarantees~\citep{geng2025comprehensive}.

\paragraph{Alternative Directions.}
Recent efforts explore complementary directions beyond weight-based or input-level interventions. Multimodal unlearning~\citep{li2024single}, adversarial robustness~\citep{schwinn2024soft}, and standardized benchmarks~\citep{wang2025towards} enhance evaluation protocols but do not directly address retention. \citet{liu2025rethinking} survey the design space of LLM unlearning. Other approaches like quantization-based recovery~\citep{zhang2024catastrophic} attempt to restore forgotten knowledge post hoc but remain orthogonal to the goal of preserving retention during unlearning.

\subsection{Data Attribution}
\label{subsec:apdx-rw-da}
Data attribution in machine learning aims to quantify the influence of individual data points on model behavior, supporting tasks such as debugging, data cleaning, and interpretability. Classical approaches, such as influence functions, estimate the effect of removing a single data point on model parameters or predictions. \citet{koh2017understanding} adapted influence functions to deep learning, illustrating their value in identifying mislabeled data and interpreting model decisions. However, subsequent studies revealed notable limitations. For example, \citet{basu2020influence} showed that influence estimates are highly sensitive to hyperparameters and model architecture, particularly in deep networks. \citet{hanawa2021evaluation} further found that these methods can fail basic ``sanity checks,'' attributing high influence to unrelated data. \citet{bae2022if} also demonstrated that influence functions may estimate unintended parameters in deep neural networks, raising concerns about their reliability in complex models.

In response to these reliability issues, several variants have been proposed to refine the analytical influence formulation. To mitigate the tendency of influence functions to select outliers or mislabeled examples, \citet{barshan2020relatif} introduced RelatIF (Relative Influence). By normalizing the influence score against the self-influence of a training sample, RelatIF disentangles global impact from local relevance, prioritizing samples that are directionally aligned with the target prediction. More recently, addressing the approximation errors in high-dimensional settings, \citet{rubinstein2025rescaled} proposed Rescaled Influence Functions (RIF). RIF corrects the standard influence estimate using leverage scores, offering a more accurate prediction of the model change by explicitly accounting for the impact of data removal on the Hessian curvature.

To further address the limitations of classical influence functions, recent work has introduced more robust and scalable data attribution techniques. Among these, direct estimators such as regression-based data models have demonstrated improved predictive accuracy and generalization, particularly in large-scale settings where computational resources are sufficient~\citep{ilyas2022datamodels}. Two state-of-the-art approaches---TRAK~\citep{park2023trak} and EK-FAC~\citep{grosse2023studying}---offer practical solutions across diverse modalities. TRAK approximates a deep neural network with a linear surrogate model and estimates influence by ensembling over multiple model checkpoints and random projections. EK-FAC improves efficiency through a combination of strategies: partitioning the Hessian into block-diagonal components, expressing gradients as Kronecker products of smaller factors, and applying spectral regularization to stabilize curvature estimates.

Despite their empirical success, these methods remain computationally demanding for large-scale models such as large language models (LLMs). Moreover, their underlying attribution philosophy is not directly aligned with the goals of machine unlearning, which demands targeted forgetting and simultaneous preservation of nontargeted knowledge.

\section{Theoretical Guarantees of the Proposed Attribution-Tailored LLM Unlearning Approach}
\label{sec:apdx-proxy-da}

\subsection{Preliminaries: Conventional Data Attribution}
\label{subsec:apdx-DA}
\paragraph{The Leave-One-Out Framework.} Data attribution quantifies the impact of a specific training data point on the loss of the model under the Leave-One-Out framework~\citep{rad2018scalable}. Given a dataset $D$, the influence of a data point $(x_j, y_j)$ on the optimal loss is defined as:
\begin{align*}
    \Delta_L^{j} = \min_{\theta} {L}_D(\theta) - \min_{\theta} {L}_{D \setminus \{(x_j,y_j)\}}(\theta).
\end{align*}

\paragraph{Influence Function.} The \textit{Influence Function}~\citep{koh2017understanding} provides an efficient approximation for the leave-one-out influence. Let $H_D(\theta)$ denote the Hessian matrix summing over data-wise second derivatives:
$
    H_D(\theta) := \sum_{(x_i,y_i) \in D} \nabla^2 l(x_i,y_i;\theta),
$
The influence function for a data point $(x_j, y_j)$ is then given by:
\begin{align}
    a_j^{\text{IF}} = J_{avg}(\theta_0)^\top \cdot [H_D(\theta_0)]^{-1} \cdot J_{j}(\theta_0),
    \label{eq:prelim-IF}
\end{align}
where $J_{j}(\theta) := \nabla_{\theta}l(x_j,y_j;\theta)$ represents the Jacobian of the loss function with respect to the model parameters, and $J_{avg}(\theta) := \frac{1}{|D|}\sum_{(x_i,y_i) \in D} J_i(\theta)$ is its average over $D$.

\subsection{Influence Function Interpretation of the Proxy Score}
\label{subsec:method-if}

Two choices in the design of \textsc{GUARD} deserve justification: dropping the inverse Hessian from the attribution score in Eq.~\eqref{eq:def-J-prod}, and down-weighting forget samples according to their alignment with the Retain Set in Eq.~\eqref{eq:guard-multiplier}. This subsection addresses the first choice, and Section~\ref{sec:theorem} addresses the second.

Compared to the influence function in Eq.~\eqref{eq:prelim-IF}, the proxy attribution score in Eq.~\eqref{eq:def-J-prod} differs in two respects. First, the average gradient is taken over the Retain Set $D_r$ rather than over the whole fine-tuning set $D_0$ (i.e., $D = D_0$ in Eq.~\eqref{eq:prelim-IF}). This change is intended, since unlearning requires preserving the performance on $D_r$ rather than on $D_0$. Second, the inverse Hessian is dropped. To quantify the effect of the second change, we compare $a_j^{\text{GUARD}}$ with the influence function whose average is also taken over $D_r$, which we refer to as the \emph{retention influence}:
\begin{align}
\label{eq:retention-IF}
    a_j^{\text{IF},r} := J^{D_r}_{\text{avg}}(\theta_0)^\top H^{-1} J_j(\theta_0),\quad \forall j\in I_f, \qquad \text{where } H := H_{D_0}(\theta_0).
\end{align}
The retention influence approximates the effect of removing forget sample $j$ on the average loss over the Retain Set, in the same leave-one-out sense as Eq.~\eqref{eq:prelim-IF}. Our analysis relies on the following assumption.

\begin{assumption}
\label{asump:curvature}
The loss function $l(x, y; \theta)$ is twice differentiable with respect to $\theta$, and the Hessian matrix $H$ is symmetric positive definite.
\end{assumption}

\textbf{Remark}: When $H$ is indefinite, as is common at LLM checkpoints, the damped Hessian $H + \mu I$ with $\mu$ large enough that $H + \mu I$ is positive definite~\citep{koh2017understanding} is used in place of $H$, and the results below apply to it.

Let $0 < \lambda_{\min} \le \lambda_{\max}$ denote the smallest and largest eigenvalues of $H^{-1}$, and let $\{v_k\}_k$ be an orthonormal eigenbasis of $H^{-1}$.

\begin{lemma}
\label{lemma:retention-if}
Suppose that Assumption~\ref{asump:curvature} holds. Let $c := (\lambda_{\max} + \lambda_{\min})/2$ and $r_j := a_j^{\text{IF},r} - c \cdot a_j^{\text{GUARD}}$. Then, for all $j \in I_f$,
\begin{align}
\label{eq:retention-if-bound}
    |r_j| \le \frac{\lambda_{\max} - \lambda_{\min}}{2} \cdot \big\|J^{D_r}_{\text{avg}}(\theta_0)\big\|_2 \cdot \big\|J_j(\theta_0)\big\|_2 .
\end{align}
If, in addition, $(v_k^\top J^{D_r}_{\text{avg}}(\theta_0)) \cdot (v_k^\top J_j(\theta_0)) \ge 0$ for all $k$, then $\lambda_{\min} \cdot a_j^{\text{GUARD}} \le a_j^{\text{IF},r} \le \lambda_{\max} \cdot a_j^{\text{GUARD}}$.
\end{lemma}

The proof is provided below.

Write $\bar J := J^{D_r}_{\text{avg}}(\theta_0)$ and $J_j := J_j(\theta_0)$. Let $\{\lambda_k\}_k$ be the eigenvalues of $H^{-1}$ associated with the orthonormal eigenvectors $\{v_k\}_k$, so that $\lambda_k \in [\lambda_{\min}, \lambda_{\max}]$ with $\lambda_{\min} > 0$. We define the eigenbasis components
\[
    q_k := (v_k^\top \bar J) \cdot (v_k^\top J_j), \qquad\text{so that}\qquad \sum_k q_k = \bar J^\top J_j = a_j^{\text{GUARD}}, \qquad \sum_k \lambda_k q_k = \bar J^\top H^{-1} J_j = a_j^{\text{IF},r}.
\]

\begin{proof}
We first prove the bound on the residual. With $c = (\lambda_{\max}+\lambda_{\min})/2$, we have
\[
r_j = a_j^{\text{IF},r} - c \cdot a_j^{\text{GUARD}} = \bar J^\top \big(H^{-1} - c I\big) J_j .
\]
The matrix $H^{-1} - cI$ is symmetric with eigenvalues $\lambda_k - c \in [-(\lambda_{\max}-\lambda_{\min})/2,\, (\lambda_{\max}-\lambda_{\min})/2]$. Therefore, its operator norm is at most $(\lambda_{\max}-\lambda_{\min})/2$, and the Cauchy-Schwarz inequality gives
\[
|r_j| \le \|H^{-1} - cI\|_{\mathrm{op}} \cdot \|\bar J\|_2 \cdot \|J_j\|_2 \le \frac{\lambda_{\max}-\lambda_{\min}}{2} \cdot \|\bar J\|_2 \cdot \|J_j\|_2 .
\]
For the second statement, if $q_k \ge 0$ for all $k$, then $\lambda_{\min} q_k \le \lambda_k q_k \le \lambda_{\max} q_k$ for every $k$. Summing over $k$ gives $\lambda_{\min} \cdot a_j^{\text{GUARD}} \le a_j^{\text{IF},r} \le \lambda_{\max} \cdot a_j^{\text{GUARD}}$.
\end{proof}

Lemma~\ref{lemma:retention-if} shows that the retention influence equals a global rescaling $c \cdot a_j^{\text{GUARD}}$ of the proxy score, with the same constant $c$ for all forget samples, plus a residual $r_j$. Relative to $c \cdot \|J^{D_r}_{\text{avg}}(\theta_0)\|_2 \|J_j(\theta_0)\|_2$, the size of the residual is at most $(\lambda_{\max}-\lambda_{\min})/(\lambda_{\max}+\lambda_{\min})$, which depends only on the condition number of $H$. The global rescaling does not affect \textsc{GUARD}: using $c \cdot a_j^{\text{GUARD}}$ in Eq.~\eqref{eq:guard-multiplier} with temperature $\tau'$ gives the same weights as using $a_j^{\text{GUARD}}$ with temperature $\tau = \tau'/c$. Hence the constant $c$ is absorbed into the temperature and never needs to be estimated, and only the residual changes the weights, each by a factor in $[e^{-2R/\tau'}, e^{2R/\tau'}]$ with $R := \max_{j \in I_f} |r_j|$ (Appendix~\ref{subsec:apdx-residual}). Dropping the inverse Hessian is therefore exact when $H^{-1}$ acts as a multiple of the identity on the span of the gradients, and otherwise leads to a bounded perturbation of the weights.

\subsection{Effect of the Residual on the Unlearning Weights}
\label{subsec:apdx-residual}

Let $\omega_i^{\text{IF}} := n_f \cdot e^{-a_i^{\text{IF},r}/\tau'} / \sum_{k \in I_f} e^{-a_k^{\text{IF},r}/\tau'}$ be the weights obtained from the retention influence at temperature $\tau'$, and let $\omega_i := n_f \cdot e^{-c\,a_i^{\text{GUARD}}/\tau'} / \sum_{k \in I_f} e^{-c\,a_k^{\text{GUARD}}/\tau'}$ be the \textsc{GUARD} weights at temperature $\tau = \tau'/c$. With $R := \max_{k \in I_f} |r_k|$, we have
\[
\frac{\omega_i^{\text{IF}}}{\omega_i}
= e^{-r_i/\tau'} \cdot \frac{\sum_{k} e^{-c\,a_k^{\text{GUARD}}/\tau'}}{\sum_{k} e^{-c\,a_k^{\text{GUARD}}/\tau'}\, e^{-r_k/\tau'}}
\in \big[e^{-2R/\tau'},\, e^{2R/\tau'}\big],
\]
since the first factor lies in $[e^{-R/\tau'}, e^{R/\tau'}]$ and the second factor is the reciprocal of a weighted average of the terms $e^{-r_k/\tau'}$, which lie in the same interval. Hence, the global constant $c$ is absorbed into the temperature, and only the residual changes the weights, multiplicatively and by a bounded factor. \textsc{GUARD} does not need to estimate $c$, since $\tau$ is a hyperparameter.

%% file: contents/a.1.theory_boya.tex
\section{Proof of Results in Section \ref{sec:theorem}}
\label{sec:apdx-proof-sr}

\subsection{Proof of Theorem \ref{lemma:r-loss-GUARD-GA}}

Recall that the losses of different models on $D_r$ are defined as
\begin{align*}
    L^{0}_{D_r} &= \frac{1}{n_r} \sum_{i \in I_r} l(x_i, y_i; \theta^{0}),\quad
    L^{\text{GA}}_{D_r} = \frac{1}{n_r} \sum_{i \in I_r} l(x_i, y_i; \theta^{GA}),\quad
     L^{\text{GUARD}}_{D_r} = \frac{1}{n_r} \sum_{i \in I_r} l(x_i, y_i; \theta^{\text{GUARD}}).
\end{align*}
The model update rules given by GA and by GUARD are, 
\begin{align*}
    \theta^{\text{GA}} = \theta^0 + \eta \cdot\frac{1}{n_f}\sum_{i \in I_f} g_i,\quad
    \theta^{\text{GUARD}}  = \theta^0 + \eta \cdot \frac{1}{n_f}\sum_{i\in I_f} [\omega_i^{\text{GUARD}} \cdot g_i],
\end{align*}
where $\omega_i^{\text{GUARD}}$ for $i \in I_f$ is defined in \eqref{eq:guard-multiplier}.
Also recall that $a_j = \langle \bar g_r, g_j \rangle$, and thus we have 
\begin{align}
\frac{1}{n_f}\sum_{j\in I_f} a_j=\kappa, \quad
\frac{1}{n_f}\sum_{j\in I_f} a_j^2
 = \sigma_\kappa^2 + \kappa^2.
\label{eq.kappas}
\end{align}
From the loss definitions and model update rules, we get the approximations of the increase in the retention loss compared with the model before unlearning. Specifically, both updates have norm at most $\eta G$, since the GUARD weights are nonnegative and average to one. Hence, by Assumption~\ref{asum:small-update}, using a first-order Taylor expansion around $\theta_0$, the increase in the loss on the $D_r$ induced by GA is given by the first-order approximation
\begin{align}
    L_{D_r}^{\text{GA}} -  L_{D_r}^{\text{0}}& =\frac{1}{n_r} \sum_{i\in I_r} g_i^\top \cdot [\theta^{\text{GA}} - \theta^0] + O(\eta^2 G^2) \notag\\
    & = \frac{1}{n_r} \sum_{i\in I_r} g_i^\top \cdot [\eta \cdot\frac{1}{n_f}\sum_{j \in I_f} g_j] + O(\eta^2 G^2) \notag\\
    & = \eta \cdot \langle \overline{g}_f , \overline{g}_r \rangle + O(\eta^2 G^2) \notag\\
    & = \eta \cdot \kappa + O(\eta^2 G^2).
    \label{eq:dlr_GA}
\end{align}

Similarly, the increase in the loss on $D_r$ induced by GUARD is given by 
\begin{align}
    L_{D_r}^{\text{GUARD}} -  L_{D_r}^{\text{0}}& = \frac{1}{n_r} \sum_{i\in I_r} g_i^\top \cdot [\theta^{\text{GUARD}} - \theta^0]+ O(\eta^2 G^2)\notag\\
    & = \frac{1}{n_r} \sum_{i\in I_r} g_i^\top \cdot [\eta \cdot \frac{1}{n_f}\sum_{j\in I_f} [\omega_j^{\text{GUARD}} \cdot g_j]]+ O(\eta^2 G^2).
\end{align}
Writing $\overline{g}'_f := \frac{1}{n_f}\sum_{i\in I_f} \left[\omega_i^{\text{GUARD}} \cdot g_i\right]$, we can express the above equation as
\begin{align}
\begin{aligned}
L_{D_r}^{\text{GUARD}} -  L_{D_r}^{\text{0}}
=\eta \cdot \langle \overline{g}_r , \overline{g}'_f \rangle + O(\eta^2 G^2).
\end{aligned}
\label{eq:dlr_GUARD_0}
\end{align}

Using the Taylor expansion $e^{-x} = 1 - x + O(x^2)$,  $\omega_i^{\text{GUARD}}$ becomes
\begin{align}
\begin{aligned}
\omega_i^{\text{GUARD}} 
=& n_f \frac{e^{-a_i/\tau}}{\sum_{k\in I_f} e^{-a_k/\tau}}\\
=& n_f \frac{1-a_i/\tau + O(\delta_a^2) }{\sum_{k\in I_f} \left(1-a_k/\tau +  O(\delta_a^2) \right)} \\
=& \frac{1-a_i/\tau + O(\delta_a^2) }{1 -\frac{\kappa }{\tau} +  O(\delta_a^2)} \\
=& \frac{1-a_i/\tau}{ 1 -\frac{\kappa }{\tau}}+  O(\delta_a^2),\\
\end{aligned}
\end{align}
where we use $|a_k|/\tau \le \delta_a$ for all $k \in I_f$, and the last line holds since $1-\kappa/\tau \ge 1-\delta_a$ is bounded away from zero under Assumption \ref{asum:r-f-entangle}, as $\kappa$ is the average of the $a_j$. We denote the remainder in the last line by $R_i$, so that $|R_i| = O(\delta_a^2)$. Based on the above expression, and since $\langle \overline{g}_r, g_j\rangle = a_j$, we can simplify the first term in \eqref{eq:dlr_GUARD_0} as
\begin{align}
\begin{aligned}
    \langle \overline{g}_r , \overline{g}'_f \rangle 
    =&\frac{1}{1 - \kappa/\tau}
    \left[
    \frac{1}{n_f}\sum_{j\in I_f} a_j
    -\frac{1}{\tau}\frac{1}{n_f}\sum_{j\in I_f} a_j^2
    \right]
    + O(\tau\,\delta_a^3)\\
    \stackrel{(a)}{=}&\frac{
    \kappa - \frac{1}{\tau}(\sigma_\kappa^2 + \kappa^2)
    }{1 - \kappa/\tau}+ O(\tau\,\delta_a^3)\\
    =&\kappa-\frac{\sigma_\kappa^2}{\tau - \kappa}
    + O(\tau\,\delta_a^3),
\end{aligned}
\label{eq:dlr_GUARD_1}
\end{align}
where (a) follows from \eqref{eq.kappas}, and the remainder equals $\frac{1}{n_f}\sum_{j\in I_f} R_j a_j = O(\tau\,\delta_a^3)$, since $|a_j| \le \tau\delta_a$. 
Substituting \eqref{eq:dlr_GUARD_1} into \eqref{eq:dlr_GUARD_0}, we have
\begin{align}
\begin{aligned}
L_{D_r}^{\text{GUARD}} -  L_{D_r}^{\text{0}}
=\eta \left(\kappa-\frac{\sigma_\kappa^2}{\tau - \kappa} \right)  + \eta \cdot O(\tau\,\delta_a^3) + O(\eta^2 G^2).
\end{aligned}
\label{eq:dlr_GUARD}
\end{align}
Hence, we can conclude that
\begin{align*}
     L_{D_r}^{\text{GA}} - L_{D_r}^{\text{GUARD}} =(1- \delta_{\kappa})^{-1} \cdot \eta/\tau \cdot \sigma_\kappa^2 + \eta \cdot O(\tau\,\delta_a^3) + O(\eta^2 G^2).
\end{align*}

\subsection{Proof of Theorem \ref{lemma:f-loss-GUARD-GA}}

Recall that the losses of different models on the Forget Set are defined as
\begin{align*}
    L^{0}_{D_f} &= \frac{1}{n_f} \sum_{i \in I_f} l(x_i, y_i; \theta^{0}),\\
    L^{\text{GA}}_{D_f} &= \frac{1}{n_f} \sum_{i \in I_f} l(x_i, y_i; \theta^{GA}),\\
     L^{\text{GUARD}}_{D_f} &= \frac{1}{n_f} \sum_{i \in I_f} l(x_i, y_i; \theta^{\text{GUARD}}).
\end{align*}

Correspondingly, the model update rules given by GA and by GUARD are, respectively,
\begin{align*}
    \theta^{\text{GA}} &= \theta^0 + \eta \cdot\frac{1}{n_f}\sum_{i \in I_f} g_i\\
    \theta^{\text{GUARD}} & = \theta^0 + \eta \cdot \frac{1}{n_f}\sum_{i\in I_f} [\omega_i^{\text{GUARD}} \cdot g_i].
\end{align*}

As in the proof of Theorem~\ref{lemma:r-loss-GUARD-GA}, both updates have norm at most $\eta G$, so by Assumption~\ref{asum:small-update} the second-order Taylor remainders are $O(\eta^2 G^2)$. Hence, based on loss definitions and model update rules, we get the approximations of the increase in the forget loss compared with the model before unlearning. Specifically, the increase in the loss on the Forget Set induced by GA is given by
\begin{align}
    L_{D_f}^{\text{GA}} -L_{D_f}^{0} & =\frac{1}{n_f} \sum_{i\in I_f} g_i^\top \cdot [\theta^{\text{GA}} - \theta^0] + O(\eta^2 G^2)\notag\\
    & = \frac{1}{n_f} \sum_{i\in I_f} g_i^\top \cdot [\eta \cdot\frac{1}{n_f}\sum_{j \in I_f} g_j] + O(\eta^2 G^2)\notag \\
    & = {\eta}\cdot  \langle \overline{g}_f, \overline{g}_f \rangle+ O(\eta^2 G^2), 
    \label{eq:dlf_GA}
\end{align}
and the increase in the loss on the Forget Set induced by GUARD is given by 
\begin{align}
   L_{D_f}^{\text{GUARD}}-L_{D_f}^{0} & = \frac{1}{n_f} \sum_{i\in I_f} g_i^\top \cdot [\theta^{\text{GUARD}} - \theta^0] + O(\eta^2 G^2) \notag\\
    & = \frac{1}{n_f} \sum_{i\in I_f} g_i^\top \cdot \left[\eta \cdot \frac{1}{n_f}\sum_{j\in I_f} [\omega_j^{\text{GUARD}} \cdot g_j]\right]+ O(\eta^2 G^2) \notag\\
    & = \eta \cdot \langle \overline{g}'_f, \overline{g}_f \rangle + O(\eta^2 G^2),
    \label{eq:dlf_GUARD}
\end{align}
where $\overline{g}'_{f}$ is 
the weighted average gradient over Forget Set, $\overline{g}'_{f} := \frac{1}{n_f} \sum_{i\in I_f} [\omega_i^{\text{GUARD}} \cdot g_i]$. Also,
\begin{align}
    \langle \overline{g}_f, \overline{g}'_f\rangle & = \left[\frac{1}{n_f}\sum_{i\in I_f} g_i\right]^\top \cdot \left[ \frac{1}{n_f}\sum_{j\in I_f}[\omega_j^{\text{GUARD}} \cdot g_j]\right] \notag\\
    & = \frac{1}{{n_f}^2} \sum_{i,j\in I_f} \left[n_f \cdot \frac{e^{-a^{\text{GUARD}}_j/\tau}}{\sum_{k\in I_f} e^{-a^{\text{GUARD}}_k/\tau}}\cdot g_i^\top \cdot g_j\right]\notag\\
    & = \frac{1}{n_f} \sum_{i,j \in I_f} \left[ \frac{1-\overline{g}_r^\top \cdot g_j/\tau}{\sum_{k\in I_f}[1-\overline{g}_r^\top \cdot g_k/\tau] }  \cdot g_i^\top \cdot g_j\right]+ O(\delta_a^2 G^2) \notag\\
    & = \frac{1}{{n_f}^2} \cdot \frac{\tau}{\tau - \kappa} \sum_{i,j \in I_f} \left[[1 - \overline{g}_r^\top \cdot g_j / \tau] \cdot g_i^\top \cdot g_j\right]+ O(\delta_a^2 G^2) \notag\\
    & = \frac{\tau}{{n_f} (\tau - \kappa)} \left[{n_f}\|\overline{g}_f\|_2^2 - \frac{1}{\tau}\sum_{ j \in I_f}\left[\overline{g}_r^\top \cdot g_j \cdot g_j^\top \cdot \overline{g}_f  \right]\right]+ O(\delta_a^2 G^2). 
    \label{eq:dlf_GUARD_1}
\end{align}

Here, the third line uses the expansion of $\omega_j^{\text{GUARD}}$ from the proof of Theorem~\ref{lemma:r-loss-GUARD-GA}, and its remainder equals $\frac{1}{n_f}\sum_{j\in I_f} R_j\, g_j^\top \overline{g}_f = O(\delta_a^2 G^2)$, since $|g_j^\top \overline{g}_f| \le G^2$. Let \( d \) be the dimension of the gradient vectors. Then we can write
\begin{align}
\begin{aligned}
    \sum_{j \in I_f} \overline{g}_r^\top g_j \cdot g_j^\top \overline{g}_f
    &= \overline{g}_r^\top \left( \sum_{j \in I_f} g_j g_j^\top \right) \overline{g}_f \notag \\
    &= n_f \cdot \overline{g}_r^\top \Sigma_f \overline{g}_f \notag\\
    &=n_f  \bar{g}_r^\top \left( \mathrm{Cov}_f + \bar{g}_f \bar{g}_f^\top \right) \bar{g}_f\\
    &= n_f  \bar{g}_r^\top \mathrm{Cov}_f \bar{g}_f 
+ n_f \, \langle \bar{g}_r, \bar{g}_f \rangle \, \|\bar{g}_f\|_2^2 \\
    &= n_f  \bar{g}_r^\top \mathrm{Cov}_f \bar{g}_f 
+ n_f \, \kappa \, \|\bar{g}_f\|_2^2.
\end{aligned}
\end{align}

Thus, we have
\begin{align}
\begin{aligned}
\langle \overline{g}'_f , \overline{g}_f\rangle
=&\frac{\tau}{\tau-\kappa} \|\bar{g}_f\|_2^2
- \frac{1}{\tau-\kappa} \, \bar{g}_r^\top \Sigma_f \bar{g}_f + O(\delta_a^2 G^2)\\
=& \frac{\tau}{\tau-\kappa} \|\bar{g}_f\|_2^2
- \frac{1}{\tau-\kappa} \left( \kappa \|\bar{g}_f\|_2^2 + \bar{g}_r^\top \mathrm{Cov}_f \bar{g}_f \right) + O(\delta_a^2 G^2)\\
=& \|\bar{g}_f\|_2^2 - \frac{1}{\tau-\kappa}  \bar{g}_r^\top \mathrm{Cov}_f \bar{g}_f
+ O(\delta_a^2 G^2).
\end{aligned}
\label{eq:dlf_GUARD_4}
\end{align}

Substituting Eq. \eqref{eq:dlf_GUARD_4} into Eq.~\eqref{eq:dlf_GUARD}, we get
\begin{align}
\begin{aligned}
L_{D_f}^{\text{GUARD}} - L_{D_f}^{\text{0}} 
=\eta \langle \bar{g}_f', \bar{g}_f \rangle + O(\eta^2 G^2)
=\eta \|\bar{g}_f\|_2^2 - \frac{\eta}{\tau-\kappa} \, \bar{g}_r^\top \mathrm{Cov}_f \bar{g}_f + \eta \cdot O(\delta_a^2 G^2) + O(\eta^2 G^2).
\end{aligned}
\label{eq:dlf_GUARD_5}
\end{align}

Hence, using Eq.~\eqref{eq:dlf_GUARD_5} and Eq.~\eqref{eq:dlf_GA}, we have
\begin{align}
L_{D_f}^{\text{GA}} - L_{D_f}^{\text{GUARD}}
= \frac{\eta}{\tau-\kappa} \, \bar{g}_r^\top \mathrm{Cov}_f \bar{g}_f + \eta \cdot O(\delta_a^2 G^2) + O(\eta^2 G^2).
\end{align}

Under Assumption \ref{asum:r-f-entangle}, we have $\tau > \kappa$ (since $\kappa/\tau \le \delta_a \ll 1$), and we have
\begin{align}
\begin{aligned}
\left|L_{D_f}^{\text{GA}} - L_{D_f}^{\text{GUARD}}\right|  
=& \frac{\eta}{\tau-\kappa} \left| \bar{g}_r^\top \mathrm{Cov}_f \bar{g}_f \right| + \eta \cdot O(\delta_a^2 G^2) + O(\eta^2 G^2)\\
\leq & \frac{\eta}{\tau-\kappa} \sqrt{\bar{g}_r^\top \mathrm{Cov}_f \bar{g}_r}
\sqrt{\bar{g}_f^\top \mathrm{Cov}_f \bar{g}_f}+ \eta \cdot O(\delta_a^2 G^2) + O(\eta^2 G^2),
\end{aligned}
\end{align}
where the last line follows from the Cauchy-Schwarz inequality
\begin{align}
\left| \bar{g}_r^\top \mathrm{Cov}_f \bar{g}_f \right|
\leq \sqrt{\bar{g}_r^\top \mathrm{Cov}_f \bar{g}_r}
\sqrt{\bar{g}_f^\top \mathrm{Cov}_f \bar{g}_f}.
\label{eq.Cov.CS}
\end{align}

By definition of $\mathrm{Cov}_f$, we have
\begin{align}
\begin{aligned}
\bar{g}_r^\top \mathrm{Cov}_f \bar{g}_r
=& \frac{1}{n_f} \sum_{i \in I_f} \bar{g}_r^\top
(g_i - \overline{g}_f)(g_i - \overline{g}_f)^\top \bar{g}_r\\
=&  \frac{1}{n_f} \sum_{i \in I_f} 
\left( \bar{g}_r^\top (g_i - \overline{g}_f) \right)^2\\
=& \frac{1}{n_f}  \sum_{i \in I_f}
\left(a_i - \kappa\right)^2\\
=& \sigma_\kappa^2.
\end{aligned}
\label{eq.Cov.sigma}
\end{align}
Thus, 
\begin{align}
\left| \bar{g}_r^\top \mathrm{Cov}_f \bar{g}_f \right|
\leq \sqrt{\bar{g}_r^\top \mathrm{Cov}_f \bar{g}_r}
\sqrt{\bar{g}_f^\top \mathrm{Cov}_f \bar{g}_f}
\leq \sigma_\kappa \sqrt{\bar{g}_f^\top \mathrm{Cov}_f \bar{g}_f}.
\label{eq.Cov.bound}
\end{align}

We can further bound the RHS using the spectrum of $\mathrm{Cov}_f$. Specifically, we have 
\begin{align}
\begin{aligned}
\bar{g}_f^\top \mathrm{Cov}_f \bar{g}_f
\leq& \lambda_{\max}(\mathrm{Cov}_f)\, \|\bar{g}_f\|_2^2,
\end{aligned}
\label{eq.spectrum}
\end{align}
and also
\begin{align}
\begin{aligned}
\left| \bar{g}_r^\top \mathrm{Cov}_f \bar{g}_f \right|
\leq \sigma_\kappa \sqrt{\lambda_{\max}(\mathrm{Cov}_f)} \|\bar{g}_f\|_2.
\end{aligned}
\label{eq.Cov.bound.2}
\end{align}

Hence, we conclude 
\begin{align}
\left| L^{\text{GA}}_{D_f} - L^{\text{GUARD}}_{D_f} \right|
\leq \frac{\eta \sigma_\kappa}{\tau-\kappa}  \sqrt{\lambda_{\max}(\mathrm{Cov}_f)} \|\bar{g}_f\|_2
+ \eta \cdot O(\delta_a^2 G^2) + O(\eta^2 G^2).
\end{align}

\subsection{Proof of Theorem \ref{thrm:SR-GUARD-GA}}

Throughout this proof, we work to leading order, i.e., we drop the remainder terms in Theorems~\ref{lemma:r-loss-GUARD-GA} and~\ref{lemma:f-loss-GUARD-GA}, and $\rho^{\text{GA}}$ and $\rho^{\text{GUARD}}$ denote the resulting leading-order Sacrifice Rates. From Eq.~\eqref{eq:dlr_GA} and Eq.~\eqref{eq:dlf_GA}, we have 
\begin{align}
\begin{aligned}
\rho^{\text{GA}} 
= \frac{ L_{D_r}^{\text{GA}} - L_{D_r}^{\text{0}}}{ L_{D_f}^{\text{GA}} - L_{D_f}^{\text{0}}} 
= \frac{\eta \cdot \kappa }{{\eta}\cdot  \langle \overline{g}_f, \overline{g}_f \rangle}
= \frac{\kappa}{\|\overline{g}_f\|_2^2} ,
\end{aligned}
\label{eq:sr_GA}
\end{align}
which is well defined since $\overline{g}_f \neq 0$ under condition~\eqref{eq:suff-cond}.

To compute $\rho^{\text{GUARD}}$, using Eq. \eqref{eq:dlr_GUARD} and Eq.~\eqref{eq:dlf_GUARD_5} we get
\begin{align}
\begin{aligned}
    \rho^{\text{GUARD}} 
    & = \frac{ L_{D_r}^{\text{GUARD}} - L_{D_r}^{\text{0}}}{ L_{D_f}^{\text{GUARD}} - L_{D_f}^{\text{0}}} \\
    & = \frac{\eta \left(\kappa-\frac{\sigma_\kappa^2}{\tau - \kappa} \right)  }
{\eta\left(\|\bar g_f\|_2^2 - \dfrac{1}{\tau-\kappa}\bar g_r^\top \mathrm{Cov}_f \bar g_f\right)}  \\
    & = \frac{\kappa-\frac{\sigma_\kappa^2}{\tau - \kappa}}
{\|\bar g_f\|_2^2 - \dfrac{1}{\tau-\kappa}\bar g_r^\top \mathrm{Cov}_f \bar g_f} , \\
\end{aligned}
\label{eq:sr_GUARD}
\end{align}
which is well defined if its denominator is positive. Using \eqref{eq.Cov.bound.2}, a sufficient condition for this to hold true is 
\begin{align}
\|\bar g_f\|_2^2 > \frac{\sigma_\kappa}{\tau-\kappa}   \sqrt{\lambda_{\max}(\mathrm{Cov}_f)} \|\bar{g}_f\|_2.  
\end{align}
That is, if $\|\bar g_f\|_2 > \frac{\sigma_\kappa}{\tau-\kappa}   \sqrt{\lambda_{\max}(\mathrm{Cov}_f)}$, then we have
\begin{align}
\begin{aligned}
\rho^{\mathrm{GA}} - \rho^{\mathrm{GUARD}} 
=
\frac{
\|\bar g_f\|_2^2 \sigma_\kappa^2- \kappa \, \bar g_r^\top \mathrm{Cov}_f \bar g_f }{(\tau - \kappa)\,\|\bar g_f\|_2^2 \left(
\|\bar g_f\|_2^2 -\frac{1}{\tau - \kappa}\,\bar g_r^\top \mathrm{Cov}_f \bar g_f\right)}
.
\end{aligned}
\label{eq:sr-diff}
\end{align}

From the above expression, GUARD is better than GA (to leading order), when $\rho^{\mathrm{GA}} - \rho^{\mathrm{GUARD}} \geq 0$. We next discuss two cases.
First, when $\sigma_\kappa = 0$, $a_j =\kappa$ for all $j \in I_f$, i.e., every forget gradient has the same alignment with the retain gradient. In this case, the GUARD weights become uniform, and the GUARD reduces to the GA updates.
When $\sigma_\kappa \neq 0$, since we already ensured that the denominator is positive, this reduces to the numerator being nonnegative. Using \eqref{eq.Cov.bound.2}, a sufficient condition is
\begin{align}
\|\bar g_f\|_2 \geq 
\frac{\kappa}{\sigma_\kappa}  \sqrt{\lambda_{\max}(\mathrm{Cov}_f)}.   
\end{align}
The denominator is positive if $\|\bar g_f\|_2 > \frac{\sigma_\kappa}{\tau-\kappa}  \sqrt{\lambda_{\max}(\mathrm{Cov}_f)}$.
Putting these together, a sufficient condition is given by
\begin{align}
\|\bar g_f\|_2
> 
\max \left\{
\frac{\sigma_\kappa}{\tau-\kappa},
\frac{\kappa }{\sigma_\kappa}\right\}   \sqrt{\lambda_{\max}(\mathrm{Cov}_f)}.   
\end{align}
This is condition~\eqref{eq:suff-cond} in Theorem~\ref{thrm:SR-GUARD-GA}.

Under the spiked covariance model with $\mathrm{Cov}_f = \nu^2 I + U \Lambda U^\top$, $\lambda_{\max}(\mathrm{Cov}_f) = \nu^2 + \lambda_{\max}(\Lambda)$. We hence obtain a sufficient condition for $\rho^{\mathrm{GA}} - \rho^{\mathrm{GUARD}} \geq 0$ as
\begin{align}
\|\bar g_f\|_2
> 
\max \left\{
\frac{\sigma_\kappa}{\tau-\kappa},
\frac{\kappa }{\sigma_\kappa}\right\}
\sqrt{\nu^2 + \lambda_{\max}(\Lambda)}.   
\end{align}
To obtain a sharper bound, note that 
\begin{align}
\bar g_f^\top \mathrm{Cov}_f \bar g_f 
=\nu^2 \|\bar g_f\|_2^2
+\bar g_f^\top U \Lambda U^\top \bar g_f
\leq\nu^2 \|\bar g_f\|_2^2
+\lambda_{\max}(\Lambda)\|U^\top \bar g_f\|_2^2.
\end{align}
Hence,
\begin{align}
\begin{aligned}
\left|\bar g_r^\top \mathrm{Cov}_f \bar g_f\right|
\leq \sigma_\kappa
\sqrt{\nu^2 \|\bar g_f\|_2^2
+\lambda_{\max}(\Lambda)\|U^\top \bar g_f\|_2^2}.
\end{aligned}
\end{align}
Thus, by the argument for the numerator above, the second condition in~\eqref{eq:suff-cond} can be relaxed to 
\begin{align}
\| \bar g_f\|_2^2
\sigma_\kappa
\geq
\kappa \sqrt{\nu^2 \|\bar g_f\|_2^2 +
\lambda_{\max}(\Lambda)\|U^\top \bar g_f\|_2^2}.
\end{align}

Dividing both sides by $\|\bar g_f\|_2$ and writing $\cos^2\phi := \|U^\top \bar g_f\|_2^2/\|\bar g_f\|_2^2$ shows that this is condition~\eqref{eq:spiked}.

\paragraph{The Special Case of Isotropic Gradients.} Let $\mathrm{Cov}_f = \nu^2 I$ for some $\nu > 0$, i.e., $\Lambda = 0$ in the spiked model. By~\eqref{eq.Cov.sigma}, $\sigma_\kappa^2 = \bar g_r^\top \mathrm{Cov}_f \bar g_r = \nu^2 \|\bar g_r\|_2^2$, which is positive since $\kappa > 0$ implies $\bar g_r \neq 0$. Hence, by Theorem~\ref{lemma:r-loss-GUARD-GA}, to first order \textsc{GUARD} reduces the loss on the Retain Set by $\eta\, \nu^2 \|\bar g_r\|_2^2/(\tau-\kappa)$ compared to GA. Moreover, $\bar g_r^\top \mathrm{Cov}_f \bar g_f = \nu^2 \kappa$, so the numerator of~\eqref{eq:sr-diff} equals $\nu^2\big(\|\bar g_f\|_2^2 \|\bar g_r\|_2^2 - \kappa^2\big)$, which is nonnegative by the Cauchy-Schwarz inequality in $\mathbb{R}^d$. Thus the second condition in~\eqref{eq:suff-cond} is not needed. The denominator of~\eqref{eq:sr-diff} equals $(\tau-\kappa)\|\bar g_f\|_2^2\big(\|\bar g_f\|_2^2 - \nu^2\kappa/(\tau-\kappa)\big)$, which is positive if and only if $\|\bar g_f\|_2^2 > \nu^2 \kappa/(\tau - \kappa)$. Hence, under this condition, the leading-order term of $\rho^{\text{GA}} - \rho^{\text{GUARD}}$ is nonnegative.

\subsection{Empirical Verification of the Sufficient Condition on TOFU}
\label{subsec:apdx-thrm-verify}

The sacrifice rate guarantee of Theorem~\ref{thrm:SR-GUARD-GA} relies on the sufficient condition
\begin{align}
\|\bar g_f\|_2 \;>\; \max\!\left\{\frac{\sigma_\kappa}{\tau-\kappa},\;\frac{\kappa}{\sigma_\kappa}\right\} \sqrt{\lambda_{\max}(\mathrm{Cov}_f)},
\label{eq:apdx-suff-cond}
\end{align}
together with Assumption~\ref{asum:r-f-entangle} (positive global alignment $\kappa>0$ and weak per-sample alignment $|a_j|/\tau \ll 1$). To check if this condition is realistic in  the LLM domains that \textsc{GUARD} operates, we estimate every quantity in~\eqref{eq:apdx-suff-cond} using the full TOFU forget10 split and random subsamples of the retain split on the LLaMA-2-7B checkpoint fine-tuned under the official TOFU protocol, using the same temperature $\tau = 0.03$ adopted throughout the paper.

\paragraph{Estimation protocol.} For all $n_f = 400$ forget points of the TOFU forget10 split and a random subsample of $n_r = 400$ retain points:
\begin{enumerate}
    \item Compute per-sample gradients $\{g_i\}_{i\in I_f}, \{g_j\}_{j\in I_r}$ at the fine-tuned checkpoint $\theta_0$.
    \item Form the means $\bar g_f = \frac{1}{n_f}\sum_{i\in I_f} g_i$, $\bar g_r = \frac{1}{n_r}\sum_{j\in I_r} g_j$.
    \item Compute scalars $\kappa = \langle \bar g_f, \bar g_r \rangle$, $\|\bar g_f\|_2$, and per-sample alignments $a_j = \langle \bar g_r, g_j\rangle$ for $j\in I_f$.
    \item Compute the alignment variance $\sigma_\kappa^2 = \frac{1}{n_f}\sum_j a_j^2 - \kappa^2$.
    \item Form the centered forget-gradient matrix $\tilde G \in \mathbb{R}^{n_f \times d}$ with rows $g_i - \bar g_f$, and obtain $\lambda_{\max}(\mathrm{Cov}_f)$ as the largest eigenvalue of the $n_f\times n_f$ Gram matrix $\frac{1}{n_f}\tilde G \tilde G^\top$ (the non-zero spectrum of $\mathrm{Cov}_f = \frac{1}{n_f}\tilde G^\top \tilde G$ is identical to that of $\frac{1}{n_f}\tilde G \tilde G^\top$, so the eigendecomposition is performed over $\mathbb{R}^{n_f\times n_f}$ rather than $\mathbb{R}^{d\times d}$).
\end{enumerate}
The check on MUSE BOOKS in Figure~\ref{fig:suff-cond-muse-books} follows the same estimation steps.

\paragraph{Empirical evaluation over 50 random subsamples.} We repeat the protocol above for $50$ independent random retain subsamples of size $n_r = 400$, using all $n_f = 400$ forget points in each trial. For each subsample we compute every quantity in~\eqref{eq:apdx-suff-cond} and record the gap $\mathrm{LHS} - \mathrm{RHS} = \|\bar g_f\|_2 - \max\{\sigma_\kappa/(\tau-\kappa),\ \kappa/\sigma_\kappa\}\sqrt{\lambda_{\max}(\mathrm{Cov}_f)}$; a positive gap indicates that the sufficient condition is satisfied for that subsample. The empirical distribution of the gap across the $50$ trials is shown in Figure~\ref{fig:suff-cond-dist}.

\begin{figure}[h]
    \centering
    \includegraphics[width=0.85\linewidth]{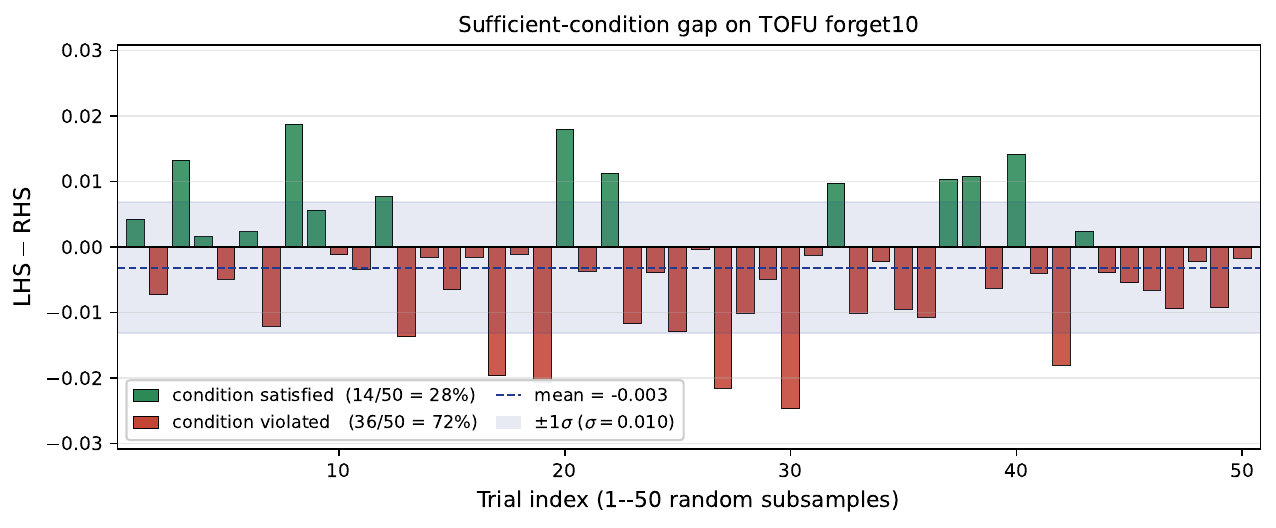}
    \caption{Empirical distribution of $\mathrm{LHS} - \mathrm{RHS}$ in the sufficient condition~\eqref{eq:apdx-suff-cond} across $50$ trials, each using the full TOFU forget10 split ($n_f = 400$) and a random retain subsample ($n_r = 400$) (LLaMA-2-7B, $\tau = 0.03$). Green bars indicate trials where the sufficient condition is satisfied, red bars indicate violations.}
    \label{fig:suff-cond-dist}
\end{figure}

While Theorem~\ref{thrm:SR-GUARD-GA} provides a conservative sufficient condition, Figure~\ref{fig:suff-cond-dist} shows that condition~\eqref{eq:apdx-suff-cond} is indeed satisfied on a subset of subsamples ($28\%$), with violations being small ($|\mathrm{LHS} - \mathrm{RHS}| < 0.025$). Moreover, even when the sufficient condition is violated, \textsc{GUARD} still demonstrates improved empirical performance, improving over GA in every Sacrifice Rate at the $10\%$ split, as shown in Table~\ref{tab:llama-forget10}.

%% file: contents/a.2.algorithm.tex
\section{Algorithms}

\subsection{Update Rules of Baseline Methods}
\label{subsec:apdx-unlearn-update}
We list the baseline unlearning algorithms included in the TOFU benchmark~\citep{maini2024tofu}: Gradient Ascent (GA), Gradient Difference (GD), KL Minimization (KM), and Preference Optimization (PO). We also list Negative Preference Optimization (NPO), the loss-reweighting method that we combine with \textsc{GUARD}.

\textbf{Gradient Ascent (GA)}. The GA method aims to reduce the confidence of the model in its initial predictions on the Forget Set. Specifically, for each instance in \( D_f \), the goal is to maximize the training loss, encouraging the model to diverge from its prior learned outputs, and the overall loss objective is defined as:
\begin{equation}
    L_{GA} = - L_{D_f}(\theta)
\end{equation}

\textbf{Gradient Difference (GD)}~\citep{liu2022continual}. The Gradient Difference method builds upon the gradient ascent approach by not only increasing loss on the Forget Set \( D_f \) but also preserving model performance on the Retain set \( D_r \). The objective function for this approach is formulated as:
\begin{equation}
\label{eq:apdx-GD-update}
    L_{\text{GD}} = -L_{D_f}(\theta) + L_{D_r}(\theta).
\end{equation}
To balance computational cost, for each sample processed from \( D_f \), a corresponding example from \( D_r \) is selected, ensuring the model remains within feasible resource constraints.

\textbf{KL Divergence Minimization (KM)}. The KL Minimization technique aims to regulate the model's unlearning process by ensuring that predictions on the Retain set \( D_r \) remain consistent with those of the original model while still reducing accuracy on \( D_f \). Let \( M \) be the model function, where \( M(\cdot) \) produces a probability distribution over possible token predictions. The training objective is expressed as:
\begin{equation}
\label{eq:apdx-KM-update}
    L_{\text{KM}} = -L_{D_f}(\theta) + \frac{1}{|D_r|} \sum_{s \in D_r} \frac{1}{|s|} \sum_{i=2}^{|s|} \text{KL} \big( M_{\text{0}}(s_{<i}) \parallel M_{\text{u}}(s_{<i}) \big).
\end{equation}
Here, \( M_{\text{0}} \) and \( M_{\text{u}} \) refer to the model before and after unlearning, respectively. Due to resource constraints, instances from \( D_r \) are sampled randomly, while the entire Forget Set is used.

\textbf{Preference Optimization (PO)}. Inspired by Direct Preference Optimization (DPO) \citep{rafailov2024dpo}, this method substitutes responses in \( D_f \) with alternative neutral responses such as ``I do not have that information''. The loss function is as follows:
\begin{equation}
\label{eq:apdx-PO-update}
    L_{\text{PO}} = L_{D_r}(\theta) + L_{D_f^{\text{idk}}}(\theta).
\end{equation}
This ensures that while the model updates its responses for \( D_f \), its overall linguistic abilities and predictive accuracy on \( D_r \) remain stable.

\textbf{Negative Preference Optimization (NPO)}~\citep{zhang2024negative}. NPO treats the forget samples as negative responses in a preference objective, which bounds the per-sample unlearning signal. We use NPO with a retain loss, i.e., the retain-regularized variant of~\citep{zhang2024negative}, which corresponds to \texttt{NPO\_GDR} in the MUSE codebase:
\begin{equation}
\label{eq:apdx-NPO-base}
    L_{\text{NPO}} = \frac{1}{n_f} \sum_{i \in I_f} \frac{2}{\beta} \log\left(1 + \left(\frac{\pi_\theta(y_i|x_i)}{\pi_{\text{ref}}(y_i|x_i)}\right)^\beta\right) + L_{D_r}(\theta),
\end{equation}
where $\pi_\theta$ and $\pi_{\text{ref}}$ denote the current model and the reference model before unlearning, respectively, and $\beta > 0$ is an inverse temperature.

\subsection{GUARD Objective for Frameworks Other than GA}
\label{subsec:apdx-guard-obj}
Given the Forget Set $D_f$ and a model parameterized by $\theta$, with slight abuse of notation, the GUARD losses for additional frameworks are defined as follows.
\begin{itemize}
    \item GUARD objective under Gradient Difference:
    \begin{align*}
        L^{GUARD}_{GD}(\theta) :=  -L^{GUARD}_{D_f}(\theta) + L_{D_r}(\theta)
    \end{align*}
    \item GUARD objective under KL Minimization:
    \begin{equation}
        L_{\text{KM}}^{GUARD} = -L^{GUARD}_{D_f}(\theta) + \frac{1}{|D_r|} \sum_{s \in D_r} \frac{1}{|s|} \sum_{i=2}^{|s|} \text{KL} \big( M_{\text{0}}(s_{<i}) \parallel M_{\text{u}}(s_{<i}) \big).
    \end{equation}
    \item GUARD objective under Preference Optimization (PO):
    \begin{equation}
    \label{eq:apdx-PO-guard-update}
        L_{\text{PO}}^{GUARD} = L_{D_r}(\theta) + L^{GUARD}_{D_f^{\text{idk}}}(\theta).
    \end{equation}
    \item GUARD objective under Negative Preference Optimization (NPO):
    \begin{equation}
    \label{eq:apdx-NPO-update}
        \begin{aligned}
        L_{\text{NPO}}^{GUARD} &= L^{GUARD}_{D_f,\text{NPO}}(\theta) + L_{D_r}(\theta), \\
        \text{where}\quad L^{GUARD}_{D_f,\text{NPO}}(\theta) &:= \frac{1}{n_f} \sum_{i \in I_f} \omega^{GUARD}_i \cdot \frac{2}{\beta} \log\left(1 + \left(\frac{\pi_\theta(y_i|x_i)}{\pi_{\text{ref}}(y_i|x_i)}\right)^\beta\right).
        \end{aligned}
    \end{equation}
\end{itemize}

In all of these objectives, \textsc{GUARD} only reweights the forget-side term, and the retain-side term is unchanged. The first-order analysis of Section~\ref{sec:theorem}, which compares a single GA step with a single \textsc{GUARD} step, therefore carries over to the forget-side terms of GD and KM, since their retain-side terms contribute identically to both updates. It also covers the first step of NPO, where all likelihood ratios $\pi_\theta(y_i|x_i)/\pi_{\text{ref}}(y_i|x_i)$ equal one and the gradient of the NPO forget-side term coincides with that of GA. For PO, whose forget-side term is minimized on the neutral responses rather than maximized on the original ones, and for NPO beyond its first step, the reweighting is a heuristic extension that is not covered by the analysis of Section~\ref{sec:theorem}.

%% file: contents/a.5.exp.tex
\section{Additional results}
\subsection{Additional Settings}
\label{subsec:config}
\paragraph{Configuration and Hyperparameters.} All experiments were evaluated on equal-cost hardware, Ubuntu 22.04 LTS system with 2 Intel(R) Xeon(R) Gold 6138 CPUs with 20 cores, 4 NVIDIA H100 GPUs with 80 GB memory, and 384 GB RAM. All experiments use 40 CPU cores and 4 GPUs.
Unless otherwise specified, the following configurations are used:

\texttt{TOFU benchmark:} We \emph{strictly follow} the official protocol of TOFU~\citep{maini2024tofu} for both fine-tuning and unlearning to ensure that all reported numbers are directly comparable to the official TOFU leaderboard\footnote{\url{https://huggingface.co/spaces/locuslab/tofu_leaderboard}}. All baselines, including the loss-reweighting methods and RMU, are run under this single protocol. Since their original papers use different forget splits, training lengths, checkpoint selection, batch sizes, or metrics, our numbers for these methods are not directly comparable with the published ones. Specifically:
\begin{itemize}
    \item \textbf{Optimizer.} AdamW with learning rate $10^{-5}$ and weight decay $0.01$.
    \item \textbf{Schedule.} $5$ epochs for both fine-tuning and unlearning, with a linear warm-up during the first epoch.
    \item \textbf{Batch size.} An effective batch size of $32$ question--answer pairs on $4$ GPUs, matching the effective batch size of TOFU's released configuration.
    \item \textbf{Reporting.} All TOFU results in this paper correspond to the checkpoint at the end of the $5$-th unlearning epoch, matching the convention used by the TOFU leaderboard.
    \item \textbf{Data attribution.} We use a temperature $\tau = 0.03$ unless stated otherwise. Attribution scores are computed once on the fine-tuned model checkpoint and reused across epochs.
\end{itemize}

We deliberately avoid altering any of TOFU's training hyperparameters: tuning epochs, batch size, or learning rate against the unlearning objective could implicitly bias the comparison in favour of GUARD. By inheriting TOFU's setup as-is, the only difference between a baseline and its GUARD-augmented counterpart is the per-sample reweighting introduced by GUARD's attribution scores.

\texttt{MUSE benchmark:} We strictly follow the official protocol of MUSE~\citep{shi2024musemachineunlearningsixway} so that our numbers for the methods evaluated in the MUSE paper (GA, GA\_GDR, GA\_KLR, and, up to the choice of $\beta$ discussed below, NPO\_GDR) are directly comparable to the MUSE leaderboard \citep[Table~3]{shi2024musemachineunlearningsixway}. Specifically:
\begin{itemize}
    \item \textbf{Target models.} We use the official MUSE-released target checkpoints: \texttt{LLaMA-2-7B} fine-tuned on BBC NEWS for the NEWS corpus, and \texttt{ICLM-7B} fine-tuned on Harry Potter for the BOOKS corpus. These are downloaded directly from the MUSE Huggingface repository.
    \item \textbf{Optimizer.} AdamW with learning rate $10^{-5}$ and per-device batch size of $32$.
    \item \textbf{Schedule.} Up to $10$ unlearning epochs with the official MUSE \emph{early-stopping} rule: training is stopped when the unlearned model's utility (Retain Set KnowMem) drops below the released MUSE \texttt{Retrain} reference. We report the checkpoint selected by this rule, matching the MUSE leaderboard convention.
    \item \textbf{Method-name mapping.} For consistency with TOFU notation, we use \emph{KM} to refer to MUSE's \texttt{GA\_KLR} (gradient ascent + KL regularizer on Retain Set), \emph{GD} to refer to MUSE's \texttt{GA\_GDR} (gradient ascent + gradient descent on Retain Set), and \emph{NPO} to refer to MUSE's \texttt{NPO\_GDR} (NPO with retain regularizer used in MUSE Table~3). Our GA matches MUSE's pure GA.
    \item \textbf{Data attribution.} We use a temperature $\tau = 0.03$ unless stated otherwise; attribution is computed once on the released target checkpoint and reused across epochs.
\end{itemize}

As with TOFU, we make no other modifications to the training pipeline beyond GUARD's per-sample reweighting; baseline hyperparameters are kept at their published defaults except where noted below.

\paragraph{Hyperparameter Selection Strategy for GUARD Integration.} To evaluate GUARD as a modular and robust enhancement of unlearning schemes, we adopt a conservative experimental design regarding hyperparameter selection. When applying GUARD to a baseline method (e.g., NPO, GA), we intentionally maintained the baseline's hyperparameters (such as learning rate, weight decay, and NPO's $\beta$) identical to their standalone settings (for NPO on MUSE, the value $\beta = 0.08$ discussed below, which is shared by NPO and NPO+\textsc{GUARD}). We only introduced GUARD's temperature $\tau$ as an additional hyperparameter. This design choice serves two purposes:
\begin{enumerate}
    \item \textbf{Modularity Check:} It ensures that GUARD functions as a ``plug-and-play" module that improves performance without requiring extensive method-specific retuning or changing the optimal parameter settings of the underlying algorithm.
    \item \textbf{Performance Lower Bound:} We note  that the hyperparameters optimized for the standalone baseline may not remain the global optimum when combined with GUARD's attribution-driven weighting. Therefore, our reported results likely represent a \textit{lower bound} on the potential performance of GUARD. While jointly optimizing both the baseline's parameters (e.g., NPO's $\beta$) and GUARD's $\tau$ could yield further improvements, we prioritized demonstrating that GUARD provides substantial gains even under these restrictive, non-co-optimized conditions.
\end{enumerate}

\paragraph{Baseline-Specific Hyperparameters.} For all baselines, we keep the algorithm-internal hyperparameters at the values reported in the original papers (or in the official TOFU / MUSE codebases) to ensure that comparisons are not biased by per-method retuning. The two explicit exceptions are \textbf{NPO on MUSE}, discussed below, and \textbf{RMU}, for which we use the \texttt{open-unlearning} configuration. A summary appears in Table~\ref{tab:baseline_hyperparams}.

\begin{itemize}
    \item \textbf{GA, PO}: No algorithm-internal hyperparameter beyond the learning rate. The loss is $-L_{D_f}$ (GA) or the equally-weighted sum $L_{D_r} + L_{D_f^{\text{idk}}}$ (PO).
    \item \textbf{GD, KM}: A single retain-regularizer weight, implicitly set to $\lambda = 1$ in TOFU's \texttt{dataloader.py} and MUSE's \texttt{baselines/iterative.py}, i.e., $L = -L_{D_f} + L_{D_r}$ for GD and $L = -L_{D_f} + \mathrm{KL}(M_0\,\|\,M_u)$ for KM.
    \item \textbf{NPO}~\citep{zhang2024negative}: The inverse-temperature $\beta$ controls the strength of unlearning: smaller $\beta$ makes the NPO loss collapse toward GA (more aggressive forgetting), while larger $\beta$ produces a slower, more retention-preserving divergence. \emph{On TOFU, we use the original paper's default $\beta = 0.1$.} \emph{On MUSE, we deliberately set $\beta = 0.08$ (slightly more aggressive than the default).} The reason for this deviation is privacy-driven: with $\beta = 0.1$, NPO leaves substantial residual Forget Set knowledge in the model, and once \textsc{GUARD}'s reweighting further protects Retain Set utility, the under-removed forget side translates into a poor PrivLeak score (membership inference becomes easy because the model still assigns high likelihood to forget samples). Lowering $\beta$ from $0.1$ to $0.08$ produces a more thorough forget side decay and brings PrivLeak meaningfully closer to the ideal $[-5\%, 5\%]$ window used in the MUSE protocol. This trades a modest amount of Retain KnowMem (utility) for a substantially better privacy outcome, which we judged to be the more meaningful comparison point under the MUSE evaluation.
    \item \textbf{SimNPO}~\citep{fan2025simplicityprevailsrethinkingnegative}: $\beta = 4.5$, $\gamma = 0$, \texttt{npo\_coeff} $= 0.125$, \texttt{grad\_diff\_coeff} $= 1.0$ on TOFU; $\beta = 0.7$ (NEWS) / $0.75$ (BOOKS), $\gamma = 0$, \texttt{npo\_coeff} $= 1.0$, retain weight $= 0.1$ on MUSE. Values follow the SimNPO official TOFU config and MUSE README.
    \item \textbf{FLAT (TV)}~\citep{wang2024llm}: f-divergence variant set to Total Variation ($f^*(x) = \tanh(x)/2$); \texttt{smooth\_rate} $= -0.4$; equal weighting $\lambda_e = \lambda_f = 1$ between the template-response and forget-response branches. Values follow the FLAT TOFU config; we apply the same configuration on MUSE.
    \item \textbf{SatImp}~\citep{yang2025exploringcriterialossreweighting}: Token-reweighting exponents $\beta_1 = 5$ (saturation) and $\beta_2 = 1$ (importance); retain regularizer weight $\lambda = 1$. Values follow Section~5 of the original SatImp paper and are reused on MUSE.
    \item \textbf{RMU}~\citep{li2024wmdp}: Steering coefficient $c = 2$, retain-loss weight $\alpha = 1$, target layer $\ell = 7$, forget-loss weight $\gamma = 1$. Values follow the widely-used \texttt{open-unlearning} RMU re-tune for TOFU/MUSE; the original WMDP paper reports much larger $c$ and $\alpha$ values that do not transfer well to TOFU/MUSE.
    \item \textbf{\textsc{GUARD} (ours)}: Temperature $\tau = 0.03$ on both TOFU and MUSE, applied identically across GA+\textsc{GUARD}, GD+\textsc{GUARD}, KM+\textsc{GUARD}, PO+\textsc{GUARD} (TOFU only, since PO is omitted on MUSE), and NPO+\textsc{GUARD}. We do not retune the underlying baseline's hyperparameters when integrating \textsc{GUARD}, so the only difference between any baseline and its \textsc{GUARD}-augmented counterpart is the per-sample reweighting derived from $a^{\text{GUARD}}_i$.
\end{itemize}

\begin{table}[h]
\centering
\caption{Algorithm-internal hyperparameters of all baselines and \textsc{GUARD} on TOFU and MUSE. All values follow the corresponding original papers / official codebases except \textbf{RMU} and \textbf{NPO on MUSE}, where we set $\beta = 0.08$ (vs.\ the default $0.1$) to obtain a more thorough forget-side decay. This brings PrivLeak meaningfully closer to the ideal $[-5\%,5\%]$ window at the cost of a modest Retain KnowMem decrease. ``---'' indicates no algorithm-internal hyperparameters present beyond the learning rate.}
\label{tab:baseline_hyperparams}
\vspace{0.1in}
\fontsize{8pt}{9pt}\selectfont
\begin{tabular}{llll}
\toprule
\textbf{Method} & \textbf{Hyperparameter} & \textbf{TOFU} & \textbf{MUSE} \\
\midrule
GA  & --- & --- & --- \\
PO  & --- & --- & --- \\
GD  & retain weight $\lambda$ & $1.0$ & $1.0$ \\
KM  & KL weight $\lambda$ & $1.0$ & $1.0$ \\
\midrule
NPO & inverse temperature $\beta$ & $0.1$ (default) & \textbf{$0.08$ (lowered for privacy)} \\
    & retain composite weight & $1.0$ & $1.0$ \\
SimNPO & $\beta$ & $4.5$ & $0.7$\textsuperscript{(N)} / $0.75$\textsuperscript{(B)} \\
       & $\gamma$ (margin) & $0$ & $0$ \\
       & \texttt{npo\_coeff} / retain weight & $0.125$ / $1.0$ & $1.0$ / $0.1$ \\
FLAT (TV) & \texttt{smooth\_rate} & $-0.4$ & $-0.4$ \\
          & template / forget weights & $1$ / $1$ & $1$ / $1$ \\
SatImp & $\beta_1$ / $\beta_2$ & $5$ / $1$ & $5$ / $1$ \\
       & retain weight $\lambda$ & $1.0$ & $1.0$ \\
RMU & $c$ / $\alpha$ / $\ell$ / $\gamma$ & $2$ / $1$ / $7$ / $1$ & $2$ / $1$ / $7$ / $1$ \\
\midrule
\textsc{GUARD} (ours) & temperature $\tau$ & $0.03$ & $0.03$ \\
\bottomrule
\multicolumn{4}{l}{\fontsize{6pt}{7pt}\selectfont \textsuperscript{(N)} MUSE NEWS; \textsuperscript{(B)} MUSE BOOKS.}\\
\end{tabular}
\end{table}

\subsection{Results for LLaMA-2-7B on the TOFU Data Set}
The results are presented in Tables~\ref{tab:llama-forget01} and~\ref{tab:llama-forget05}.
\label{subsec:apdx-exp-llama}
\input{Tables/llama_01}

\input{Tables/llama_05}

\subsection{Results for Phi-1.5B on the TOFU Data Set}
\label{subsec:apdx-exp-phi}
We present the results of Phi-1.5B in Tables~\ref{tab:phi-forget01}, \ref{tab:phi-forget05} and \ref{tab:phi-forget10}.

\input{Tables/phi_01}

\input{Tables/phi_05}

\input{Tables/phi_10}

\paragraph{Effectiveness of GUARD.} Across all data sets, methods, and splits, GUARD consistently reduces the sacrifice rate across key metrics ($\rho_{r}$, $\rho_{p}$, and $\rho_{t}$) wherever the baseline degrades performance (positive values), demonstrating its ability to mitigate unintended knowledge degradation while effectively removing targeted data. The numbers below are reported at the final ($5$-th) unlearning epoch under TOFU's official protocol; specifically:

\begin{itemize}
    \item \textit{Forgetting 1\% of data:} On the Retain Set, GUARD reduces $\rho_{r}$ from $4.83$ to $-1.58$ (compared to GA), $\rho_{p}$ by $13.38\%$ (compared to PO), and $\rho_{t}$ by $11.99\%$ (compared to PO).
    \item \textit{Forgetting 5\% of data:} On the Retain Set, GUARD lowers $\rho_{r}$ by $19.54\%$ (compared to KM), $\rho_{p}$ by $21.65\%$ (compared to KM), and $\rho_{t}$ by $18.68\%$ (compared to GA).
    \item \textit{Forgetting 10\% of data:} On the Retain Set, \textsc{GUARD} achieves reductions of $\rho_{r}$ by $19.38\%$ (compared to PO), $\rho_{p}$ by $30.48\%$ (compared to PO), and $\rho_{t}$ by $20.64\%$ (compared to GA).
\end{itemize}   

\paragraph{Data Set-Specific Observations.} Across all data sets, splits, and evaluation metrics, GUARD reduces every Retain-Set $\rho_{r}$ and $\rho_{p}$, which aligns with its design: the proxy data attribution is computed based on the alignment between the forget data and the overall retain data, directly benefiting retention performance. In contrast, improvements on the generalization sets (Real World Set and Real Author Set) are indirect. Since unlearning the Forget Set inherently reduces overfitting, GUARD in some cases turns a positive sacrifice rate of the baseline on the Real World Set or the Real Author Set into a negative one (e.g., $\rho_{p}$ of GA on the Real World Set at $5\%$, from $2.46$ to $-0.97$), which suggests that GUARD not only prevents unintended knowledge degradation but may also enhance model generalization.  

\subsection{Results for ICLM-7B on the MUSE BOOKS Data Set}
\label{subsec:apdx-exp-muse-books}
\input{Tables/muse_books}
We present the results of ICLM-7B on MUSE BOOKS data set in Table \ref{tab:unlearning_muse_book}. As on MUSE NEWS, \textsc{GUARD} improves Retain KnowMem for every baseline at comparable Forget Set VerbMem and KnowMem, and NPO+\textsc{GUARD} attains the highest Retain KnowMem ($52.4$).

\subsection{Computational Efficiency Analysis}
\label{subsec:time}
\input{Tables/time}

To evaluate the computational overhead introduced by GUARD, we conducted runtime comparisons between standard Gradient Ascent (GA) and GA augmented with GUARD across different data sets and model architectures. Table~\ref{tab:runtime_estimates} presents the per-epoch runtime measurements under TOFU's official protocol (effective batch size $32$) and the MUSE protocol on our experimental setup with 4 NVIDIA H100 GPUs. Each epoch time includes the end-of-epoch evaluation. The GA + GUARD column reports its first epoch, which additionally computes the attribution scores, and later epochs reuse the cached scores and cost the same as the baseline.

We point out that the runtime reported for GA + GUARD represents the \textbf{total integrated time}, encompassing both the unlearning optimization and the necessary gradient computations for data attribution. As shown in Table~\ref{tab:runtime_estimates}, GUARD introduces minimal overhead (approximately $11.5\%$ of one baseline epoch on average, with the per-setting range $\sim 7$--$21\%$, incurred only in the first epoch). For instance, on the TOFU $10\%$ split with LLaMA-2-7B, GUARD adds only $3.2$ minutes to the $32.3$-minute baseline (in the first epoch only). This efficiency stems from two key factors:
\begin{enumerate}
    \item \textbf{Efficient Gradient Computation:} The per-sample gradient computation (as defined in Eq.~\eqref{eq:def-J-prod}) requires only standard forward-backward passes. Unlike conventional influence functions, GUARD avoids computationally prohibitive Hessian estimation or model retraining.
    \item \textbf{Amortization via Preprocessing:} The attribution score calculation serves as a one-time preprocessing step. Once computed for a given model-data set combination, these scores are cached and reused across multiple unlearning epochs or hyperparameter tuning experiments, significantly amortizing the initial cost.
\end{enumerate}

In conclusion, the modest runtime increases observed across diverse experimental settings confirm that GUARD is a computationally viable enhancement. By avoiding the bottleneck of repetitive gradient recalculations and Hessian approximations, GUARD achieves substantial retention improvements without imposing prohibitive computational costs.

\subsection{Robustness with Respect to Limited and Synthetic Retain Sets on MUSE}
\label{subsec:Robustness}

We apply the same robustness protocol described in the main paper (Section~\ref{sec:exp}, Table~\ref{tab:robustness_tofu}) to MUSE NEWS (LLaMA-2-7B) and MUSE BOOKS (ICLM-7B). All other hyperparameters (NPO's $\beta$, \textsc{GUARD}'s temperature $\tau$) are kept at the values used in the main MUSE experiments (Tables~\ref{tab:unlearning_muse_new} and~\ref{tab:unlearning_muse_book}); only the gradient $J^{D_r}_{\text{avg}}$ used to compute the attribution score is replaced by its $20\%$-subsample or GPT-4 synthetic counterpart.

Table~\ref{tab:robustness_muse} reports the four MUSE metrics across NPO, NPO+\textsc{GUARD} (full / $20\%$ / synthetic retain), and RMU at full-retain access (included as a reference for the strongest competing baseline). Mirroring the TOFU pattern, both \textsc{GUARD} ablations preserve most of full-retain \textsc{GUARD}'s Retain KnowMem gain over NPO on both NEWS and BOOKS: on NEWS the Retain KnowMem moves from $37.2$ (NPO) to $47.6$ ($20\%$) and $46.5$ (synthetic), close to the $50.0$ obtained with full-retain; on BOOKS it moves from $24.3$ (NPO) to $49.0$ ($20\%$) and $47.0$ (synthetic), close to the full-retain $52.4$. Notably, on NEWS the synthetic-retain variant ($46.5$) already exceeds RMU at full-retain access ($43.1$); on BOOKS it is comparable to RMU at full-retain ($47.0$ vs.\ $47.8$), despite using \emph{no} original retain data in the attribution. Forget-side metrics (VerbMem, KnowMem) remain essentially unchanged across all settings, confirming that surrogate-based attribution does not compromise forgetting effectiveness. Interestingly, PrivLeak moves slightly closer to $0$ as the retain surrogate becomes more approximate ($98.2 \to 97.5 \to 96.2$ on NEWS; $-38.4 \to -37.1 \to -36.2$ on BOOKS), suggesting that less precise retain matching also incidentally relaxes membership-inference exposure---a small privacy benefit that partially offsets the slight Retain KnowMem drop relative to full-retain.

These MUSE findings, taken together with the TOFU results in the main paper, validate \textsc{GUARD}'s practical robustness when the original Retain Set is partially or entirely inaccessible. For surrogate-construction strategies beyond random subsampling and LLM generation, prior work on Fisher-information sampling~\citep{liu2024fisher} and coreset methods~\citep{mirzasoleiman2020coresets} can be plugged into \textsc{GUARD} without modification.

\subsection{The Synthetic Retain Set: Pipeline and Quality Validation}
\label{subsec:synthetic_retain}

The synthetic-retain ablation reported in the main paper (Table~\ref{tab:robustness_tofu}) and in Appendix~\ref{subsec:Robustness} (Table~\ref{tab:robustness_muse}) addresses the strictly harder case in which \emph{no} portion of the original Retain Set is accessible to \textsc{GUARD}'s attribution---e.g., when the data has been deleted for compliance reasons, or when an organization deploys an unlearning pipeline on a third-party model whose training corpus is unknown. In such settings, \textsc{GUARD}'s attribution score $a^{\text{GUARD}}_i = [J^{D_r}_{\text{avg}}]^\top J_i(\theta_0)$ is no longer directly computable, since $J^{D_r}_{\text{avg}}$ requires gradient evaluation on the retain corpus. Our workaround is to replace $D_r$ with a \emph{synthetic retain proxy} $\tilde{D}_r$ produced by an external LLM (here, GPT-4) prompted to generate samples that share only the general format and topic of the natural retain corpus of each benchmark, using only public benchmark information that typically survives compliance-driven data deletion; the attribution score becomes $a^{\text{GUARD-syn}}_i = [J^{\tilde{D}_r}_{\text{avg}}]^\top J_i(\theta_0)$ and plugs into the standard \textsc{GUARD} pipeline without modification. Note that the synthetic proxy replaces $D_r$ only in \textsc{GUARD}'s attribution. Base methods whose own objective contains a retain term (GD, KM, PO, and NPO) still use $D_r$ for training, so the whole pipeline is free of the original retain data only for base methods without such a term, such as GA.

\paragraph{Synthesis Pipelines.} For each benchmark we use a single, lightweight prompt:
\begin{itemize}
    \item \emph{TOFU.} We generate $200$ \emph{new} fictitious authors (each with a JSON profile of biographical metadata) and $20$ question--answer pairs per author, yielding $\sim 4{,}000$ synthetic QA pairs that match the scale of the full TOFU dataset. Author names are explicitly disallowed to overlap with any Forget Set author.
    \item \emph{MUSE NEWS.} We generate $\sim 1.6$M tokens of BBC-style news articles spanning sports, weather, business, and arts---topics chosen to be disjoint from the forget split's political-news topics.
    \item \emph{MUSE BOOKS.} We generate $\sim 0.5$M tokens of fantasy fiction excerpts featuring magical schools / wizards, with the prompt explicitly forbidding any Harry Potter character or place names.
\end{itemize}
The total GPT-4 API cost across all three settings is approximately \$$40$.

\paragraph{Quality Validation.} To verify that $\tilde{D}_r$ does not leak Forget Set content, we compute (i) the maximum BERT-embedding cosine similarity between any synthetic sample and any Forget Set sample (max sim. $\le 0.68$ across all three benchmarks, well below the $\ge 0.92$ typical of paraphrase pairs); and (ii) GPT-4-as-judge labels less than $3\%$ of synthetic samples as containing Forget Set entities. The synthetic data therefore shares the format and topic of the natural retain corpus while remaining nearly content-disjoint from the Forget Set.

\input{Tables/synthetic_retain}

\subsection{Proxy for Data Attribution}
\label{subsec:apdx-exp-da}

For the unification methods, the reversely unified scores controlled by temperature are plotted in Fig. \ref{fig:score}. Exponential unification provides a smoother attribution distribution at a higher temperature.

\begin{figure}[ht]
    \centering
    \includegraphics[width=0.45\linewidth]{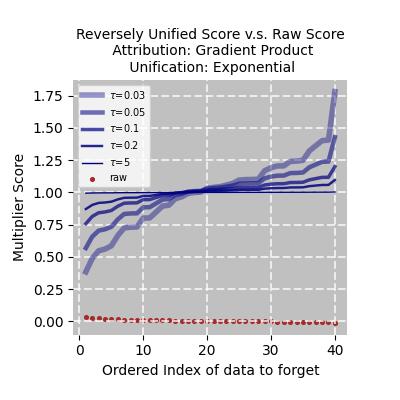}
    \caption{Raw Score and Inverse unified score under different temperatures $\tau$ for the $40$ samples of the TOFU forget01 split (Phi-1.5B), ordered by score (gradient-product attribution, exponential unification).}
    \label{fig:score}
\end{figure}

\subsection{Detailed Results for Alternative Attribution Methods}
\label{app:alternative_attribution}

In Section~\ref{sec:exp}, we compared \textsc{GUARD} with three alternative attribution mechanisms: Embedding Similarity (ES), RelatIF, and RIF. Here we provide the detailed quantitative results. Table~\ref{tab:attribution_sacrifice} presents the Sacrifice Rates, and Table~\ref{tab:attribution_absolute} presents the Absolute Performance metrics. Consistent with our main findings, \textsc{GUARD} achieves superior retention performance (lower sacrifice rates $\rho_r$ and $\rho_p$ and higher Retain ROUGE-L and Probability) compared to these first-order approximations and feature-space methods. Here, the Embedding Similarity (ES) score of a forget sample is the cosine similarity between its BERT embedding and the average BERT embedding of the Retain Set. The first-order RelatIF score is $a_i^{\text{GUARD}}/\|J_i(\theta_0)\|_2$, i.e., the RelatIF normalization with $H \approx I$, and the first-order RIF score is $(1-h_i)^{-1} a_i^{\text{GUARD}}$, where the generalized leverage score $h_i$ is also evaluated with $H \approx I$.

\input{Tables/comparison_of_different_influence_function}
\input{Tables/asb_comparison_of_different_influence_function}

%% file: Tables/llama_01.tex
\begin{table}[ht]
\centering
\caption{Unlearning experiments with LLaMA-2-7B on $1\%$ of all data points in TOFU under the official TOFU protocol; all values are reported at the final ($5$-th) unlearning epoch. \textsc{GUARD} consistently reduces the Sacrifice Rates across all baselines and evaluation metrics wherever the baseline degrades performance (positive values).
}
\label{tab:llama-forget01}
\vspace{0.1in}
\small
\begin{tabular}{lccccccccc}
\toprule
\multirow{2}{*}{\textbf{Methods}}
& \multicolumn{3}{c}{\textbf{Retain Set($\downarrow$)}}
& \multicolumn{3}{c}{\textbf{Real Author Set($\downarrow$)}}
& \multicolumn{3}{c}{\textbf{Real World Set($\downarrow$)}} \\
\cmidrule(lr){2-4} \cmidrule(lr){5-7} \cmidrule(lr){8-10}
& $\rho_{r}$ & $\rho_{p}$ & $\rho_{t}$
& $\rho_{r}$ & $\rho_{p}$ & $\rho_{t}$
& $\rho_{r}$ & $\rho_{p}$ & $\rho_{t}$ \\
\midrule
GA
& 17.53 & 3.94 & \textbf{-2.81}
& 5.56 & 4.22 & 52.92
& 0.44 & 1.86 & 1.94 \\
GA + GUARD
&  \textbf{12.30} &  \textbf{2.75} &   -2.00
&  \textbf{4.00} &  \textbf{3.10} &  \textbf{38.00}
&  \textbf{0.30} &  \textbf{1.30} &  \textbf{1.50} \\
\midrule
KM
& 18.28 & 3.90 & \textbf{-3.04}
& 5.72 & 4.37 & 56.30
& \textbf{-1.33} & 1.90 & 3.26 \\
KM + GUARD
&  \textbf{12.50} &  \textbf{2.80} &   -2.20
&  \textbf{4.20} &  \textbf{3.20} &  \textbf{40.00}
&   -0.90 &  \textbf{1.40} &  \textbf{2.50} \\
\midrule
PO
& 11.20 & 16.73 & 29.90
& \textbf{-1.48} & \textbf{-8.36} & \textbf{-12.06}
& 0.39 & \textbf{-21.10} & \textbf{-7.41} \\
PO + GUARD
&  \textbf{8.50} &  \textbf{12.50} &  \textbf{22.00}
&   -1.20 &   -6.50 &   -9.00
&  \textbf{0.30} &   -16.00 &   -5.50 \\
\midrule
GD
& 16.67 & 3.84 &  \textbf{-5.99}
& 7.26 & 4.99 & 73.96
& 0.46 & 2.16 & \textbf{-6.25} \\
GD + GUARD
&  \textbf{12.00} &  \textbf{2.80} & -4.20
&  \textbf{5.50} &  \textbf{3.80} &  \textbf{55.00}
&  \textbf{0.30} &  \textbf{1.50} &   -4.50 \\
\bottomrule
\end{tabular}
\end{table}

%% file: Tables/llama_05.tex
\begin{table}[ht]
\centering
\caption{Unlearning experiments with LLaMA-2-7B on $5\%$ of all data points in TOFU under the official TOFU protocol; all values are reported at the final ($5$-th) unlearning epoch. \textsc{GUARD} consistently reduces the Sacrifice Rates across all evaluation metrics and baselines wherever the baseline degrades performance (positive values), with the most substantial gains observed on the Retain Set for GD ($\rho_r$ and $\rho_p$ reduced by $39.32\%$ and $47.14\%$) and PO ($\rho_t$ reduced by $39.53\%$).}

\label{tab:llama-forget05}
\vspace{0.1in}
\small
\begin{tabular}{lccccccccc}
\toprule
\multirow{2}{*}{\textbf{Methods}}
& \multicolumn{3}{c}{\textbf{Retain Set($\downarrow$)}}
& \multicolumn{3}{c}{\textbf{Real Author Set($\downarrow$)}}
& \multicolumn{3}{c}{\textbf{Real World Set($\downarrow$)}} \\
\cmidrule(lr){2-4} \cmidrule(lr){5-7} \cmidrule(lr){8-10}
& $\rho_{r}$ & $\rho_{p}$ & $\rho_{t}$
& $\rho_{r}$ & $\rho_{p}$ & $\rho_{t}$
& $\rho_{r}$ & $\rho_{p}$ & $\rho_{t}$ \\
\midrule
GA
& 100.14 & 99.63 & 204.51
& 94.56 & 17.50 & 103.89
& 89.64 & 16.38 & 113.42 \\
GA + GUARD
&  \textbf{83.31} &  \textbf{86.52} &  \textbf{158.75}
&  \textbf{77.47} &  \textbf{13.61} &  \textbf{84.57}
&  \textbf{72.03} &  \textbf{12.14} &  \textbf{96.46} \\
\midrule
KM
& 99.68 & 100.02 & 454.89
& 95.09 & 18.71 & 281.02
& 89.10 & 12.69 & 212.39 \\
KM + GUARD
&  \textbf{82.16} &  \textbf{86.39} &  \textbf{418.08}
&  \textbf{76.11} &  \textbf{14.93} &  \textbf{261.36}
&  \textbf{71.02} &  \textbf{8.38} &  \textbf{195.91} \\
\midrule
PO
& 33.52 & 41.04 & 66.15
& 78.82 & 26.54 & 121.14
& 14.30 & 9.16 & 78.52 \\
PO + GUARD
&  \textbf{20.48} &  \textbf{28.83} &  \textbf{40.00}
&  \textbf{68.21} &  \textbf{18.62} &  \textbf{101.05}
&  \textbf{12.83} &  \textbf{5.14} &  \textbf{64.33} \\
\midrule
GD
& 90.51 & 83.92 & \textbf{-58.70}
& 54.91 & 4.92 & 21.01
& 12.91 & \textbf{-0.59} & \textbf{-9.38} \\
GD + GUARD
&  \textbf{54.92} &  \textbf{44.36} &   -18.88
&  \textbf{32.40} &  \textbf{3.50} &  \textbf{12.50}
&  \textbf{8.30} &   -0.40 &   -6.10 \\
\bottomrule
\end{tabular}
\end{table}

%% file: Tables/phi_01.tex
\begin{table}[ht]
\centering
\caption{{Sacrifice rate ($\%$) of different unlearning methods when unlearning Phi-1.5B on $1\%$ of TOFU data, under the official TOFU protocol; all values are reported at the final ($5$-th) unlearning epoch. Across all datasets (Retain Set, Real Author Set, and Real World Set), GUARD consistently reduces the sacrifice rate across various methods (GA, GD, KM, PO) wherever the baseline degrades performance (positive values). On the Retain Set, GUARD lowers $\rho_{r}$ of GA from $4.83$ to $-1.58$, $\rho_{p}$ by $13.38\%$ over PO, and $\rho_{t}$ by $11.99\%$ over PO.}}
\vspace{0.1in}
\begin{tabular}{lccccccccc}
\specialrule{1pt}{0pt}{0pt}
& \multicolumn{3}{c}{\textbf{Retain Set}($\downarrow$)}  & \multicolumn{3}{c}{\textbf{Real Author Set}($\downarrow$)}  & \multicolumn{3}{c}{\textbf{Real World Set}($\downarrow$)}  \\ \cline{2-10}
\multirow{-2}{*}{\textbf{Methods}} & $\rho_{r}$ & $\rho_{p}$  & $\rho_{t}$ & $\rho_{r}$ & $\rho_{p}$  & $\rho_{t}$  & $\rho_{r}$ & $\rho_{p}$   & $\rho_{t}$   \\ \specialrule{1pt}{0pt}{0pt}
GA                                 & 4.83          & 2.27           &  \textbf{-0.35}            &  \textbf{-17.79}           & 3.01           & 117.87          & 7.01           & 2.97            & 105.41             \\
GA+GUARD                           &  \textbf{-1.58}  &  \textbf{0.85}  &   0.77   & -15.90  &  \textbf{1.61}  &  \textbf{104.06} &  \textbf{6.17}  &  \textbf{2.15}   &  \textbf{104.72}    \\ \hline
GD                                 & 9.21          & 0.96           & 210.08            &  \textbf{-6.35}           & 0.89           & \textbf{-63.85}          & 10.91           & 0.44            & 32.37             \\
GD+GUARD                           &  \textbf{4.77} &  \textbf{-0.20}  &  \textbf{208.53}   &   -4.23  &  \textbf{-0.43}  &   -46.32 &  \textbf{10.64}  &  \textbf{-0.31}   &  \textbf{30.40}    \\ \hline
KM                                 & 5.82          & 2.59           &  \textbf{-8.36}            &  \textbf{-9.09}           & 1.85           & 114.83          & 10.47           & 2.23            & 112.09             \\
KM+GUARD                           &  \textbf{0.51} &  \textbf{1.56}  &   -7.41   &   -7.18  &  \textbf{0.59}  &  \textbf{99.71} &  \textbf{10.01}  &  \textbf{1.85}   &  \textbf{111.17}    \\ \hline
PO                                 & 27.87          & 30.57          & 65.98           & 54.31           &  \textbf{-17.04}           &  \textbf{-27.08}          & 46.64           & \textbf{-22.50}           &  \textbf{-89.73}             \\
PO+GUARD                           &  \textbf{25.06}  &  \textbf{26.48} &  \textbf{58.07}  &  \textbf{53.98}  &   -14.67  &   -23.97  &  \textbf{46.49}  &   -17.14  &   -88.02    \\ \specialrule{1pt}{0pt}{0pt}
\end{tabular}

\label{tab:phi-forget01}
\end{table}

%% file: Tables/phi_05.tex
\begin{table}[ht]
\centering
\caption{{Sacrifice Rate of GUARD upon unlearning Phi-1.5B on $5\%$ of all datapoints in TOFU, under the official TOFU protocol; all values are reported at the final ($5$-th) unlearning epoch. When unlearning a larger portion (compared to the portion of $1\%$ in Table \ref{tab:phi-forget01}) of the dataset, GUARD generally induces a larger absolute reduction on the sacrifice rate across all evaluation sets (Retain Set, Real Author Set, and Real World Set) and methods (GA, GD, KM, and PO). On the Retain Set, GUARD lowers $\rho_{r}$ by $19.54\%$ over KM, $\rho_{p}$ by $21.65\%$ over KM, and $\rho_{t}$ by $18.68\%$ over GA. On the Real World Set, GA+\textit{GUARD} achieves a negative $\rho_{p}$ while GA does not, indicating that GUARD can in some cases improve generalization by reducing overfitting on the Forget Set.}}
\vspace{0.1in}
\begin{tabular}{lccccccccc}
\specialrule{1pt}{0pt}{0pt}
& \multicolumn{3}{c}{\textbf{Retain Set}($\downarrow$)}  & \multicolumn{3}{c}{\textbf{Real Author Set}($\downarrow$)}  & \multicolumn{3}{c}{\textbf{Real World Set}($\downarrow$)}  \\ \cline{2-10}
\multirow{-2}{*}{\textbf{Methods}} & $\rho_{r}$ & $\rho_{p}$  & $\rho_{t}$ & $\rho_{r}$ & $\rho_{p}$  & $\rho_{t}$  & $\rho_{r}$ & $\rho_{p}$   & $\rho_{t}$   \\ \specialrule{1pt}{0pt}{0pt}
GA                                 & 96.88          & 95.89          & 102.48           & 34.08          & 2.14           & 36.04          & 23.06          & 2.46           & 22.87            \\
GA+GUARD                           &  \textbf{78.49} &  \textbf{79.21} &  \textbf{83.34}  &  \textbf{25.77} &  \textbf{-0.20}  &  \textbf{31.01} &  \textbf{7.87} &  \textbf{-0.97}   &  \textbf{11.20}   \\ \hline
GD                                 & 81.30          & 46.58          & 40.28           & 35.73          & \textbf{-0.61}           & 10.57          & 23.39          & 1.50            & 28.14            \\
GD+GUARD                           &  \textbf{67.92} &  \textbf{34.43} &  \textbf{36.41}  &  \textbf{26.66} &   -0.53  &  \textbf{8.02}  &  \textbf{15.11} &  \textbf{1.33}   &  \textbf{24.53}    \\ \hline
KM                                 & 96.65          & 94.02          & 94.36           & 29.07          & 2.61          & 37.23          & 33.37          & 3.69           & 31.62            \\
KM+GUARD                           &  \textbf{77.76} &  \textbf{73.66} &  \textbf{76.23}  &  \textbf{20.46} &  \textbf{0.22}  &  \textbf{31.83} &  \textbf{18.61} &  \textbf{0.04}  &  \textbf{21.29}   \\ \hline
PO                                 & 47.22          & 53.08          & 106.86           & 14.21          & 9.25          & 103.74          & 4.35           & -0.09           &  \textbf{-21.03}             \\
PO+GUARD                           &  \textbf{39.11} &  \textbf{41.39} &  \textbf{88.70}  &  \textbf{8.23} &  \textbf{5.57}  &  \textbf{82.02} &  \textbf{3.11}  &  \textbf{-0.40}  &   -18.63    \\ \specialrule{1pt}{0pt}{0pt}
\end{tabular}

\label{tab:phi-forget05}
\end{table}

%% file: Tables/phi_10.tex
\begin{table}[ht]
\centering
\caption{{Sacrifice Rate of GUARD upon unlearning Phi-1.5B on $10\%$ of all datapoints in TOFU, under the official TOFU protocol; all values are reported at the final ($5$-th) unlearning epoch. When unlearning a larger portion of the dataset, GUARD generally induces an even larger absolute reduction on the sacrifice rate across all evaluation sets (Retain Set, Real Author Set, and Real World Set) and methods (GA, GD, KM, and PO). On the Retain Set, GUARD lowers $\rho_{r}$ by $19.38\%$ over PO, $\rho_{p}$ by $30.48\%$ over PO, and $\rho_{t}$ by $20.64\%$ over GA.}}
\vspace{0.1in}
\begin{tabular}{lccccccccc}
\specialrule{1pt}{0pt}{0pt}
& \multicolumn{3}{c}{\textbf{Retain Set}($\downarrow$)}  & \multicolumn{3}{c}{\textbf{Real Author Set}($\downarrow$)}  & \multicolumn{3}{c}{\textbf{Real World Set}($\downarrow$)}  \\ \cline{2-10}
\multirow{-2}{*}{\textbf{Methods}} & $\rho_{r}$ & $\rho_{p}$  & $\rho_{t}$ & $\rho_{r}$ & $\rho_{p}$  & $\rho_{t}$  & $\rho_{r}$ & $\rho_{p}$   & $\rho_{t}$   \\ \specialrule{1pt}{0pt}{0pt}
GA                                 & 100.25         & 99.77          & 124.13          & 47.76          & 10.96          & 28.72          & 86.85          & 17.56           & 61.66            \\
GA+GUARD                           &  \textbf{81.50} &  \textbf{83.43} &  \textbf{98.51}  &  \textbf{38.20} &  \textbf{8.50}  &  \textbf{22.10} &  \textbf{67.30} &  \textbf{12.40}  &  \textbf{42.80}   \\ \hline
GD                                 & 80.69          & 68.66          & 22.79           & 44.41          & \textbf{-0.62}           & \textbf{-12.76}          & 37.96          & 0.78            & \textbf{-16.15}            \\
GD+GUARD                           &  \textbf{63.20} &  \textbf{52.40} &  \textbf{16.80}  &  \textbf{32.50} &   -0.50  &   -9.40  &  \textbf{27.40} &  \textbf{0.55}   &   -11.50   \\ \hline
KM                                 & 100.59         & 99.77          & 121.35          & 47.44          & 11.67          & 27.42          & 86.27          & 17.81           & 59.41            \\
KM+GUARD                           &  \textbf{82.10} &  \textbf{83.50} &  \textbf{96.40}  &  \textbf{38.50} &  \textbf{9.30}  &  \textbf{21.20} &  \textbf{67.80} &  \textbf{12.80}  &  \textbf{41.30}   \\ \hline
PO                                 & 35.76          & 45.24          & 65.17           & 21.64          & 13.91          & 180.20          & 4.19           & \textbf{-1.04}            & 8.43             \\
PO+GUARD                           &  \textbf{28.83} &  \textbf{31.45} &  \textbf{43.80}  &  \textbf{14.50} &  \textbf{9.20}  &  \textbf{152.30} &  \textbf{2.80}  &   -0.70   &  \textbf{5.60}    \\ \specialrule{1pt}{0pt}{0pt}
\end{tabular}

\label{tab:phi-forget10}
\end{table}

%% file: Tables/muse_books.tex
\begin{table}[htbp]
\centering
\caption{Unlearning experiments with ICLM-7B on the MUSE BOOKS dataset under the official MUSE protocol; ``w/o Unlearn'' is the released MUSE BOOKS target model.}
\vspace{0.1in}
\label{tab:unlearning_muse_book}
\resizebox{0.9\textwidth}{!}{
\begin{tabular}{lcccc}
\toprule

\multirow{3}{*}{\textbf{Methods}} & \textbf{Forget Set VerbMem} & \textbf{Forget Set KnowMem} & \textbf{PrivLeak} & \textbf{Retain Set KnowMem} \\
& \textbf{($\downarrow$)} & \textbf{($\downarrow$)} & \textbf{($\in [-5\%, 5\%]$)} & \textbf{($\uparrow$)} \\
\cmidrule(lr){2-5}

\multicolumn{5}{c}{\textbf{BOOKS}} \\
\midrule
w/o Unlearn & 99.8 & 59.4 & -57.5 & 66.9 \\

\midrule
GA & \textbf{0.0} & \textbf{0.0} & \textbf{-25.4} & 0.0 \\
GA + GUARD & \textbf{\textbf{0.0}} & \textbf{\textbf{0.0}} & -31.0 & \textbf{10.6} \\
\midrule
KM & \textbf{15.7} & 22.3 & \textbf{-39.5} & 36.2 \\
KM + GUARD & 16.2 & \textbf{20.4} & -41.8 & \textbf{47.5} \\
\midrule
GD & \textbf{0.0} & \textbf{0.0} & -28.8 & 11.4 \\
GD + GUARD & \textbf{\textbf{0.0}} & \textbf{\textbf{0.0}} & \textbf{-26.9} & \textbf{27.9} \\

\midrule
NPO & \textbf{0.0} & \textbf{0.0} & -32.0 & 24.3 \\
SimNPO & \textbf{0.0} & \textbf{0.0} & -34.8 & 47.5 \\
FLAT (TV) & \textbf{0.0} & \textbf{0.0} & -45.4 & 28.7 \\
SatImp & \textbf{0.0} & \textbf{0.0} & \textbf{-26.6} & 22.7 \\
RMU & \textbf{0.0} & \textbf{0.0} & -36.1 & 47.8\\
NPO + GUARD & \textbf{0.0} & \textbf{0.0} & -38.4 & \textbf{52.4} \\
\bottomrule
\end{tabular}}
\end{table}

%% file: Tables/time.tex
\begin{table}[h]
\centering
\caption{Per-epoch runtime of unlearning with and without GUARD, measured under the official TOFU and MUSE protocols. Attribution is the difference between the two runtimes, and Overhead is Attribution relative to the GA runtime.}
\vspace{0.1in}
\label{tab:runtime_estimates}
\resizebox{0.9\textwidth}{!}{
\begin{tabular}{lllcccc}
\toprule
\textbf{Dataset} & \textbf{Model} & \textbf{Forget set} & \textbf{GA (min)} & \textbf{GA + GUARD (min)} & \textbf{Attribution (min)} & \textbf{Overhead} \\
\midrule
\multirow{6}{*}{TOFU} & \multirow{3}{*}{LLaMA-2-7B} & 1\% & 4.2 & 5.1 & 0.9 & 21.4\% \\
 &  & 5\% & 14.4 & 16.6 & 2.2 & 15.3\% \\
 &  & 10\% & 32.3 & 35.5 & 3.2 & 9.9\% \\
\cmidrule{2-7}
 & \multirow{3}{*}{Phi-1.5B} & 1\% & 2.8 & 3.0 & 0.2 & 7.1\% \\
 &  & 5\% & 6.1 & 6.8 & 0.7 & 11.5\% \\
 &  & 10\% & 14.5 & 15.6 & 1.1 & 7.6\% \\
\midrule
\multirow{2}{*}{MUSE} & LLaMA-2-7B & NEWS & 83.7 & 91.2 & 7.5 & 9.0\% \\
 & ICLM-7B & BOOKS & 116.0 & 127.9 & 11.9 & 10.3\% \\
\bottomrule
\end{tabular}}
\end{table}

%% file: Tables/synthetic_retain.tex
\begin{table}[h]
\centering
\caption{Robustness of \textsc{GUARD} under limited (random $20\%$) and synthetic (GPT-4 generated, with no original retain data used in the attribution) Retain Sets on MUSE NEWS (LLaMA-2-7B) and MUSE BOOKS (ICLM-7B), at the early-stopping checkpoint under MUSE's official protocol. RMU is included at full retain access as a reference point for the strongest competing baseline. As on TOFU (Table~\ref{tab:robustness_tofu}), both \textsc{GUARD} ablations preserve most of full-retain \textsc{GUARD}'s Retain KnowMem gain over NPO; on NEWS, the synthetic-retain variant already outperforms RMU at full retain access, and on BOOKS it remains comparable. Bold marks the column-wise best across all rows; PrivLeak is reported for completeness (values closer to $0$ are preferred).}
\label{tab:robustness_muse}
\vspace{0.1in}
\fontsize{7pt}{8pt}\selectfont
\resizebox{\textwidth}{!}{
\begin{tabular}{lcccccccc}
\toprule
\multirow{2}{*}{\textbf{Methods}}
& \multicolumn{4}{c}{\textbf{MUSE NEWS (LLaMA-2-7B)}}
& \multicolumn{4}{c}{\textbf{MUSE BOOKS (ICLM-7B)}} \\
\cmidrule(lr){2-5} \cmidrule(lr){6-9}
& Forget VerbMem & Forget KnowMem & PrivLeak & Retain KnowMem
& Forget VerbMem & Forget KnowMem & PrivLeak & Retain KnowMem \\
& ($\downarrow$) & ($\downarrow$) & ($\in [-5\%, 5\%]$) & ($\uparrow$)
& ($\downarrow$) & ($\downarrow$) & ($\in [-5\%, 5\%]$) & ($\uparrow$) \\
\midrule
NPO
& \textbf{1.5} & 44.8 & \textbf{93.0} & 37.2
& 0.0 & 0.0 & \textbf{-32.0} & 24.3 \\
RMU (full retain, reference)
& \textbf{1.5} & \textbf{44.0} & 95.0 & 43.1
& 0.0 & 0.0 & -36.1 & 47.8 \\
\midrule
NPO + \textsc{GUARD} (full retain)
& 1.6 & 45.2 & 98.2 & \textbf{50.0}
& 0.0 & 0.0 & -38.4 & \textbf{52.4} \\
NPO + \textsc{GUARD} ($20\%$ retain)
& 1.6 & 45.4 & 97.5 & 47.6
& 0.0 & 0.0 & -37.1 & 49.0 \\
NPO + \textsc{GUARD} (synthetic retain)
& 1.6 & 45.6 & 96.2 & 46.5
& 0.0 & 0.0 & -36.2 & 47.0 \\
\bottomrule
\end{tabular}}
\end{table}

%% file: Tables/comparison_of_different_influence_function.tex
\begin{table*}[ht]
\centering
\caption{Comparison of the performance of different attribution methods integrated with NPO on the TOFU benchmark (LLaMA-2-7B, $10\%$ forget split) under the official TOFU protocol; values reported at the final ($5$-th) unlearning epoch. \textsc{GUARD} outperforms both the feature-based baseline (ES) and other first-order influence approximations (RelatIF, RIF) across all evaluation sets in $\rho_r$ and $\rho_p$. The substantially lower sacrifice rates on the Retain Set demonstrate that the gradient-alignment strategy of \textsc{GUARD} minimizes unintended harm to preserved knowledge more effectively than alternative attribution metrics.}
\vspace{0.1in}
\label{tab:attribution_sacrifice}
\fontsize{7pt}{7.5pt}\selectfont
\resizebox{0.8\textwidth}{!}{
\begin{tabular}{lccccccccc}
\toprule
\multirow{2}{*}{\textbf{Methods}}
& \multicolumn{3}{c}{\textbf{Retain Set($\downarrow$)}}
& \multicolumn{3}{c}{\textbf{Real Author Set($\downarrow$)}}
& \multicolumn{3}{c}{\textbf{Real World Set($\downarrow$)}} \\
\cmidrule(lr){2-4} \cmidrule(lr){5-7} \cmidrule(lr){8-10}
& $\rho_{r}$ & $\rho_{p}$ & $\rho_{t}$
& $\rho_{r}$ & $\rho_{p}$ & $\rho_{t}$
& $\rho_{r}$ & $\rho_{p}$ & $\rho_{t}$ \\
\midrule

NPO + ES
& 26.53 & 9.04 & \textbf{-12.54}
& 19.60 & 10.20 & \textbf{-17.40}
& 4.80 & 4.10 & \textbf{-14.30} \\

\midrule
NPO + RelatIF
& 22.08 & 6.35 & -9.85
& 17.40 & 9.10 & -14.50
& 4.10 & 3.50 & -13.10 \\

\midrule
NPO + RIF
& 19.08 & 4.52 & -8.50
& 16.20 & 8.70 & -13.20
& 3.60 & 3.20 & -12.30 \\

\midrule
NPO + GUARD
& \textbf{14.32} & \textbf{2.87} & -6.10
& \textbf{13.59} & \textbf{7.14} & -10.44
& \textbf{2.70} & \textbf{2.23} & -10.29 \\
\bottomrule

\end{tabular}}
\end{table*}

%% file: Tables/asb_comparison_of_different_influence_function.tex
\begin{table}[t]
\centering
\caption{Comparison of absolute performance metrics between \textsc{GUARD} and alternative attribution methods integrated with NPO, under the official TOFU protocol; values reported at the final ($5$-th) unlearning epoch. Per-metric arrows follow TOFU~ \citep{maini2024tofu}. While all methods achieve comparable forgetting strength on the Forget Set, \textsc{GUARD} delivers the highest Retain ROUGE-L and Probability.}
\label{tab:attribution_absolute}
\vspace{0.1in}
\fontsize{7pt}{7.5pt}\selectfont
\resizebox{0.9\textwidth}{!}{
\begin{tabular}{lcccccc}
\toprule
\multirow{2}{*}{\textbf{Methods}}
& \multicolumn{3}{c}{\textbf{Retain Set}}
& \multicolumn{3}{c}{\textbf{Forget Set}} \\
\cmidrule(lr){2-4} \cmidrule(lr){5-7}
& ROUGE-L ($\uparrow$) & Prob. (\%) ($\uparrow$) & T.Ratio ($\uparrow$)
& ROUGE-L ($\downarrow$) & Prob. (\%) ($\downarrow$) & T.Ratio ($\uparrow$) \\
\midrule
NPO + ES
& 81.50 & 93.50 & \textbf{49.20}
& 34.20 & 38.80 & \textbf{65.40} \\
\midrule
NPO + RelatIF
& 84.30 & 95.10 & 48.80
& 33.90 & 38.50 & 65.10 \\
\midrule
NPO + RIF
& 86.20 & 96.20 & 48.60
& \textbf{33.70} & 38.30 & 64.90 \\
\midrule
NPO + GUARD
& \textbf{89.30} & \textbf{97.20} & 48.30
& 33.80 & \textbf{38.20} & 65.20 \\

\bottomrule
\end{tabular}}
\vspace{-1em}
\end{table}

%% file: contents/a.7.limitation.tex
\section{Future Work}
\label{sec:limitations_future}

We see several directions for future work. (i) \emph{Adaptive attribution schedules}: re-computing or smoothly updating the proxy attribution every few unlearning steps, possibly via a low-rank gradient sketch, could close the gap with a fully online influence-function while keeping the cost amortized. (ii) \emph{Method-agnostic theory}: extending the Sacrifice-Rate analysis beyond GA to a general family of unlearning objectives---including bounded-loss methods (NPO, SimNPO), regularized variants (KL/GD), and feature-space methods (RMU)---would clarify when \textsc{GUARD}'s reweighting is provably tight. (iii) \emph{Beyond entity-level unlearning}: applying \textsc{GUARD} to behavior-level alignment, copyrighted-content removal at scale, and sequential ``Right-to-be-Forgotten'' streams would test the framework's practical reach; combining \textsc{GUARD} with certified unlearning bounds is a particularly promising direction. (iv) \emph{Better synthetic retain proxies}: beyond GPT-4 generation, exploring retrieval-and-paraphrase pipelines, self-consistency-filtered model self-generation, and embedding-space coverage objectives could reduce the gap between synthetic and full-retain GUARD on the harder TOFU split. We view these directions as natural next steps for retention-aware unlearning at deployment scale.

%% file: main.bbl
\newpage

\begin{thebibliography}{48}
\providecommand{\natexlab}[1]{#1}
\providecommand{\url}[1]{\texttt{#1}}
\expandafter\ifx\csname urlstyle\endcsname\relax
  \providecommand{\doi}[1]{doi: #1}\else
  \providecommand{\doi}{doi: \begingroup \urlstyle{rm}\Url}\fi

\bibitem[Bae et~al.(2022)Bae, Ng, Lo, Ghassemi, and Grosse]{bae2022if}
Juhan Bae, Nathan Ng, Alston Lo, Marzyeh Ghassemi, and Roger~B Grosse.
\newblock If influence functions are the answer, then what is the question?
\newblock \emph{Advances in Neural Information Processing Systems},
  35:\penalty0 17953--17967, 2022.

\bibitem[Barshan et~al.(2020)Barshan, Brunet, and
  Dziugaite]{barshan2020relatif}
Elnaz Barshan, Marc-Etienne Brunet, and Gintare~Karolina Dziugaite.
\newblock {RelatIF}: Identifying explanatory training samples via relative
  influence.
\newblock In \emph{International Conference on Artificial Intelligence and
  Statistics}, pp.\  1899--1909. PMLR, 2020.

\bibitem[Basu et~al.(2020)Basu, Pope, and Feizi]{basu2020influence}
Samyadeep Basu, Philip Pope, and Soheil Feizi.
\newblock Influence functions in deep learning are fragile.
\newblock \emph{arXiv preprint arXiv:2006.14651}, 2020.

\bibitem[Bhaila et~al.(2025)Bhaila, Van, and Wu]{bhaila2025soft}
Karuna Bhaila, Minh-Hao Van, and Xintao Wu.
\newblock Soft prompting for unlearning in large language models.
\newblock In \emph{Proceedings of the 2025 Conference of the Nations of the
  Americas Chapter of the Association for Computational Linguistics: Human
  Language Technologies (Volume 1: Long Papers)}, pp.\  4046--4056, 2025.

\bibitem[Carlini et~al.(2021)Carlini, Tramer, Wallace, Jagielski, Herbert-Voss,
  Lee, Roberts, Brown, Song, Erlingsson, et~al.]{carlini2021memorization}
Nicholas Carlini, Florian Tramer, Eric Wallace, Matthew Jagielski, Ariel
  Herbert-Voss, Katherine Lee, Adam Roberts, Tom Brown, Dawn Song, Ulfar
  Erlingsson, et~al.
\newblock Extracting training data from large language models.
\newblock In \emph{30th USENIX security symposium (USENIX Security 21)}, pp.\
  2633--2650, 2021.

\bibitem[Fan et~al.(2024)Fan, Liu, Lin, Jia, Zhang, Mei, and
  Liu]{fan2025simplicityprevailsrethinkingnegative}
Chongyu Fan, Jiancheng Liu, Licong Lin, Jinghan Jia, Ruiqi Zhang, Song Mei, and
  Sijia Liu.
\newblock Simplicity prevails: Rethinking negative preference optimization for
  {LLM} unlearning.
\newblock \emph{arXiv preprint arXiv:2410.07163}, 2024.

\bibitem[Feng et~al.(2024)Feng, Chen, Li, and Lin]{feng2024fine}
XiaoHua Feng, Chaochao Chen, Yuyuan Li, and Zibin Lin.
\newblock Fine-grained pluggable gradient ascent for knowledge unlearning in
  language models.
\newblock In \emph{Proceedings of the 2024 Conference on Empirical Methods in
  Natural Language Processing}, pp.\  10141--10155, 2024.

\bibitem[Geng et~al.(2025)Geng, Li, Woisetschlaeger, Chen, Cai, Wang, Nakov,
  Jacobsen, and Karray]{geng2025comprehensive}
Jiahui Geng, Qing Li, Herbert Woisetschlaeger, Zongxiong Chen, Fengyu Cai,
  Yuxia Wang, Preslav Nakov, Hans-Arno Jacobsen, and Fakhri Karray.
\newblock A comprehensive survey of machine unlearning techniques for large
  language models.
\newblock \emph{arXiv preprint arXiv:2503.01854}, 2025.

\bibitem[Grosse et~al.(2023)Grosse, Bae, Anil, Elhage, Tamkin, Tajdini,
  Steiner, Li, Durmus, Perez, et~al.]{grosse2023studying}
Roger Grosse, Juhan Bae, Cem Anil, Nelson Elhage, Alex Tamkin, Amirhossein
  Tajdini, Benoit Steiner, Dustin Li, Esin Durmus, Ethan Perez, et~al.
\newblock Studying large language model generalization with influence
  functions.
\newblock \emph{arXiv preprint arXiv:2308.03296}, 2023.

\bibitem[Ha et~al.(2025)Ha, Park, and Yoon]{ha2025unlearning}
SeungBum Ha, Saerom Park, and Sung~Whan Yoon.
\newblock Unlearning's blind spots: Over-unlearning and prototypical relearning
  attack.
\newblock \emph{arXiv preprint arXiv:2506.01318}, 2025.

\bibitem[Hanawa et~al.(2021)Hanawa, Yokoi, Hara, and
  Inui]{hanawa2021evaluation}
Kazuaki Hanawa, Sho Yokoi, Satoshi Hara, and Kentaro Inui.
\newblock Evaluation of similarity-based explanations.
\newblock In \emph{International Conference on Learning Representations}, 2021.

\bibitem[Ilyas et~al.(2022)Ilyas, Park, Engstrom, Leclerc, and
  Madry]{ilyas2022datamodels}
Andrew Ilyas, Sung~Min Park, Logan Engstrom, Guillaume Leclerc, and Aleksander
  Madry.
\newblock {Datamodels}: Understanding predictions with data and data with
  predictions.
\newblock In \emph{International Conference on Machine Learning}, volume 162 of
  \emph{Proceedings of Machine Learning Research}, 2022.

\bibitem[Ji et~al.(2024)Ji, Liu, Zhang, Liu, Kompella, Liu, and
  Chang]{ji2024reversing}
Jiabao Ji, Yujian Liu, Yang Zhang, Gaowen Liu, Ramana~R Kompella, Sijia Liu,
  and Shiyu Chang.
\newblock Reversing the forget-retain objectives: An efficient {LLM} unlearning
  framework from logit difference.
\newblock \emph{Advances in Neural Information Processing Systems},
  37:\penalty0 12581--12611, 2024.

\bibitem[Jia et~al.(2024)Jia, Liu, Zhang, Ram, Baracaldo, and
  Liu]{jia2024wagle}
Jinghan Jia, Jiancheng Liu, Yihua Zhang, Parikshit Ram, Nathalie Baracaldo, and
  Sijia Liu.
\newblock {WAGLE}: Strategic weight attribution for effective and modular
  unlearning in large language models.
\newblock \emph{arXiv preprint arXiv:2410.17509}, 2024.

\bibitem[Koh \& Liang(2017)Koh and Liang]{koh2017understanding}
Pang~Wei Koh and Percy Liang.
\newblock Understanding black-box predictions via influence functions.
\newblock In \emph{International conference on machine learning}, pp.\
  1885--1894. PMLR, 2017.

\bibitem[Li et~al.(2024{\natexlab{a}})Li, Wei, Zhang, Qi, Du, Chen, Bi, and
  Liu]{li2024single}
Jiaqi Li, Qianshan Wei, Chuanyi Zhang, Guilin Qi, Miaozeng Du, Yongrui Chen,
  Sheng Bi, and Fan Liu.
\newblock Single image unlearning: Efficient machine unlearning in multimodal
  large language models.
\newblock \emph{Advances in Neural Information Processing Systems},
  37:\penalty0 35414--35453, 2024{\natexlab{a}}.

\bibitem[Li et~al.(2024{\natexlab{b}})Li, Pan, Gopal, Yue, Berrios, Gatti, Li,
  Dombrowski, Goel, Phan, et~al.]{li2024wmdp}
Nathaniel Li, Alexander Pan, Anjali Gopal, Summer Yue, Daniel Berrios, Alice
  Gatti, Justin~D Li, Ann-Kathrin Dombrowski, Shashwat Goel, Long Phan, et~al.
\newblock The {WMDP} benchmark: Measuring and reducing malicious use with
  unlearning.
\newblock \emph{arXiv preprint arXiv:2403.03218}, 2024{\natexlab{b}}.

\bibitem[Li et~al.(2023)Li, Bubeck, Eldan, Del~Giorno, Gunasekar, and
  Lee]{li2023textbooks}
Yuanzhi Li, S{\'e}bastien Bubeck, Ronen Eldan, Allie Del~Giorno, Suriya
  Gunasekar, and Yin~Tat Lee.
\newblock Textbooks are all you need {II}: {Phi}-1.5 technical report.
\newblock \emph{arXiv preprint arXiv:2309.05463}, 2023.

\bibitem[Liu et~al.(2022)Liu, Liu, and Stone]{liu2022continual}
Bo~Liu, Qiang Liu, and Peter Stone.
\newblock Continual learning and private unlearning.
\newblock In \emph{Conference on Lifelong Learning Agents}, pp.\  243--254.
  PMLR, 2022.

\bibitem[Liu et~al.(2024{\natexlab{a}})Liu, Wang, Flanigan, and
  Liu]{liu2024large}
Chris~Y Liu, Yaxuan Wang, Jeffrey Flanigan, and Yang Liu.
\newblock Large language model unlearning via embedding-corrupted prompts.
\newblock \emph{Advances in Neural Information Processing Systems},
  37:\penalty0 118198--118266, 2024{\natexlab{a}}.

\bibitem[Liu et~al.(2024{\natexlab{b}})Liu, Ren, Jin, Zhang, Zhou, Valduriez,
  and Dou]{liu2024fisher}
Ji~Liu, Jiaxiang Ren, Ruoming Jin, Zijie Zhang, Yang Zhou, Patrick Valduriez,
  and Dejing Dou.
\newblock Fisher information-based efficient curriculum federated learning with
  large language models.
\newblock \emph{arXiv preprint arXiv:2410.00131}, 2024{\natexlab{b}}.

\bibitem[Liu et~al.(2025)Liu, Yao, Jia, Casper, Baracaldo, Hase, Yao, Liu, Xu,
  Li, et~al.]{liu2025rethinking}
Sijia Liu, Yuanshun Yao, Jinghan Jia, Stephen Casper, Nathalie Baracaldo, Peter
  Hase, Yuguang Yao, Chris~Yuhao Liu, Xiaojun Xu, Hang Li, et~al.
\newblock Rethinking machine unlearning for large language models.
\newblock \emph{Nature Machine Intelligence}, pp.\  1--14, 2025.

\bibitem[Maini et~al.(2024)Maini, Feng, Schwarzschild, Lipton, and
  Kolter]{maini2024tofu}
Pratyush Maini, Zhili Feng, Avi Schwarzschild, Zachary~C Lipton, and J~Zico
  Kolter.
\newblock {TOFU}: A task of fictitious unlearning for {LLMs}.
\newblock \emph{arXiv preprint arXiv:2401.06121}, 2024.

\bibitem[Mirzasoleiman et~al.(2020)Mirzasoleiman, Bilmes, and
  Leskovec]{mirzasoleiman2020coresets}
Baharan Mirzasoleiman, Jeff Bilmes, and Jure Leskovec.
\newblock Coresets for data-efficient training of machine learning models.
\newblock In \emph{International Conference on Machine Learning}, pp.\
  6950--6960. PMLR, 2020.

\bibitem[Nguyen et~al.(2024)Nguyen, Yang, Anand, Yang, and
  Mirzasoleiman]{nguyen2025minibatchcoresetsmemoryefficientlanguage}
Dang Nguyen, Wenhan Yang, Rathul Anand, Yu~Yang, and Baharan Mirzasoleiman.
\newblock Mini-batch coresets for memory-efficient language model training on
  data mixtures.
\newblock \emph{arXiv preprint arXiv:2407.19580}, 2024.

\bibitem[Ouyang et~al.(2022)Ouyang, Wu, Jiang, Almeida, Wainwright, Mishkin,
  Zhang, Agarwal, Slama, Ray, et~al.]{ouyang2022training}
Long Ouyang, Jeffrey Wu, Xu~Jiang, Diogo Almeida, Carroll Wainwright, Pamela
  Mishkin, Chong Zhang, Sandhini Agarwal, Katarina Slama, Alex Ray, et~al.
\newblock Training language models to follow instructions with human feedback.
\newblock \emph{Advances in neural information processing systems},
  35:\penalty0 27730--27744, 2022.

\bibitem[Park et~al.(2023)Park, Georgiev, Ilyas, Leclerc, and
  Madry]{park2023trak}
Sung~Min Park, Kristian Georgiev, Andrew Ilyas, Guillaume Leclerc, and
  Aleksander Madry.
\newblock {TRAK}: Attributing model behavior at scale.
\newblock \emph{arXiv preprint arXiv:2303.14186}, 2023.

\bibitem[Qiu et~al.(2025)Qiu, Tan, Pu, Wang, Gao, and Sun]{qiu2025survey}
Ruichen Qiu, Jiajun Tan, Jiayue Pu, Honglin Wang, Xiao-Shan Gao, and Fei Sun.
\newblock A survey on unlearning in large language models.
\newblock \emph{arXiv preprint arXiv:2510.25117}, 2025.

\bibitem[Rad \& Maleki(2018)Rad and Maleki]{rad2018scalable}
Kamiar~Rahnama Rad and Arian Maleki.
\newblock A scalable estimate of the extra-sample prediction error via
  approximate leave-one-out.
\newblock \emph{arXiv preprint arXiv:1801.10243}, 2018.

\bibitem[Rafailov et~al.(2023)Rafailov, Sharma, Mitchell, Manning, Ermon, and
  Finn]{rafailov2024dpo}
Rafael Rafailov, Archit Sharma, Eric Mitchell, Christopher~D Manning, Stefano
  Ermon, and Chelsea Finn.
\newblock Direct preference optimization: Your language model is secretly a
  reward model.
\newblock \emph{Advances in neural information processing systems},
  36:\penalty0 53728--53741, 2023.

\bibitem[Rubinstein \& Hopkins(2025)Rubinstein and
  Hopkins]{rubinstein2025rescaled}
Ittai Rubinstein and Samuel~B Hopkins.
\newblock Rescaled influence functions: Accurate data attribution in high
  dimension.
\newblock \emph{arXiv preprint arXiv:2506.06656}, 2025.

\bibitem[Schwinn et~al.(2024)Schwinn, Dobre, Xhonneux, Gidel, and
  G{\"u}nnemann]{schwinn2024soft}
Leo Schwinn, David Dobre, Sophie Xhonneux, Gauthier Gidel, and Stephan
  G{\"u}nnemann.
\newblock Soft prompt threats: Attacking safety alignment and unlearning in
  open-source {LLMs} through the embedding space.
\newblock \emph{Advances in Neural Information Processing Systems},
  37:\penalty0 9086--9116, 2024.

\bibitem[Shi et~al.(2023)Shi, Min, Lomeli, Zhou, Li, Szilvasy, James, Lin,
  Smith, Zettlemoyer, et~al.]{shi2024incontextpretraininglanguagemodeling}
Weijia Shi, Sewon Min, Maria Lomeli, Chunting Zhou, Margaret Li, Gergely
  Szilvasy, Rich James, Xi~Victoria Lin, Noah~A Smith, Luke Zettlemoyer, et~al.
\newblock In-context pretraining: Language modeling beyond document boundaries.
\newblock \emph{arXiv preprint arXiv:2310.10638}, 2023.

\bibitem[Shi et~al.(2024)Shi, Lee, Huang, Malladi, Zhao, Holtzman, Liu,
  Zettlemoyer, Smith, and Zhang]{shi2024musemachineunlearningsixway}
Weijia Shi, Jaechan Lee, Yangsibo Huang, Sadhika Malladi, Jieyu Zhao, Ari
  Holtzman, Daogao Liu, Luke Zettlemoyer, Noah~A Smith, and Chiyuan Zhang.
\newblock {MUSE}: Machine unlearning six-way evaluation for language models.
\newblock \emph{arXiv preprint arXiv:2407.06460}, 2024.

\bibitem[Thudi et~al.(2022)Thudi, Deza, Chandrasekaran, and
  Papernot]{thudi2022unrolling}
Anvith Thudi, Gabriel Deza, Varun Chandrasekaran, and Nicolas Papernot.
\newblock Unrolling {SGD}: Understanding factors influencing machine
  unlearning.
\newblock In \emph{2022 IEEE 7th European Symposium on Security and Privacy
  (EuroS\&P)}, pp.\  303--319. IEEE, 2022.

\bibitem[Touvron et~al.(2023{\natexlab{a}})Touvron, Lavril, Izacard, Martinet,
  Lachaux, Lacroix, Rozi{\`e}re, Goyal, Hambro, Azhar,
  et~al.]{touvron2023llama}
Hugo Touvron, Thibaut Lavril, Gautier Izacard, Xavier Martinet, Marie-Anne
  Lachaux, Timoth{\'e}e Lacroix, Baptiste Rozi{\`e}re, Naman Goyal, Eric
  Hambro, Faisal Azhar, et~al.
\newblock {LLaMA}: Open and efficient foundation language models.
\newblock \emph{arXiv preprint arXiv:2302.13971}, 2023{\natexlab{a}}.

\bibitem[Touvron et~al.(2023{\natexlab{b}})Touvron, Martin, Stone, Albert,
  Almahairi, Babaei, Bashlykov, Batra, Bhargava, Bhosale,
  et~al.]{touvron2023llama2}
Hugo Touvron, Louis Martin, Kevin Stone, Peter Albert, Amjad Almahairi, Yasmine
  Babaei, Nikolay Bashlykov, Soumya Batra, Prajjwal Bhargava, Shruti Bhosale,
  et~al.
\newblock {Llama 2}: Open foundation and fine-tuned chat models.
\newblock \emph{arXiv preprint arXiv:2307.09288}, 2023{\natexlab{b}}.

\bibitem[Wang et~al.(2025{\natexlab{a}})Wang, Han, Yang, Zhu, Liu, and
  Sugiyama]{wang2025towards}
Qizhou Wang, Bo~Han, Puning Yang, Jianing Zhu, Tongliang Liu, and Masashi
  Sugiyama.
\newblock Towards effective evaluations and comparisons for {LLM} unlearning
  methods.
\newblock In \emph{The Thirteenth International Conference on Learning
  Representations}, 2025{\natexlab{a}}.

\bibitem[Wang et~al.(2025{\natexlab{b}})Wang, Zhou, Zhou, Shin, Han, and
  Weinberger]{wang2025rethinking}
Qizhou Wang, Jin~Peng Zhou, Zhanke Zhou, Saebyeol Shin, Bo~Han, and Kilian~Q
  Weinberger.
\newblock Rethinking {LLM} unlearning objectives: A gradient perspective and go
  beyond.
\newblock \emph{arXiv preprint arXiv:2502.19301}, 2025{\natexlab{b}}.

\bibitem[Wang et~al.(2024)Wang, Wei, Liu, Pang, Liu, Shah, Bao, Liu, and
  Wei]{wang2024llm}
Yaxuan Wang, Jiaheng Wei, Chris~Yuhao Liu, Jinlong Pang, Quan Liu, Ankit~Parag
  Shah, Yujia Bao, Yang Liu, and Wei Wei.
\newblock {LLM} unlearning via loss adjustment with only forget data.
\newblock \emph{arXiv preprint arXiv:2410.11143}, 2024.

\bibitem[Wei et~al.(2021)Wei, Bosma, Zhao, Guu, Yu, Lester, Du, Dai, and
  Le]{wei2022finetuned}
Jason Wei, Maarten Bosma, Vincent~Y Zhao, Kelvin Guu, Adams~Wei Yu, Brian
  Lester, Nan Du, Andrew~M Dai, and Quoc~V Le.
\newblock Finetuned language models are zero-shot learners.
\newblock \emph{arXiv preprint arXiv:2109.01652}, 2021.

\bibitem[Wei et~al.(2025)Wei, Niu, Hsu, Wu, Yin, Ghassemi, Li, Potluru, Chien,
  Chaudhuri, et~al.]{wei2025llmsreallyforgetevaluating}
Rongzhe Wei, Peizhi Niu, Hans Hao-Hsun Hsu, Ruihan Wu, Haoteng Yin, Mohsen
  Ghassemi, Yifan Li, Vamsi~K Potluru, Eli Chien, Kamalika Chaudhuri, et~al.
\newblock Do {LLMs} really forget? evaluating unlearning with knowledge
  correlation and confidence awareness.
\newblock \emph{arXiv preprint arXiv:2506.05735}, 2025.

\bibitem[Xu et~al.(2025)Xu, Ni, Broecker, and Strohmer]{xu2025forgetting}
Shizhou Xu, Yuan Ni, Stefan Broecker, and Thomas Strohmer.
\newblock {Forgetting-MarI}: {LLM} unlearning via marginal information
  regularization.
\newblock \emph{arXiv preprint arXiv:2511.11914}, 2025.

\bibitem[Yang et~al.(2025)Yang, Wang, Huang, Liu, Zhang, and
  Han]{yang2025exploringcriterialossreweighting}
Puning Yang, Qizhou Wang, Zhuo Huang, Tongliang Liu, Chengqi Zhang, and Bo~Han.
\newblock Exploring criteria of loss reweighting to enhance {LLM} unlearning.
\newblock \emph{arXiv preprint arXiv:2505.11953}, 2025.

\bibitem[Yao et~al.(2024{\natexlab{a}})Yao, Duan, Xu, Cai, Sun, and
  Zhang]{yao2024survey}
Yifan Yao, Jinhao Duan, Kaidi Xu, Yuanfang Cai, Zhibo Sun, and Yue Zhang.
\newblock A survey on large language model ({LLM}) security and privacy: The
  good, the bad, and the ugly.
\newblock \emph{High-Confidence Computing}, 4\penalty0 (2):\penalty0 100211,
  2024{\natexlab{a}}.

\bibitem[Yao et~al.(2024{\natexlab{b}})Yao, Xu, and Liu]{yao2024large}
Yuanshun Yao, Xiaojun Xu, and Yang Liu.
\newblock Large language model unlearning.
\newblock \emph{Advances in Neural Information Processing Systems},
  37:\penalty0 105425--105475, 2024{\natexlab{b}}.

\bibitem[Zhang et~al.(2024{\natexlab{a}})Zhang, Lin, Bai, and
  Mei]{zhang2024negative}
Ruiqi Zhang, Licong Lin, Yu~Bai, and Song Mei.
\newblock Negative preference optimization: From catastrophic collapse to
  effective unlearning.
\newblock \emph{arXiv preprint arXiv:2404.05868}, 2024{\natexlab{a}}.

\bibitem[Zhang et~al.(2024{\natexlab{b}})Zhang, Wang, Li, Wu, Tang, Liu, He,
  Yin, and Wang]{zhang2024catastrophic}
Zhiwei Zhang, Fali Wang, Xiaomin Li, Zongyu Wu, Xianfeng Tang, Hui Liu, Qi~He,
  Wenpeng Yin, and Suhang Wang.
\newblock Catastrophic failure of {LLM} unlearning via quantization.
\newblock \emph{arXiv preprint arXiv:2410.16454}, 2024{\natexlab{b}}.

\end{thebibliography}
